\PassOptionsToPackage{unicode}{hyperref}
\PassOptionsToPackage{hyphens}{url}
\PassOptionsToPackage{dvipsnames,svgnames,x11names}{xcolor}
\documentclass[
  12pt]{article}
\usepackage{xr-hyper}
\usepackage{setspace}
\usepackage{titletoc}
\usepackage{amsmath,amssymb}
\usepackage{iftex}
\ifPDFTeX
  \usepackage[T1]{fontenc}
  \usepackage[utf8]{inputenc}
  \usepackage{textcomp} 
\else 
  \usepackage{unicode-math}
  \defaultfontfeatures{Scale=MatchLowercase}
  \defaultfontfeatures[\rmfamily]{Ligatures=TeX,Scale=1}
\fi
\usepackage{lmodern}
\ifPDFTeX\else  
\fi
\IfFileExists{upquote.sty}{\usepackage{upquote}}{}
\IfFileExists{microtype.sty}{
  \usepackage[]{microtype}
  \UseMicrotypeSet[protrusion]{basicmath} 
}{}
\makeatletter
\@ifundefined{KOMAClassName}{
  \IfFileExists{parskip.sty}{%
    \usepackage{parskip}
  }{
    \setlength{\parindent}{0pt}
    \setlength{\parskip}{6pt plus 2pt minus 1pt}}
}{
  \KOMAoptions{parskip=half}}
\makeatother
\usepackage{xcolor}
\makeatletter
\ifx\paragraph\undefined\else
  \let\oldparagraph\paragraph
  \renewcommand{\paragraph}{
    \@ifstar
      \xxxParagraphStar
      \xxxParagraphNoStar
  }
  \newcommand{\xxxParagraphStar}[1]{\oldparagraph*{#1}\mbox{}}
  \newcommand{\xxxParagraphNoStar}[1]{\oldparagraph{#1}\mbox{}}
\fi
\ifx\subparagraph\undefined\else
  \let\oldsubparagraph\subparagraph
  \renewcommand{\subparagraph}{
    \@ifstar
      \xxxSubParagraphStar
      \xxxSubParagraphNoStar
  }
  \newcommand{\xxxSubParagraphStar}[1]{\oldsubparagraph*{#1}\mbox{}}
  \newcommand{\xxxSubParagraphNoStar}[1]{\oldsubparagraph{#1}\mbox{}}
\fi
\makeatother

\usepackage{calc} 
\usepackage{etoolbox}
\usepackage{lscape}
\usepackage{longtable}
\usepackage{booktabs}
\usepackage{longtable,booktabs,array}
\usepackage{calc} 
\usepackage{etoolbox}
\makeatletter
\patchcmd\longtable{\par}{\if@noskipsec\mbox{}\fi\par}{}{}
\makeatother
\IfFileExists{footnotehyper.sty}{\usepackage{footnotehyper}}{\usepackage{footnote}}
\makesavenoteenv{longtable}
\usepackage{graphicx}
\makeatletter
\def\maxwidth{\ifdim\Gin@nat@width>\linewidth\linewidth\else\Gin@nat@width\fi}
\def\maxheight{\ifdim\Gin@nat@height>\textheight\textheight\else\Gin@nat@height\fi}
\makeatother
\setkeys{Gin}{width=\maxwidth,height=\maxheight,keepaspectratio}
\makeatletter
\def\fps@figure{htbp}
\makeatother

\makeatletter
\@ifpackageloaded{caption}{}{\usepackage{caption}}
\AtBeginDocument{%
\ifdefined\contentsname
  \renewcommand*\contentsname{Table of contents}
\else
  \newcommand\contentsname{Table of contents}
\fi
\ifdefined\listfigurename
  \renewcommand*\listfigurename{List of Figures}
\else
  \newcommand\listfigurename{List of Figures}
\fi
\ifdefined\listtablename
  \renewcommand*\listtablename{List of Tables}
\else
  \newcommand\listtablename{List of Tables}
\fi
\ifdefined\figurename
  \renewcommand*\figurename{Figure}
\else
  \newcommand\figurename{Figure}
\fi
\ifdefined\tablename
  \renewcommand*\tablename{Table}
\else
  \newcommand\tablename{Table}
\fi
}
\@ifpackageloaded{float}{}{\usepackage{float}}
\floatstyle{ruled}
\@ifundefined{c@chapter}{\newfloat{codelisting}{h}{lop}}{\newfloat{codelisting}{h}{lop}[chapter]}
\floatname{codelisting}{Listing}

\makeatother
\makeatletter
\@ifpackageloaded{caption}{}{\usepackage{caption}}
\@ifpackageloaded{subcaption}{}{\usepackage{subcaption}}
\makeatother

\ifLuaTeX
  \usepackage{selnolig}  
\fi
\usepackage[]{natbib}
\usepackage{bookmark}

\IfFileExists{xurl.sty}{\usepackage{xurl}}{} 
\hypersetup{
  pdftitle={Title},
  pdfauthor={Author 1; Author 2},
  pdfkeywords={3 to 6 keywords, that do not appear in the title},
  colorlinks=true,
  linkcolor={blue},
  filecolor={Maroon},
  citecolor={Blue},
  urlcolor={Blue},
  pdfcreator={LaTeX via pandoc}}

\newcommand{\anon}{1}
\usepackage[margin=1in]{geometry}
\usepackage{subcaption}
\usepackage[toc]{appendix}
\usepackage{stackengine}
\usepackage{bbold}
\RequirePackage[OT1]{fontenc}
\usepackage{times}
\usepackage{pifont}
\newcommand{\cmark}{\ding{51}} 
\newcommand{\xmark}{\ding{55}} 
\usepackage{adjustbox}
\usepackage{listings}
\usepackage{xcolor}

\usepackage{algorithm}
\usepackage{algpseudocode}
\usepackage{amsmath,amssymb,amsthm,bm,latexsym}
\newcommand\independent{\protect\mathpalette{\protect\independenT}{\perp}}
\def\independenT#1#2{\mathrel{\rlap{$#1#2$}\mkern2mu{#1#2}}}

\newenvironment{subproof}[1][\proofname]{%
  \begin{proof}[#1]%
}{%
  \end{proof}%
}
\usepackage{graphicx,psfrag,epsf}
\usepackage{epigraph,xcolor}
\usepackage[figuresright]{rotating}
\usepackage{enumerate}

\usepackage{url}
\usepackage{booktabs}
\usepackage{multirow}
\usepackage{hyperref}
\hypersetup{colorlinks=true,citecolor=blue,
	pdfpagemode=FullScreen,linktocpage=true}
\usepackage{cleveref}
\usepackage{amsfonts}
\usepackage{amsmath}
\usepackage{orcidlink}
\usepackage[dvipsnames]{xcolor}
\usepackage{amsthm}
\usepackage{mathrsfs}
\usepackage[T1]{fontenc}
\usepackage{makecell}
\usepackage{adjustbox}

\newcommand{\RNum}[1]{\uppercase\expandafter{\romannumeral #1\relax}}
\DeclareMathOperator*{\argmin}{arg\,min}
\DeclareMathOperator*{\argmax}{arg\,max}
\def\tr{\rm tr}
\newcommand{\abs}[1]{\left\lvert #1 \right\rvert}
\newcommand{\norm}[1]{\left\| #1 \right\|}

\numberwithin{equation}{section}  
\newtheoremstyle{general}
{3mm} 
{3mm} 
{\it} 
{} 
{\bfseries} 
{.} 
{.5em} 
{} 

\theoremstyle{general}
\newtheorem{definition}{Definition}
\newtheorem{lemma}{Lemma}
\newtheorem{theorem}{Theorem}

\newtheorem{assumption}{Assumption}
\newtheorem{remark}{Remark}

\begin{document}

\def\spacingset#1{\renewcommand{\baselinestretch}%
{#1}\small\normalsize} \spacingset{1}



\if1\anon
{
  \title{\bf ANGLE: Angular Neural Generative Learning via Engression}
  \author{
    Rajdeep Pathak\textsuperscript{1,2,}\thanks{Joint first authors.} \quad 
    Archi Roy\textsuperscript{3,*} \quad 
    Tanujit Chakraborty\textsuperscript{1,2,}\thanks{Corresponding author: \texttt{tanujit.chakraborty@sorbonne.ae}} \\[0.5cm]
    \textsuperscript{1}LPSM, Sorbonne Universit\'{e}, Paris, France \\
    \textsuperscript{2}SAFIR, Sorbonne University Abu Dhabi, UAE \\
    \textsuperscript{3}Indian Institute of Management, Kozhikode, India \\[0.2cm]
  }
  \maketitle
} \fi

\if0\anon
{
  \bigskip
  \bigskip
  \bigskip
  \begin{center}
    {\LARGE\bf ANGLE: Angular Neural Generative Learning via Engression}
  \end{center}
  \medskip
} \fi

\bigskip
\begin{abstract}
Circular data, representing angles or directions, are frequently encountered in computer vision, biology, geology, and meteorology. Traditional regression targets the conditional mean, which is often geometrically misleading for circular responses under multimodal, skewed, or asymmetric data structures. To address these limitations, a lightweight deep generative framework, namely ANGLE, is introduced for non-parametric distributional regression on the circle. The full conditional distribution of an angular response, given Euclidean and circular covariates, is learned through a generative map optimized via a generalized circular energy score (GCES) loss. Desirable theoretical properties including the strict propriety of the loss and the rotational equivariance of the estimators are established. Furthermore, both pre- and post-additive noise models are accommodated. A unified toolbox is provided for advancing previously underexplored challenges in circular statistics: extrapolation, sufficient dimension reduction, and conditional distribution equality testing. The framework's efficacy is demonstrated through extensive simulations and real-world applications. Specifically, the proposal is utilized for object pose estimation from imagery and wind direction prediction, which are integral to surveillance, autonomous vehicles, and energy systems, respectively. Superior predictive performance and robust uncertainty quantification of the proposed method in these tasks are revealed. Software implementation is made publicly available\footnote{\url{https://github.com/PyCoder913/anglepy}} via our Python package \texttt{anglepy}\footnote{Documentation: \url{https://anglepy.readthedocs.io}}.
\end{abstract}

\noindent%
\vfill

\spacingset{1.8} 

\section{Introduction} 
\label{sec:introduction}

Circular data refers to observations that take the form of angles or directions on the unit circle $\mathbb{S}^1$, and they arise naturally in many scientific disciplines.  Common examples of circular data include times of day or days of a calendar year, which wrap naturally onto a 24-hour or 12-month clock, and directions or orientations, which are measured as angles on the unit circle. Several interesting applications of circular data are presented in \cite{gill2010circular}, where the authors model the timing of domestic terrorism events as points on the calendar circle; and in \cite{villarini2016seasonality}, where the day of annual peak river flow is modeled as an angle to capture seasonal flooding patterns. Beyond these examples, circular statistical methods have also been extensively employed in domains like  biological rhythm discovery \citep{gorczyca2025weighted, downs2002circular} and movement analysis \citep{ranalli2020model}. In most of these applications, the primary statistical objective is to estimate the conditional mean of the angular response, given a set of explanatory variables. However, there has been increasing recognition that a single summary statistic such as the mean often fails to capture the full structure of the conditional relationship, and that the entire conditional distribution is the more relevant inferential target. This need is particularly acute for circular responses where the conditional distribution can be multimodal or asymmetric in ways that render point summaries uninformative and geometrically misleading \citep{di2016note}. 

As a motivating example, consider pose estimation, a well-established task in computer vision. Recovering object orientation from imagery or videos is critical for head pose \citep{prokudin2018deep} and gaze direction \citep{nonaka2022dynamic} detection, autonomous steering \citep{rasib2021pixel}, projectile tracking, robotic manipulation, augmented reality, and in-orbit operations (see \cite{liu2024deep} for a comprehensive review). However, occlusion, poor illumination, low resolution, or symmetries frequently result in multiple plausible poses. Estimating the full conditional distribution captures this uncertainty; as shown in Fig. \ref{fig:object_poses}, ambiguous inputs (e.g., the left-facing motorbike) necessitate wider predictive intervals to bound the true pose, which point estimators like circular linear regression (CLR) fails to provide. An analogous challenge arises in meteorology, where atmospheric dynamics, terrain-induced channeling, land-sea interactions, and competing weather systems create complex, multimodal wind direction distributions \citep{rad2022enhancing}. Because single point predictions obscure these underlying regimes, estimating the full distribution is vital for generating the probabilistic predictions (predictive intervals) required by safety-critical operations like aircraft takeoff and landing \citep{khattak2023assessing}, wildfire spread prediction \citep{quill2019modeling}, and maritime navigation \citep{ning2025hybrid}.

\begin{figure}[t!]
        \centering
    \includegraphics[width=0.9\linewidth]{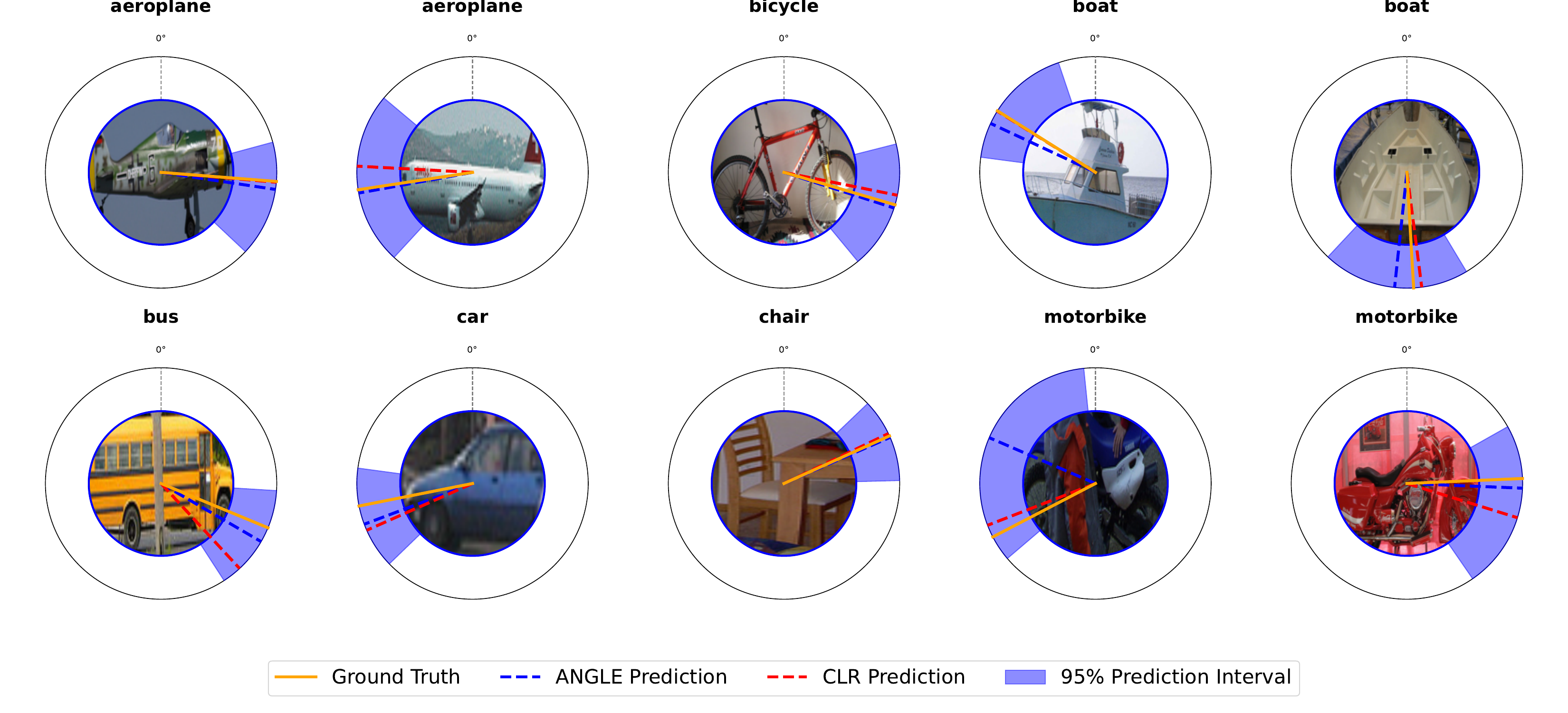}
        \caption{{\small Object pose estimation with uncertainty quantification with the proposed ANGLE. A detailed discussion is provided in Sec. \ref{sec:object_pose_estimation}.}} 
        \label{fig:object_poses}
\end{figure}


Motivated by the inadequacy of point summaries in many such applications involving angular responses, we develop \textit{ANGLE} - Angular Neural Generative Learning via Engression, a deep generative distributional regression framework for circular data. Given covariates $X\in\mathbb{R}^d$ and an angular response $Y\in[0,2\pi)$, the method learns the conditional distribution of $Y|X$ through a generative map $g(X,\varepsilon)$, where $\varepsilon\sim\mathcal U(0,1)$, estimated by minimizing a generalized circular energy score (GCES) loss. While our work draws conceptual inspiration from \cite{shen2025engression}, the extension to circular responses is far from straightforward. Notably, standard deep networks cannot be used directly to fit bounded and periodic angular responses. Moreover, the Euclidean energy score fails to accommodate the periodic wrapping effect ($0 \equiv 2\pi$), leading to erroneous predictive penalization. Unlike the Euclidean setting, where the energy score is known to be strictly proper, the circular setting lacks a canonical notion of distance, making distributional learning substantially more challenging. To address this issue, we introduce a family of GCES losses based on geodesic and chordal distances on $\mathbb S^1$ and characterize conditions under which they remain strictly proper. The resulting framework is nonparametric, bandwidth-free, and rotationally equivariant, ensuring that both distributional estimates and derived functionals respect the geometry of the circle. Beyond conditional distribution estimation, the learned generative object naturally enables a range of downstream statistical tasks. We develop methodologies for point prediction with uncertainty quantification - specifically, the estimation of circular mean, median, mode, density, dispersion, and quantiles; sufficient dimension reduction; extrapolation; and testing equality of conditional distributions, accompanied by desirable theoretical guarantees. The practical utility of the proposed framework is demonstrated through applications to object pose estimation from images using the PASCAL3D+  benchmark \citep{xiang2014beyond} and wind direction prediction using datasets from India and Germany \citep{cohen2025bayesian}. 

In summary, this work provides a unified lightweight generative framework for distributional learning on circular data. We summarize our key contributions below.

\begin{enumerate}
    \item[(a)] \textit{ANGLE framework and theoretical foundations.} We propose ANGLE, a deep generative distributional regression method for circular data, establishing population-level guarantee and rotational equivariance for the derived estimators (\Cref{thm:population_guarantee_engression}; \Cref{lem:point_predictions_without_origin}). A family of GCES losses is introduced and conditions for their strict propriety are characterized.
    \item[(b)] \textit{Comprehensive downstream applications.} We develop a unified toolkit for angular responses with theoretical guarantees across all tasks, enabling point prediction with uncertainty quantification (\Cref{lem:point_predictions_without_origin}), sufficient dimension reduction (Theorems \ref{thm:central_sigma_field_exists}-\ref{thm:central_subspace_recovery}), extrapolation (\Cref{thm:extrapolability_criteria}), and conditional distribution equality testing (\Cref{alg:conditional_test}).
    \item[(c)] \textit{Architectural flexibility and computational efficiency.} The framework adapts to both pre-additive (covariate noise) and post-additive (response noise) models, remaining significantly more computationally lightweight than existing classical, Bayesian, and modern deep learning alternatives available in circular statistics.
    \item[(d)] \textit{Applications to image and tabular data.} We demonstrate the method's efficacy on image data (object pose estimation, SDR, and distribution equality testing) and tabular data (wind direction prediction and extrapolation), alongside extensive simulation experiments. The proposed ANGLE methodology provides model-intrinsic uncertainty quantification for circular responses, the efficacy of which is systematically benchmarked against state-of-the-art alternatives using the circular continuous ranked probability score (CRPS).
\end{enumerate}

\subsection{Background}

To estimate the conditional distribution of an Euclidean target variable given covariates, \textit{Engression} was recently introduced as a deep distributional regression framework \citep{shen2025engression}. Engression optimizes a strictly proper scoring rule within an additive noise model, providing theoretical extrapolation guarantees under mild monotonicity assumptions. Recent works have utilized this methodology for problems with temporal \citep{kraft2026modeling} and spatiotemporal \citep{pathak2026deep} data. While it was introduced for Euclidean responses with the main narrative around extrapolation, the original framework did not explore other downstream statistical tasks such as SDR and conditional distribution equality testing. We bridge this gap for the circular data domain by introducing a comprehensive and unified toolkit designed for distributional regression and subsequent downstream analysis of angular data. In this paper, we develop the ANGLE methodology and evaluate its performance across various data modalities, including both imagery and tabular datasets.

\subsection{Organization and Notations}
\label{sec:notations}
The remainder of the paper is organized as follows. Following the notation setup, Sec. \ref{sec:methodology} outlines the ANGLE methodology and model architecture. Sec. \ref{sec:application_CircEngression} develops downstream statistical tasks: conditional functional estimation (Sec. \ref{sec:Point_prediction}), sufficient dimension reduction (Sec. \ref{sec:application_sdr}), extrapolation (Sec. \ref{sec:identifiability_and_extrapolation}), and conditional distribution equality testing (Sec. \ref{sec:application_equality_testing}). Empirical validation via simulations is presented in Sec. \ref{sec:simulations}, followed by applications to object pose and wind direction prediction in Sec. \ref{sec:applications}. Sec. \ref{sec:conclusion} concludes with future directions. Proofs, software demonstrations, implementation details, and supporting materials are deferred to the Appendix.

Throughout, $\overset{d}{=}$ denotes equality in distribution and $\iota=\sqrt{-1}$. For a real-valued random variable $X$, let $\psi_X(r)=\mathbf{E}[\exp(\iota rX)]$, $r\in\mathbb Z$, denote its $r$-th circular moment; similarly, for a function $f$, let $\psi_f(r)=\mathbf{E}[\exp(\iota r f(\varepsilon))]$. The $k$-th component of a vector $u$ is denoted by $u_{(k)}$. For a measurable set $A, \ \mathcal U(A)$ denotes the uniform distribution on $A$. For a measure $\mu$, $I_\mu$ denotes its support, $f|_A$ the restriction of a function $f$ to $A$, and $\mu_*\nu$ the push-forward of a measure $\nu$ under $\mu$, defined by $\mu_*\nu(A)=\nu(\mu^{-1}(A))$. We write $\mathbf P_{\mathrm{tr}}$ for the training-data distribution and $\mathcal X\subset\mathbb R^d$ for the training support. Throughout, $\mathbb S^1$ is parameterized by $[0,2\pi)$, with circular addition and subtraction defined by $a\oplus b=(a+b)\bmod 2\pi$ and $a\ominus b=(a-b)\bmod 2\pi$. For a random variable $X$, $\sigma(X)$ denotes the generated sigma-field. We write $\mathrm{atan2}(s,c)=\arg(c+\iota s)$, where $\arg:\mathbb C\setminus\{0\}\to(-\pi,\pi]$ is the principal argument, $\mathbb{C}\setminus\{0\}$ denotes the entire complex plane excluding the origin, and $(s,c)\neq(0,0)$. $\|\cdot\|$ denotes the Euclidean norm unless stated otherwise. The notation $X\independent Y$ denotes that the random variable $X$ is independent of $Y$.

\section{Methodology}
\label{sec:methodology}

We consider a generative approach to modeling a univariate angular response $Y\in [0,2\pi)$ given the Euclidean covariates $X\in \mathbb{R}^d$. Rather than directly specifying a parametric family for the conditional distribution $\mathbf{P}_{Y|X}$, we characterize it implicitly through a generative process as follows. We posit a function $g$ belonging to a prescribed class $\mathcal{M}$, and a real-valued noise variable $\varepsilon$ independent of $X$ such that $Y|X\overset{d}{=}g(X,\varepsilon)$. For each $g\in \mathcal{M}$ and $x\in \mathcal{X} \subset \mathbb{R}^d$, we write $\mathbf{P}_g(\cdot|x)$ as the distribution induced by $g(x,\varepsilon)$, where the randomness arises solely through $\varepsilon$. This generative formulation allows for two desirable features. First, it does not require explicit density specification on the circle, and second, it enables simulation-based approximation of the conditional distribution and sample-based estimation of circular statistics. We further illustrate this in detail in \Cref{sec:application_CircEngression}. Our proposed framework is inspired by the pre-ANM and post-ANM formulations of engression \citep{shen2025engression}, where latent noise is injected either before or after the non-linear link, and the resulting response is mapped onto the circle. We consider the following classes of circular generative models:
\begin{align}\label{eq:model_classes}
&\mathbf{M}_{\text{post}}^{(1)}=\left\{\left(g\left(\beta^\top x\right)+\eta\right)\bmod 2\pi \mid g:\mathbb{R}\rightarrow\mathbb{R},\ g\in \mathcal{G},\ \beta\in \mathcal{B}\right\},\nonumber\\
&\mathbf{M}_{\text{post}}^{(2)}=\left\{\left(g\left(\beta^\top x\right)+\eta\right)\bmod 2\pi \mid g:\mathbb{R}\rightarrow[0,2\pi),\ g\in \mathcal{G},\ \beta\in \mathcal{B}\right\},\nonumber\\
&\mathbf{M}_{\text{pre}}^{(1)}=\left\{g\left(\beta^\top x+\eta\right)\bmod 2\pi \mid g:\mathbb{R}\rightarrow\mathbb{R},\ g\in \mathcal{G},\ \beta\in \mathcal{B}\right\},\nonumber\\
&\mathbf{M}_{\text{pre}}^{(2)}=\left\{g\left(\beta^\top x+\eta\right) \mid g:\mathbb{R}\rightarrow[0,2\pi),\ g\in \mathcal{G},\ \beta\in \mathcal{B}\right\}. 
\end{align}

Note that $\mathbf{M}_{\text{post}}^{(1)}$ and $\mathbf{M}_{\text{pre}}^{(1)}$ follow a `wrapped modeling' approach, where the link function $g$ is unbounded, but the final output is modulated by $2\pi$. In contrast, $g$ is bounded in $[0,2\pi)$ for models in $\mathbf{M}_{\text{post}}^{(2)}$ and $\mathbf{M}_{\text{pre}}^{(2)}$.  In each class, the noise $\eta$ enters through a transformation $\eta\overset{d}{=}h(\varepsilon)$ for some measurable $h\in \mathcal{H}$, where we take $\varepsilon\sim \mathcal{U}(0,1)$ to be independent of $X$. The classes $\mathbf{M}_{\text{post}}^{(1)}$ and $\mathbf{M}_{\text{post}}^{(2)}$ introduce noise after the signal $g\left(\beta^\top x\right)$ is formed (post-ANM), while in $\mathbf{M}_{\text{pre}}^{(1)}$ and $\mathbf{M}_{\text{pre}}^{(2)}$ the noise is directly introduced in the covariate level (pre-ANM). The parameter space $\mathcal{B}$ and the function classes $\mathcal{G},\mathcal{H}$ are chosen depending on the statistical objective. For standard prediction tasks, it suffices to take $\mathcal{B}=\left\{\beta\in \mathbb{R}^d : \norm{\beta}<\infty\right\}$, with $\mathcal{G}$ and $\mathcal{H}$ classes of smooth functions. However, as we demonstrate in \Cref{sec:identifiability_and_extrapolation}, stronger structural assumptions on the parameter and function spaces are required when the goal is extrapolation. Although the response variable is circular, we model the latent noise variable $\eta$ as real-valued rather than circular. This is primarily motivated by many real-world phenomena, where the observed angle arises as the circular projection of an underlying Euclidean system subject to ordinary real-valued perturbations. While considering circular noise is also possible in the same framework, we leave the theoretical details for a future work. For simplicity and ease of presentation, we restrict our attention to single-index additive noise structures in the theoretical derivations. The results extend straightforwardly to the multiple-index setting, where $\beta\in \mathbb{R}^d$ is replaced by $\mathbf{B}\in \mathbb{R}^{d\times b}$, with suitable modifications to the identifiability conditions imposed whenever required. 

An important element of any distributional regression method is a quantitative metric to assess a distributional fit. We introduce the \textit{generalized circular energy score (GCES) loss}, which, for any direction $\theta$, is given as 
\begin{align}\label{eq:crps_angular}
     \overset{\circ}{\mathrm{ES}}(\mathbf{P},\theta)= \mathbf{E}[\mathcal{F}\left(\mathrm{d}(\theta,\Theta)\right)]-\frac{1}{2}\mathbf{E}[\mathcal{F}(\mathrm{d}(\Theta,\Theta^*))],
\end{align}
where $\Theta,\Theta^*$ are independent copies of circular random variables coming from the distribution $\mathbf{P}$, $\mathrm{d}(\theta_1,\theta_2)=\arccos{\langle \theta_1,\theta_2\rangle} = \min\{|\theta_1-\theta_2|, 2\pi-|\theta_1-\theta_2|\}$ is the geodesic metric on $\mathbb{S}^1$, and $\mathcal{F}$ represents a continuous function on $[0,\pi]$ such that $k(x,y)=\mathcal{F}(\mathrm{d}(x,y))$ is an isotropic kernel on $\mathbb{S}^1$. As we illustrate later in \Cref{thm:population_guarantee_engression}, for the population level guarantee of the proposed methodology, we require \eqref{eq:crps_angular} to be a strictly proper scoring rule. Because we frame our objective as a loss function, we adopt the negatively oriented convention: a scoring rule is strictly proper if its expected penalty is minimized if and only if the predicted probability distribution exactly matches the true underlying distribution. According to Theorem 1.1 and 4.5 of \cite{steinwart2021strictly}, this holds if and only if $-\mathcal{F}$ belongs to the class of functions that induce a strictly positive definite kernel  on $\mathbb{S}^1$. Examples of such functions are given in \citet[Table 1]{gneiting2013strictly}. Note that a simpler version of the circular energy score loss was considered in \cite{grimit2006continuous}, taking $\mathcal{F}$ to be the identity function, where the authors showed the metric to be proper. We need the above reformulation \eqref{eq:crps_angular} since strict propriety is crucial for this study. As illustrated in \Cref{lem:strict_properness_chordal} (Appendix \ref{sec:proofs_of_main_results}), it is possible to extend the construction beyond the geodesic metric. Specifically, \eqref{eq:crps_angular} continues to define a strictly proper energy score loss when $\mathcal{F}$ is the identity function and $\mathrm{d}$ is chosen to be the chordal distance. The population version of the proposed method can then be written as the minimizer of the GCES loss \eqref{eq:crps_angular}, i.e.,
\begin{align} 
    \label{eq:engression_population_version}
    \widetilde g\in \argmin_{g\in \mathcal{M}}\left\{\mathbf{E}\left[\mathcal{F}\left(\mathrm{d}\left(Y,g(X,\varepsilon)\right)\right)-\frac{1}{2}\mathcal{F}\left(\mathrm{d}\left(g(X,\varepsilon),g(X,\varepsilon^*)\right)\right)\right]\right\},
\end{align}
where $(X,Y)\sim \mathbf{P}_{\tr}(x,y)$ and $\varepsilon,\varepsilon^*\independent X$ are independent copies drawn from the $\mathcal{U}(0,1)$ distribution. The first step now is to establish a population level guarantee for the proposed method as a general distributional approach. That is, we need to show that when the model is correctly specified, \eqref{eq:engression_population_version} learns the true conditional distribution for all $x\in \mathcal{X}$.

\begin{theorem}[Population level guarantee]
    \label{thm:population_guarantee_engression}
    Assume that there exists a $g\in \mathcal{M}$ such that $g(x,\varepsilon)\sim \mathbf{P}_{\tr}(y|x)$ for all $x\in \mathcal{X}$, and that the generalized circular energy score loss is strictly proper. Let $\widetilde{g}$ be any population minimizer in \eqref{eq:engression_population_version} that minimizes the GCES loss. Then $\widetilde g(x,\varepsilon)\sim \mathbf{P}_{\tr}(y|x)$ $\mathbf{P}_X$-almost everywhere. 
\end{theorem}

Having established the population level guarantee for our methodology, we now move on to finite sample estimation. For each observation in an i.i.d. sample $\{(X_i,Y_i)|i=1,2,\ldots,n\}$, we sample an $\varepsilon$ from the $\mathcal{U}(0,1)$ distribution $m$ times: $\left\{\epsilon_i^{(j)}|i=1,2,\ldots,n,j=1,2,\ldots,m\right\}$. Based on the finite samples, we define the empirical version of circular engression as:
{\scriptsize
\begin{align}
 \label{eq:engression_sample_version}
    \widehat g\in \argmin_{g\in \mathcal{M}}\left\{n^{-1}\sum_{i=1}^{n}{\left[\frac{1}{m}\sum_{j=1}^{m}{\mathcal{F}\left(\mathrm{d}\left(Y_i,g\left(X_i,\epsilon_i^{(j)}\right)\right)\right)}-\frac{1}{2m(m-1)}\sum_{j=1}^{m}{\sum_{k=1, k\neq j}^{m}{\mathcal{F}\left(\mathrm{d}\left(g\left(X_i,\epsilon_i^{(j)}\right),g\left(X_i,\epsilon_i^{(k)}\right)\right)\right)}}\right]}\right\}.
\end{align}}
For each fixed $g$, the objective \eqref{eq:engression_sample_version} is an unbiased estimator of the population risk \eqref{eq:engression_population_version}. In particular, the second term in the right hand side of \eqref{eq:engression_sample_version} is written as a U-statistic over pairs $j\neq k$ to obtain the unbiased estimate of the self-dispersion term. We parameterize the model $g$ using a broad class of neural networks, as detailed in the subsequent section.

\subsection{Model Architecture}\label{sec:model_architechture}
The ANGLE framework seamlessly accommodates single-index, multiple-index, and fully connected architectures, allowing the model complexity to be tailored to the task at hand. The framework also supports both pre- and post-ANMs, as well as bounded circular outputs and unconstrained real-valued outputs projected onto $[0, 2\pi)$, thereby providing a unified implementation of all model classes in \eqref{eq:model_classes}. The architecture utilizes a multilayer perceptron (MLP) backbone. The network's hidden representation incorporates fully connected transformations, batch normalization, and residual connections, augmented by stochastic noise injection (sampled from either $\mathcal{U}(0,1)$ or $\mathcal{N}(0,1)$) to enable conditional distribution estimation via the engression methodology \citep{shen2025engression}. To obtain circular outputs, the terminal layer uses one of the three transformations:
\begin{enumerate}
    \item[(I)] A bivariate output $(u,v)\in\mathbb R^2$ can be mapped through $\widehat\theta=\mathrm{atan2}(v,u)\pmod{2\pi}$ with a small numerical stabilization to avoid the singular point $(u,v)=(0,0)$;
    \item[(II)] An unbounded scalar output $u\in\mathbb R$ can be reduced modulo $2\pi$ before evaluating the loss and at inference time;
    \item[(III)] An unconstrained scalar output can be mapped to $[0,2\pi)$ using a scaled sigmoid transformation.
\end{enumerate}
The modulo transformation implements the wrapped-output variants ($\mathbf{M}_{\text{pre,post}}^{(1)}$), whereas the $\mathrm{atan}2$ and scaled-sigmoid transformations implement the bounded circular output variants ($\mathbf{M}_{\text{pre,post}}^{(2)}$). The distinction between $\mathbf{M}_{\text{pre}}$ and $\mathbf{M}_{\text{post}}$ is determined by whether the stochastic noise is injected before or after the nonlinear map. Notably, the ANGLE architecture also accommodates circular covariates by replacing each angular covariate with its trigonometric embeddings before the initial layer. Finally, the network parameters are optimized by minimizing the empirical GCES loss defined in \eqref{eq:engression_sample_version}, using either geodesic or chordal distance as the underlying circular discrepancy. For geodesic distance, the function $\mathcal{F}$ is chosen as $\mathcal{F}(t) = -k(t)$ with $k$ selected from \citet[Table 1]{gneiting2013strictly}. The different variants of the proposed model are compared in an ablation study in \Cref{sec:ablation_study}.


\section{Statistical Inference with ANGLE}
\label{sec:application_CircEngression}

Having established the generative model classes and ANGLE framework for circular responses, we now turn to the theoretical and practical utilities of the proposed framework for statistical inference. A key advantage of the proposed approach is that, once the generative map $g(x,\varepsilon)$ has been estimated, the full conditional distribution $\mathbf{P}_{Y|X=x}$ becomes accessible through forward simulation. Specifically, for any covariate value $x$, we can draw i.i.d. $\varepsilon_1,\ldots,\varepsilon_m\sim\mathcal{U}(0,1)$ and generate approximate conditional samples from the fitted model: $\widehat Y_1=g(x,\varepsilon_1),\ldots,\widehat Y_m=g(x,\varepsilon_m)$. These simulated samples can then be used to estimate a wide range of conditional circular summaries without requiring an explicit closed-form expression for the conditional density. In the remainder of this section, we investigate the various inferential tasks enabled by the proposed framework.


\subsection{Point Prediction of Conditional Functionals} \label{sec:Point_prediction}

The proposed generative formulation enables straightforward estimation of a broad class of conditional functionals for circular responses, such as the mean and median directions, measures of circular dispersion, conditional quantiles, cumulative distribution functions, and densities. Even under standard parametric circular models, many of these quantities do not admit simple closed-form expressions and are therefore most naturally estimated by sampling from the fitted generative map. For completeness, Appendix~\ref{appendix:review_point_prediction} reviews existing approaches to estimating these functionals under classical circular models. In this section, we show how each of these quantities can be estimated directly from the fitted generative map in \eqref{eq:engression_population_version}.

A fundamental requirement for statistical inference on the circle is rotational equivariance, which ensures that inference is invariant to the arbitrary choice of reference direction. \Cref{lem:point_predictions_without_origin} establishes that, under correct model specification, the conditional functionals induced by ANGLE satisfy this property. Specifically, if the response space is rotated by an angle $\alpha$, then the estimated mean direction, median direction, mode, and density rotate by the same angle $\alpha$, while the circular dispersion remains unchanged.

\begin{lemma}[Rotational equivariance] \label{lem:point_predictions_without_origin}
    Let $(X,Y)\sim \mathbf{P}_{\tr}(x,y)$, $\widetilde g$ be the population ANGLE object defined in \eqref{eq:engression_population_version}, and $(X,Y^\alpha)$ denote the rotated training problem where $Y^\alpha=Y\oplus\alpha$ for any $\alpha\in \mathbb{R}$ with the corresponding training distribution $\mathbf{P}_{\tr}^\alpha(y,x)$. Denote by $\widetilde g_\alpha$ the population ANGLE object fitted on $(X,Y^\alpha)$. Assume that the underlying model class $\mathcal{M}$ is correctly specified for both problems, i.e., there exist $g,g_\alpha\in \mathcal{M}$ such that $g(x,\varepsilon)\sim \mathbf{P}_{\tr}(\cdot|x),g_\alpha(x,\varepsilon)\sim \mathbf{P}_{\tr}^\alpha(\cdot|x)$ for all $x\in \mathcal{X}$.
    \begin{enumerate}[label=(\Alph*)]
        \item[(A)] \label{lem:A}  Assume that the mean resultant length $\left|\mathbf{E}_\varepsilon[\exp(\iota\widetilde g(x,\varepsilon))]\right|>0$ $\mathbf{P}_X$-almost everywhere. Define the conditional mean direction of $Y|X=x$ as
\begin{align}\label{eq:conditional_mean_circular}
    \mu_{\widetilde g}(x)=\mathrm{atan2}\left(\mathbf{E}_{\varepsilon}\left[\sin \widetilde g(x,\varepsilon)\right],\mathbf{E}_\varepsilon\left[\cos \widetilde g(x,\varepsilon)\right]\right),
\end{align}
where $\mathrm{atan}2$ is the two-argument arctangent operator. Then $\mu_{\widetilde g_\alpha}(x)= \mu_{\widetilde g}(x)\oplus\alpha$ $\mathbf{P}_X$-almost everywhere. Hence, the conditional mean direction under the rotated problem is equal to the original conditional mean direction rotated by the same angle.
        \item[(B)] \label{lem:B} Define the circular median as $ m_{\widetilde g}(x)=\argmin_{\theta\in [0, 2\pi)}\left\{\mathbf{E}_\varepsilon[\mathrm{d}(\theta,\widetilde g(x,\varepsilon))]\right\}$. Then $m_{\widetilde g_\alpha}(x)=m_{\widetilde g}(x)\oplus\alpha$ $\mathbf{P}_X$-almost everywhere. For notational simplicity, we may assume uniqueness of the median; however, the rotational equivariance property remains valid for each of the medians in the non-unique case. 
        \item[(C)] \label{lem:C} Define the circular dispersion measure \citep{demni2026robust} as $\sigma_{\widetilde g}(x)=\mathbf{Med}_\varepsilon[\mathrm{d}(\widetilde g(x,\varepsilon),m_{\widetilde g}(x))]$, where $\mathbf{Med}_\varepsilon$ denotes the median with respect to the random variable $\varepsilon$. Then $\sigma_{\widetilde g_\alpha}(x)=\sigma_{\widetilde g}(x)$ $\mathbf{P}_X$-almost everywhere, i.e., the conditional circular dispersion is rotationally invariant. 
\item[(D)] \label{lem:D} Assume that the conditional density of $Y|x$ is square-integrable on $[0, 2\pi)$. Following the Fourier inversion theorem for circular variables, we define the truncated conditional density representation of the population engressor as $ \mathcal{P}_{\widetilde g}(\theta|x)=(2\pi)^{-1}\sum_{r=-R}^{R}{\psi_{\widetilde g}(r|x)\exp(-\iota r\theta)}$ where $\psi_{\widetilde g}(r|x)=\mathbf{E}_\varepsilon[\exp(\iota r \widetilde g(x,\varepsilon))]$, and $R>0$ is a specified integer-valued truncation level. Then $\mathcal{P}_{\widetilde g_\alpha}(\theta|x)=\mathcal{P}_{\widetilde g}(\theta\ominus \alpha|x)$ for all $\theta\in [0, 2\pi)$, $\mathbf{P}_X$-almost everywhere.
        
\item[(E)] \label{lem:E}  Under the same assumption stated in part (D), define the conditional mode and antimode respectively as $ \bar{\upsilon}_{\widetilde g}(x)= \argmax_{\theta\in [0, 2\pi)}\left\{\mathcal{P}_{\widetilde g}(\theta|x)\right\}$ and $\underline{\upsilon}_{\widetilde g}(x)= \argmin_{\theta\in [0, 2\pi)}\left\{\mathcal{P}_{\widetilde g}(\theta|x)\right\}$. Then $\bar{\upsilon}_{\widetilde g_\alpha}(x)=\bar{\upsilon}_{\widetilde g}(x)\oplus\alpha$ and $\underline{\upsilon}_{\widetilde g_\alpha}(x)=\underline{\upsilon}_{\widetilde g}(x)\oplus\alpha$. For notational simplicity, we may assume uniqueness of the mode and antimode; however, the rotational equivariance remains valid for each mode in the non-unique case.
\end{enumerate}
\end{lemma}


Proof of \Cref{lem:point_predictions_without_origin} is provided in Appendix \ref{sec:proofs_of_main_results}. Next, following \cite{di2016nonparametric}, we define the conditional circular cumulative distribution function (CDF) estimator as
\begin{align*} 
    F_{\widetilde g}(\theta|x)=\int_{0}^{\theta}\mathcal{P}_{\widetilde g}(u|x)du,\quad
    \theta\in[0,2\pi),
\end{align*}
where the conditional density $\mathcal{P}_{\widetilde g}(\cdot|x)$ is defined in part (D) of \Cref{lem:point_predictions_without_origin}. This definition depends on a choice of reference direction, which is unavoidable for CDF-type summaries on the circle. Following the same line of work, we define the conditional circular quantile at level $100(1-\alpha)\%$, for $\alpha\in(0,1)$, as 
\begin{align} \label{eq:circular_quantile}
q_\alpha(x) = \argmin_{q\in[0,2\pi)}\mathbf{E}_\varepsilon\left[\rho_\alpha\left(\widetilde g(x,\varepsilon), q\right)\right],
\end{align} where the circular check loss function is given by
\begin{align}
    \rho_\alpha(\theta_1, \theta_2) = 
    \begin{cases}
        \alpha \delta(\theta_1, \theta_2), &\text{if } \delta(\theta_1, \theta_2) \geq 0,\\
        -(1-\alpha) \delta(\theta_1, \theta_2), &\text{if } \delta(\theta_1, \theta_2) < 0
    \end{cases},
\end{align}
 where $\delta$ is the signed shortest angular distance on the circle, wrapped to the interval $[-\pi, \pi)$. Further properties of the circular CDF and quantile estimators, including rotational equivariance, can be established following similar arguments as in \cite{di2016nonparametric}. Finite sample consistency and convergence properties of these estimators require a separate asymptotic analysis and are beyond the scope of the present work. Notably, the rotational equivariance results established in \Cref{lem:point_predictions_without_origin} rely on the assumption of correct model specification, which is an idealized population-level condition. In practice, however, the expressive power of neural network generators such as ANGLE allows the approximation error to be made arbitrarily small by increasing model capacity. Nevertheless, model misspecification may still arise due to finite network capacity, optimization error, or the choice of latent noise representation.

\subsection{Sufficient Dimension Reduction}
\label{sec:application_sdr}

In many circular regression problems, the covariates are high dimensional, making sufficient dimension reduction (SDR) an important objective. Unlike unsupervised methods such as PCA, which ignore the response and may discard information relevant for prediction \citep{artemiou2009principal}, SDR seeks low-dimensional representations of the covariates that preserve all information about the response. Although SDR has recently been studied for general non-Euclidean responses (see \cite{huang2025frechet} among others), relatively little work has explicitly developed the SDR methodology tailored to circular-valued responses. In this section, we show that ANGLE provides a flexible and principled framework to this end. 

Given the response-covariate pair $(Y,X)\in [0,2\pi)\times \mathcal{X}$, the task of SDR is to find a sub-sigma field $\mathfrak{G}_{Y|X}\subseteq\sigma(X)$, such that 
\begin{align}
    \label{eq:sdr_definition}
    Y\independent X|\mathfrak{G}_{Y|X},
\end{align}
i.e., $Y$ and $X$ are independent conditioned on the information contained in $\mathfrak{G}_{Y|X}$. Note that the choice of $\mathfrak{G}_{Y|X}$ in \eqref{eq:sdr_definition} is not unique. The countable intersection of all such choices, termed as the central sigma field, is denoted as $\mathfrak{G}^*_{Y|X}$. The existence result for $\mathfrak{G}^*_{Y|X}$ (Theorem \ref{thm:central_sigma_field_exists}) is a direct analogue of the corresponding SDR result for general responses in \citet[Theorem 1]{lee2013general}; since circular responses take values in a standard measurable space, the proof carries over to the present setting under the same regularity condition. We include the proof in Appendix \ref{sec:proofs_of_main_results} for completeness.
\begin{theorem}
    \label{thm:central_sigma_field_exists}
    Assume that the family of probability measures $\left\{\mathbf{P}_{X|Y}(\cdot|y)|y\in \mathbb{S}^1\right\}$ exists and is dominated by a $\sigma$-finite measure on $\mathcal{X}$. Then there exists a unique minimal sufficient sub-sigma field $\mathfrak{G}^*_{Y|X}$ of $\sigma(X)$ such that $Y\independent X|\mathfrak{G}_{Y|X}$, and if $\mathfrak{G}_0$ is another sub-sigma field satisfying \eqref{eq:sdr_definition}, then $\mathfrak{G}^*_{Y|X}\subseteq\mathfrak{G}_0$, upto $\mathbf{P}_X$-null sets.
\end{theorem}
Typically, $\mathfrak{G}_{Y|X}$ is abstract and difficult to estimate directly. However, in the specific model structures considered in \eqref{eq:model_classes}, we exploit the single-index form $Y=f(\beta^\top X,\eta)$ with $\eta\independent X$, to convert this infinite dimensional problem of estimating $\mathfrak{G}^*_{Y|X}$ to a finite dimensional problem of estimation of the true parameter vector $\beta$.  Under suitable identifiability conditions (see Assumptions \ref{assmp:X_support_related}-\ref{assmp:B_related} and Appendix \ref{appendix:identifiability}), this parameter $\beta$ can be uniquely recovered and the corresponding central sigma field $\mathfrak{G}^*_{Y|X}$ can be determined uniquely up to $\mathbf{P}_X$-null sets. We state this formally in \Cref{thm:central_subspace_finding_engression} and \Cref{thm:central_subspace_recovery}. 


\begin{theorem}
    \label{thm:central_subspace_finding_engression}
    Let the true data generating process (DGP) be $Y|x\sim f_0\left(\beta_0^\top x,\varepsilon\right)$, where $f_0\in \mathcal{M}$ takes the structural form $f_0(t,\varepsilon)=g(t)+h(\varepsilon)$ or $f_0(t,\varepsilon)=g(t+h(\varepsilon))$, with $\beta_0$ being the true DGP parameter. Assume that the ANGLE objective \eqref{eq:engression_population_version} is well-defined at the true parameter values, and that the map $t\mapsto f_0(t,\cdot)$ is injective on $\left\{\beta_0^\top x \mid x\in \mathcal{X}\right\}$ (i.e., the conditional law is injective with respect to the linear predictor). Then the central sigma field is given by $\mathfrak{G}_{Y|X}^*=\sigma(\beta_0^\top X)$, upto $\mathbf{P}_X$-null sets. 
\end{theorem}


Using \Cref{thm:central_subspace_finding_engression}, we now propose the following methodology for performing SDR with ANGLE. Let $f_0(x,\varepsilon)=g_0(\beta_0^\top x,h_0(\varepsilon))$ be the true DGP. We recover these parameters through ANGLE:
\begin{align}\label{eq:sdr_recovery_optimization}
    (\widetilde{g},\widetilde{\beta},\widetilde{h})\in \argmin_{g\in \mathcal{G},\,h\in \mathcal{H},\,\beta\in \mathcal{B}}\left\{\mathbf{E}\left[\mathcal{F}\left(\mathrm{d}(Y,g(\beta^\top X,h(\varepsilon)))\right)-\frac{1}{2}\mathcal{F}\left(\mathrm{d}(g(\beta^\top X,h(\varepsilon)),g(\beta^\top X,h(\varepsilon^*)))\right)\right]\right\}.
\end{align}

\begin{theorem}
\label{thm:central_subspace_recovery}
Let $X$ be absolutely continuous on $\mathbb{R}^d$ and the assumptions in \Cref{thm:central_subspace_finding_engression} are true. Let $\widetilde{\beta}$ be the index parameter obtained by  minimizing the objective \eqref{eq:sdr_recovery_optimization}. Then $\widetilde{\beta} = c\beta_0$ for some $c \neq 0$, and consequently $ \mathfrak{G}^*_{Y|X} = \sigma(\widetilde{\beta}^\top X)$.
\end{theorem}

Proofs for Theorems \ref{thm:central_sigma_field_exists}-\ref{thm:central_subspace_recovery} are provided in Appendix \ref{sec:proofs_of_main_results}. The recovery of the central sigma field proceeds similarly for both the pre- and post-ANMs in our framework, since both are estimated through the same energy score optimization. However, outside this common estimation strategy, recovering the central sigma-field in pre-additive models is generally more challenging; see \cite{chen2021ultrahigh} for related developments in the Euclidean setting. Note that the single-index mode can be extended by replacing $\beta\in\mathbb{R}^d$ with $\mathbf{B}\in \mathbb{R}^{d\times b}, b < d$, to obtain a reduced $b$-dimensional subspace through SDR. Following the dimension reduction step, one can readily fit a stage-two ANGLE on the reduced covariates for downstream tasks, or perform standard post-SDR inferential tasks \citep{kim2026valid}.

\subsection{Extrapolation}
\label{sec:identifiability_and_extrapolation}


A key advantage of engression over classical distributional regression techniques is its ability to extrapolate, i.e., to produce reliable predictions at covariate values lying outside the support of the training data. In the Euclidean case, extrapolation is facilitated by the fact that the signal $g(\beta^\top x)$ can continue to evolve continuously outside the observed covariate range \citep{shen2025engression}, while the additive noise structure remains preserved. On $\mathbb{S}^1$, however, the response necessarily wraps around the circle, so that a continuously increasing signal does not move progressively farther away from the observed responses in any meaningful geometric sense, but instead repeatedly cycles through the same angular values. To make this notion precise, we quantify the degree to which two conditional circular distributions that are indistinguishable on the training support may nevertheless differ at nearby extrapolation points, formalized through the following definition.


\begin{definition}[Distributional extrapolability]
 For a $\delta>0$, define the distributional extrapolation uncertainty as 
    \begin{align*}
        U_{\mathcal{P}}(\delta)=\sup_{\left\{x^*:\ \inf\{\abs{x^*-x}:x\in \mathcal{X}\}\leq \delta\}\right\}}\left\{\sup_{\substack{\mathbf{P},\mathbf{P}^*\in \mathcal{P}\\ \mathrm{D}(\mathbf{P,\mathbf{P}^*})=0 \ \forall x\in \mathcal{X}}}\left\{\mathrm{D}\left(\mathbf{P}(y|x^*),\mathbf{P}^*(y|x^*)\right)\right\}\right\},
    \end{align*}
  where $\mathrm{D}$ is a suitable measure of distance between probability distributions and $\mathcal{P} = \left\{\mathbf{P}(y|x)\right\}$ is a class of conditional distributions. $\mathcal{P}$ is said to be distributionally extrapolable upto $\delta_0$ if $\exists \ 0<\delta<\delta_0$ such that $U_\mathcal{P}(\delta)=0$.
\end{definition}

As identifiability of the model classes serve as a basis for extrapolability \citep{shen2025engression}, we introduce some conditions to study identifiability of \eqref{eq:model_classes}. Specifically, we take
    $\mathcal{G}=\left\{g|g\text{ satisfies \Cref{assmp:g_related}}\right\}, \text{ and } \mathcal{H}=\left\{h|h\text{ satisfies \Cref{assmp:h_related}}\right\}.$
\begin{assumption}
    \label{assmp:X_support_related}
  The training support $\mathcal{X}$ is a compact subset of $\mathbb{R}^d$.  The covariate $X$ admits a density that is bounded and strictly positive on an open convex set $\mathcal{X}' \subseteq \mathcal{X}$. 
\end{assumption}
\begin{assumption}
\label{assmp:g_related}
    \begin{enumerate}
    \item [(A)] The function $g$ is non-constant and smooth.
    \item [(B)] The support of $g$ in the model classes defined in \eqref{eq:model_classes} is compact and convex with a non-empty interior.
    \item[(C)] For model classes $\mathbf{M}_{\text{pre}}^{(1)}$ and  $\mathbf{M}_{\text{pre}}^{(2)}$ with pre-additive noise, $g$ is strictly monotone.
    \end{enumerate}
\end{assumption}
\begin{assumption}
    \label{assmp:h_related}
$h:[0,1]\rightarrow\mathbb{R}$ is a strictly monotonically increasing, bounded differentiable function with $h\left(\frac{1}{2}\right)=0$ and $\abs{h(1)-h(0)}<2\pi$.
\end{assumption}
\begin{assumption}
    \label{assmp:B_related}
    The set of coefficient vectors is $\mathcal{B} = \{\beta\in \mathbb{R}^d \ | \ \beta_{(1)} > 0, \ \|\beta\| = 1\}$.
\end{assumption}
The assumptions and identifiability of \eqref{eq:model_classes} are discussed in Appendix \ref{appendix:identifiability} in detail. We demonstrate in \Cref{thm:extrapolability_criteria} that all four model classes in \eqref{eq:model_classes} are extrapolable. 

 

\begin{theorem}
    \label{thm:extrapolability_criteria}
    Let the domain of the functions in $\mathcal{G}$ be $\mathcal{I} \subset \mathbb{R}$, which is an open connected interval. Suppose that Assumptions \ref{assmp:X_support_related}-\ref{assmp:B_related} hold and all $g\in \mathcal{G}$ are real analytic on $\mathcal{I}$. 
    \begin{enumerate}
        \item[(A)] Let $\mathfrak{T}_{\text{post}} = \{\beta^\top x \mid x \in \mathcal{X}, \beta \in \mathcal{B}\}$. Assume $\mathfrak{T}_{\text{post}} \subset \mathcal{I}$. Then there exists a $\delta_0 > 0$ such that for all $0 < \delta \le \delta_0$, $U_{\mathbf{M}_{\text{post}}^{(1)}}(\delta) = U_{\mathbf{M}_{\text{post}}^{(2)}}(\delta) = 0$.
        \item[(B)] Let $\mathfrak{T}_{\text{pre}} = \{\beta^\top x + h(\varepsilon) \mid x \in \mathcal{X}, \varepsilon \in [0,1], \beta \in \mathcal{B}, h \in \mathcal{H}\}$. Assume $\mathfrak{T}_{\text{pre}} \subset \mathcal{I}$. Then there exists a $\delta_0 > 0$ such that for all $0 < \delta \le \delta_0$, $U_{\mathbf{M}_{\text{pre}}^{(1)}}(\delta) = U_{\mathbf{M}_{\text{pre}}^{(2)}}(\delta) = 0$.
    \end{enumerate}
\end{theorem}

The proof of \Cref{thm:extrapolability_criteria} is given in Appendix \ref{sec:proofs_of_main_results}. Under the identifiability assumptions, the real analytic restriction makes the model classes (locally) distributionally extrapolable. It is useful to contrast our extrapolability result with the engression framework of \citet[Theorem 1]{shen2025engression} for classical regression setup, where functional extrapolability of post-ANMs is necessary and sufficient for distributional extrapolability. They operate under a twice-differentiable ($C^2$) assumption which does not inherently guarantee functional extrapolability, and require unbounded noise. In our circular framework, the additive noise is intrinsically bounded because of the bounded transformation $h$. To achieve extrapolability, we instead constrain our model class $\mathcal{G}$ to real analytic functions. By the identity theorem for analytic functions, agreement on the training support dictates global agreement across the connected domain $\mathcal{I}$. Consequently, the identifiability assumptions and real analytic constraint inherently renders the function class functionally extrapolable on the extended domain $\mathcal{I}$, naturally satisfying the criteria for distributional extrapolability. Another distinction concerns the covariate space itself. While \cite{shen2025engression} allow the training support $\mathcal{X}$ to be potentially unbounded, we assume that $\mathcal{X}$ is compact. Compactness is entirely consistent with practice, since observed covariates are often finite valued and necessarily lie within a bounded region. We assume that the feature space $\mathcal{I}$ extends beyond the observed training data as an open connected but bounded domain. This is a natural assumption in essentially all real applications, since covariates are always constrained by physical, geometric, or measurement limitations and therefore cannot vary arbitrarily.

\subsection{Testing Equality of Conditional Distributions}
\label{sec:application_equality_testing}
In this section, we demonstrate how the ANGLE estimator can be applied for testing of equality of conditional distributions. Consider two independent samples $\mathcal{S}_1=\{(Y_{1i},X_{1i})|i=1,2,\ldots,n_1\}$ and $\mathcal{S}_2=\{(Y_{2i},X_{2i})|i=1,2,\ldots,n_2\}$ generated from the true conditional distributions $\mathbf{P}^{(1)}_{Y|X}$ and $\mathbf{P}^{(2)}_{Y|X}$, respectively. We test 
\begin{align}\label{eq:distribution_equality_test}
    H_0:\mathbf{P}^{(1)}_{Y|X}=\mathbf{P}^{(2)}_{Y|X}\text{ vs. }H_1:\mathbf{P}^{(1)}_{Y|X}\neq\mathbf{P}^{(2)}_{Y|X}.
\end{align}
Tests of this form are widely used in various domains, especially in the context of covariate shift \citep{qiu2024efficient} and causal inference \citep{fan2024environment}. We take inspiration from the GCA-CDET test of \cite{zheng2025conditional}, adapted with minimal changes to incorporate circular responses.  The pseudoscope of the test is given in \Cref{alg:conditional_test}. The test reduces \eqref{eq:distribution_equality_test} to an unconditional two-sample problem using the idea that under $H_0$, synthetic responses generated from $\mathbf{P}^{(1)}_{Y|X}$ at covariate values from $\mathcal S_2$ should be indistinguishable from the true responses in $\mathcal S_2$. The sample split in Step 2 of \Cref{alg:conditional_test} ensures the independence required for the angular two-sample test $\mathrm{ATS}$, such as Watson's test \citep{fornasin2025effects} or Kuiper's test \citep{anichini2023measuring}. Consistency of the procedure follows from \cite{zheng2025conditional} under standard smoothness and sample size conditions, provided the test $\mathrm{ATS}$ is itself consistent. Notably, several conditional independence testing procedures originally developed for Euclidean data can, in principle, be extended to the circular setting; see \cite{ren2025score,gao2026testing,zhang2026doubly} among others. 

\begin{algorithm}[h!]
\caption{Conditional distribution equality test}
\label{alg:conditional_test}
\begin{algorithmic}[1]
\Require  $\mathcal{S}_1=\{(Y_{1i},X_{1i})\}_{i=1}^{n_1}$, $\mathcal{S}_2=\{(Y_{2i},X_{2i})\}_{i=1}^{n_2}$, level $\alpha$, circular two-sample test $\mathrm{ATS}$.

\State Fit the ANGLE network on $\mathcal{S}_1$ to obtain the generative model $\widetilde{g}_1$, such that $\widetilde{g}_1(x,\varepsilon)\sim\mathbf{P}^{(1)}_{Y|X=x}$ under correct specification.
\State Randomly partition $\mathcal{S}_2$, independently of $\mathcal{S}_1$, into two equal halves $\mathcal{S}_{21}=\{(Y_{21,i},X_{21,i})\}_{i=1}^{n_2/2}$ and $\mathcal{S}_{22}= \{(Y_{22,i},X_{22,i})\}_{i=1}^{n_2/2}.$
\For{$i=1,\ldots,n_2/2$}
    \State Draw noise $\varepsilon_i \sim P_\varepsilon$ independently of $\mathcal{S}_1$ and $\mathcal{S}_2$.
     \State Generate synthetic circular response $\widetilde{Y}_{21,i}=\widetilde{g}_1(X_{21,i}, \varepsilon_i)$.
\EndFor

\State Form the synthetic dataset $\widetilde{\mathcal{S}}_{21}=\{(\widetilde{Y}_{21,i},X_{21,i})\}_{i=1}^{n_2/2}$. Under correct specification and sufficiently large sample sizes, $\widetilde{\mathcal{S}}_{21}\sim\mathbf{P}^{(1)}_{Y|X}\otimes\mathbf{P}^{(2)}_X$ and $\mathcal{S}_{22}\sim\mathbf{P}^{(2)}_{Y|X}\otimes\mathbf{P}^{(2)}_X$. Consequently, under $H_0$, the datasets $\widetilde{\mathcal{S}}_{21}$ and $\mathcal{S}_{22}$ are approximately identically distributed.
\State \Return $\mathrm{ATS}\big(\widetilde{\mathcal{S}}_{21},\mathcal{S}_{22},\alpha\big)$ and reject $H_0$ if the $p$-value $< \alpha$.
\end{algorithmic}
\end{algorithm}

\section{Simulation Experiments}
\label{sec:simulations}
We conduct a comprehensive simulation study across varied settings to empirically assess the predictive accuracy and extrapolative properties of ANGLE. The following sections detail the data generation procedures, model specification, and empirical results for standard predictive tasks - specifically, the estimation of the circular conditional mean and distribution. The simulation study for extrapolation is subsequently presented in \Cref{appendix:simulation_extrapolation}. 

\subsection{Circular Data Generation} \label{sec:circular_data_generation_sim}
We generate synthetic circular data using four approaches derived from two base mapping mechanisms: stereographic projection (arctangent link) and the Projected Normal (PN) distribution \citep{bhuyan2026modeling}, with linear and nonlinear latent structures. Euclidean and circular covariates are sampled from standard normal and von Mises distributions, respectively, and perturbed with pre-additive uniform or Gaussian noise. Each mechanism is evaluated under four configurations varying covariate dimensions and structural parameters, yielding 16 simulation settings in total. Each setting contains 2,000 training and 200 test observations. All generation settings are summarized in \Cref{table:simulation_parameter_settings} (Appendix \ref{appendix:tables_figures}).

\subsubsection{Approach 1 and 2: Arctangent Link}
We define the linear and circular covariates as $\mathbf{X}_L \sim \mathcal{N}(0,1)^{n \times p_L}$ and $\boldsymbol{\Theta}_C \sim \mathcal{VM}(0, 1)^{n \times p_C}$, respectively, where $n$ denotes the number of data points. To simulate pre-additive measurement uncertainty, we introduce independent noise matrices $\mathbf{E}_L$ and $\mathbf{E}_C$ with elements drawn from $\mathcal{U}(-\sigma, \sigma)$ or $\mathcal{N}(0, \sigma^2)$. The perturbed covariates are given by $\widetilde{\mathbf{X}}_L = \mathbf{X}_L + \mathbf{E}_L$ and $\widetilde{\boldsymbol{\Theta}}_C = (\boldsymbol{\Theta}_C + \mathbf{E}_C) \pmod{2\pi}$. To bypass boundary discontinuities, the noisy circular features are trigonometrically embedded, yielding the unified design matrix $\widetilde{\mathbf{X}} = [\widetilde{\mathbf{X}}_L, \cos(\widetilde{\boldsymbol{\Theta}}_C), \sin(\widetilde{\boldsymbol{\Theta}}_C)] \in \mathbb{R}^{n \times (p_L + 2p_C)}$. Given independent coefficient vectors $\beta_1, \beta_2 \in \mathbb{R}^{p_L + 2p_C}$, we compute the base linear combinations $\mathbf{z}_k = \widetilde{\mathbf{X}}\beta_k$ for $k \in \{1, 2\}$. To govern functional complexity, a scalar parameter $a$ introduces coupled non-linear cross-terms to form the latent signals:$$ \boldsymbol{\xi}_1 = \mathbf{z}_1 + a \tanh(\mathbf{z}_2), \quad \boldsymbol{\xi}_2 = \mathbf{z}_2 + a \tanh(\mathbf{z}_1) $$We set $a=0$ for a strictly linear latent mapping (Approach 1) and $a=4$ to induce strong non-linear interactions (Approach 2). Treating $\boldsymbol{\xi}_1$ and $\boldsymbol{\xi}_2$ as orthogonal Cartesian coordinates, we project these unbounded signals onto the circular support $[0, 2\pi)^n$ via a full-quadrant arctangent transformation:$$ \boldsymbol{\theta} = \mathrm{atan2}(\boldsymbol{\xi}_2, \boldsymbol{\xi}_1) \bmod{2\pi}. $$This methodology establishes a robust data-generating process that natively accommodates heterogeneous predictors under specified noise and varying structural non-linearities.

\subsubsection{Approach 3 and 4: Projected Normal Distribution}
Building upon the unified design matrix $\widetilde{\mathbf{X}} \in \mathbb{R}^{n \times p}$ (where $p = p_L + 2p_C$), our second data-generating mechanism derives from the PN distribution \citep{bhuyan2026modeling}. Given independent coefficient vectors $\beta_1, \beta_2 \in \mathbb{R}^p$, we compute base linear combinations $\mathbf{z}_k = \widetilde{\mathbf{X}}\beta_k$ for $k \in \{1, 2\}$. To introduce structural non-linearity, the latent Euclidean mean vectors are coupled via a scalar parameter $a$:$$ \boldsymbol{\mu}_1 = \mathbf{z}_1 + a \tanh(\mathbf{z}_2), \quad \boldsymbol{\mu}_2 = \mathbf{z}_2 + a \tanh(\mathbf{z}_1).$$
Again, we set $a=0$ for a strictly linear latent mapping (Approach 3) and $a=2$ to induce non-linear cross-interactions (Approach 4). To capture the intrinsic structural uncertainty of the isotropic PN distribution, bivariate Cartesian targets are sampled as $\mathbf{Y}_k \sim \mathcal{N}(\boldsymbol{\mu}_k, \mathbf{I}_n)$, where $k\in\{1,2\}$ and $\mathbf{I}_n$ denotes the $n\times n$ identity matrix. Finally, these unbounded coordinates are radially projected onto the standard circular domain:
$$ \boldsymbol{\theta} = \mathrm{atan2}(\mathbf{Y}_2, \mathbf{Y}_1) \bmod{2\pi}.$$
This procedure implicitly models the conditional response as $\theta_i | \tilde{\boldsymbol{x}}_i \sim \mathcal{PN}_2(\boldsymbol{\mu}_i, \mathbf{I}_2)$.

\subsubsection{Data Generation Settings and Parameters}
Details such as number of covariates and the corresponding parameters governing these simulations are given in Table \ref{table:simulation_parameter_settings} (Appendix \ref{appendix:tables_figures}). For both the arctangent and PN formulations, the true regression coefficient vectors $\beta_k$ expand to dimension $p_L + 2p_C$ to accommodate the trigonometric embedding. Specifically, the first $p_L$ elements govern the linear covariates, while the remaining $2p_C$ elements serve as paired weights for the cosine and sine components of the circular features. For high-dimensional configurations (Settings 1.4, 2.4, 3.4, and 4.4) where explicit specification is intractable, the 125-dimensional coefficient vectors are drawn independently from $\mathcal{U}(-10, 10)$. Fig.  \ref{fig:Training_Data_Arctan_Setting_1.1} illustrates the first setting (1.1) with linear arctangent link and 2 Euclidean covariates; analogous plots for the remaining settings are omitted for brevity.

\begin{figure}
        \centering
        \includegraphics[width=\linewidth]{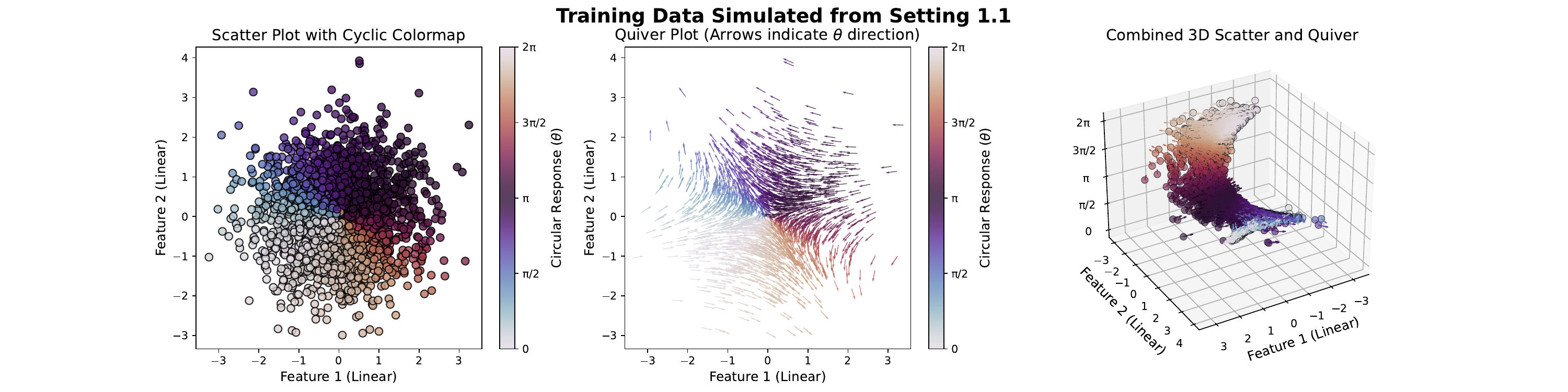}
        \caption{{\small 2000 training points generated from setting 1.1; Scatter plot (left), quiver plot with arrows depicting the direction of the observations in 2D (middle), and the same in 3D (right).}}
        \label{fig:Training_Data_Arctan_Setting_1.1}
\end{figure}

\subsection{Model Specification and Simulation Results} \label{sec:simulation_model_specification}

 All simulated datasets are fitted using a 3-layer (pre-ANM) ANGLE network with hidden dimension 100 and noise dimension 64. To evaluate the effect of geometry on distributional learning, we train separate models using geodesic and chordal distance-based energy score losses. Training is performed for 500 epochs with learning rate 0.05 and no batch normalization. Performance is assessed using mean absolute angular deviation (MAAD), circular mean directional error (CMDE or cosine loss), and circular continuous ranked probability score (CRPS; see Appendix \ref{appendix:metrics} for definitions). Since ANGLE is generative, inference is done based on 100 samples per test point. Further, the point metrics are computed from the circular mean, while CRPS uses the full predictive distribution. We compare against the baseline models listed in Table \ref{table:baseline_models}, whose architectural and implementation details are deferred to Appendix \ref{appendix:Baseline_models}. We exclude Bayesian implementations due to their substantially higher computational cost; however, Bayesian von-Mises Quasi Processes (vMQP; \cite{cohen2025bayesian}) is included for comparison during real data analysis.



{\tiny
\begin{longtable}[]{@{}p{0.3\linewidth}cccccccc@{}}
\caption{{\small Comparison of baseline models for circular data. Columns indicate native support for linear/circular covariates, input/output ranges, intrinsic uncertainty quantification (UQ), computational efficiency, and evaluation tasks (simulations, object pose estimation (OPE), and wind direction prediction (WDP)).}}\label{table:baseline_models}
\tabularnewline
\toprule\noalign{}
Model & \makecell{Linear\\covariates} & \makecell{Circular\\covariates} & \makecell{Input/output\\range} & UQ & Lightweight & Simulations & OPE & WDP \\
\midrule\noalign{}
\endfirsthead
\toprule\noalign{}
Model & Linear covariates & Circular covariates & Input/output range \\
\midrule\noalign{}
\endhead
\bottomrule\noalign{}
\endlastfoot
Circular Linear Regression (CLR; \cite{presnell1998projected}) & \cmark & \xmark & $(-\pi, \pi]$ & \xmark & \cmark & \cmark & \xmark & \cmark\\ 
Nonparametric Kernel CLR \citep{taylor2012non} & \cmark & \cmark & $(-\pi, \pi]$ & \xmark & \cmark & \cmark & \xmark & \cmark \\
SCBR-NLL \citep{hassanzadeh2021smoothing} & \cmark & \xmark & $[0, 2\pi)$ & \xmark & \cmark & \cmark & \xmark & \xmark \\
SCBR-MCPE \citep{hassanzadeh2021smoothing} & \cmark & \xmark & $[0, 2\pi)$ & \xmark & \cmark & \cmark & \cmark & \cmark \\
Lifted EP-MGvM \citep{wu2019probabilistic} & \cmark & \xmark & $(-\pi, \pi]$ & \xmark & \cmark & \cmark & \cmark & \cmark \\
MCR2 \citep{jha2017multiple} & \xmark & \cmark & $[0, 2\pi)$ & \xmark & \xmark & \cmark & \cmark & \cmark \\
vMQP \citep{cohen2025bayesian} & \cmark & \xmark & $(-\pi, \pi]$ & \cmark & \xmark & \xmark & \xmark & \cmark \\
Parametric Mixture Models \citep{prokudin2018deep} & - & - & $[0,2\pi)$ & \cmark & \xmark & \xmark & \cmark & \xmark \\
\textbf{ANGLE (Proposed)} & \cmark & \cmark & $[0, 2\pi)$ & \cmark & \cmark & \cmark & \cmark & \cmark \\
\end{longtable}
}

\begin{figure}
        \centering
        \includegraphics[width=0.7\linewidth]{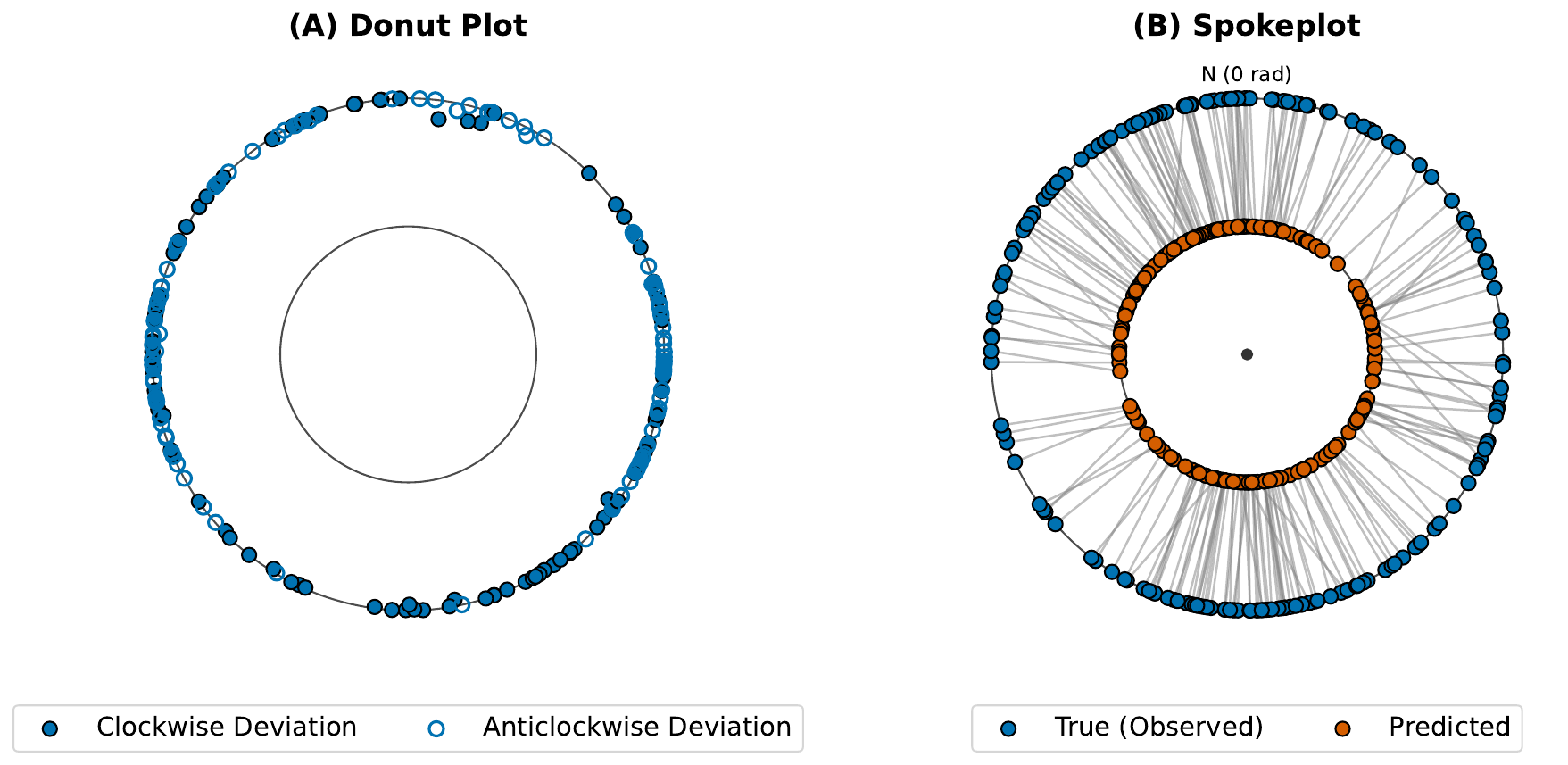}
        \caption{{\small (A) Donut plot and (B) Spokeplot of ANGLE-Chordal predictions on simulated data, setting 1.1. In (A), points near the outer circumference (radius = 2) indicate good fit; points inside the unit disc indicate poor fit. Filled and empty circles denote clockwise and anticlockwise deviations, respectively.}}
        \label{fig:simulations_errors}
\end{figure}

\begin{figure}[!h]
        \centering
        \includegraphics[width=0.6\linewidth]{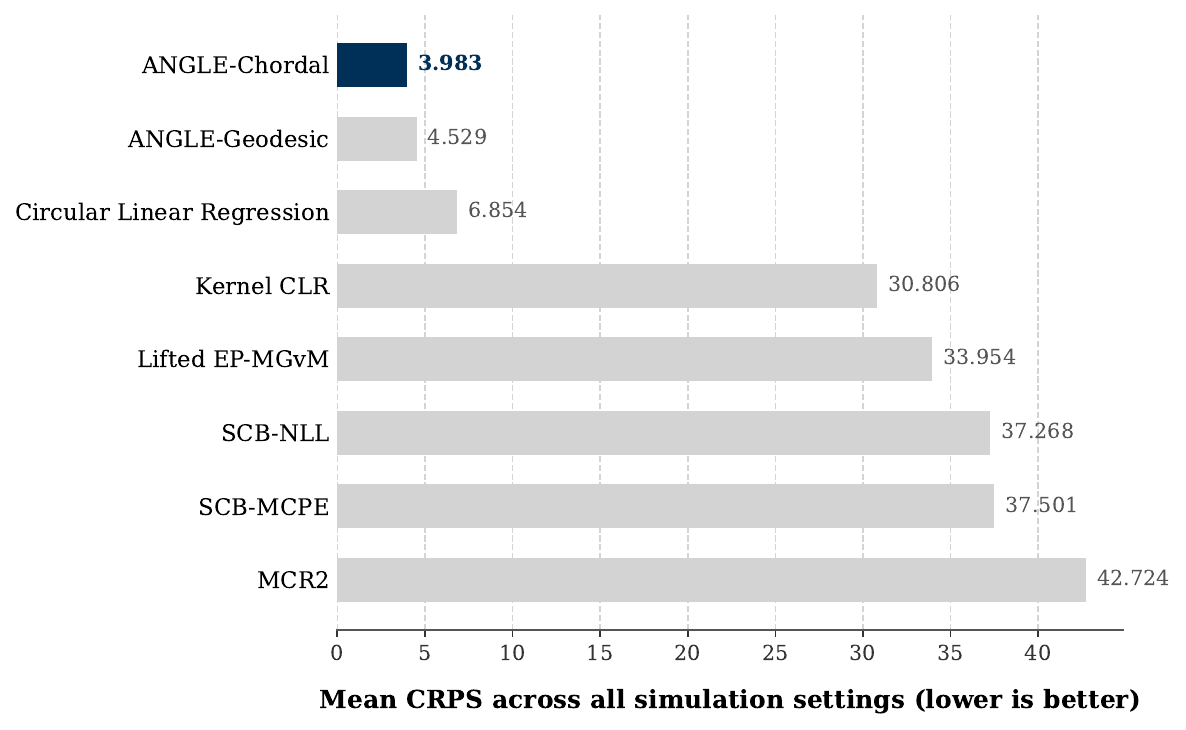}
        \caption{Comparison of model performance on simulated datasets using mean CRPS across all 16 simulation settings.}
        \label{fig:simulation_crps_comparison}
\end{figure}

Table \ref{table:simulation_results} (Appendix \ref{appendix:tables_figures}) reports the mean and standard deviation of the evaluation metrics over 200 independent trials. Across most settings, the energy score based on the chordal distance achieves the best performance. Predictions by ANGLE on the simulated test set of setting 1.1 is presented through (A) donut plot \citep{jha2017multiple} and (B) spokeplot in Fig. \ref{fig:simulations_errors}. For fairness, circular covariates were represented by their sine and cosine embeddings for baseline methods that do not natively support directional predictors. Overall, ANGLE consistently outperforms all competitors (see Fig. \ref{fig:simulation_crps_comparison}) in terms of CRPS across all 16 simulation settings. While circular linear regression (CLR) remains competitive under linear data-generating mechanisms (Approaches 1 and 3), ANGLE substantially outperforms it in the nonlinear settings (Approaches 2 and 4), highlighting the advantage of neural networks in capturing complex nonlinear relationships.




\subsection{Simulations for Extrapolation} \label{appendix:simulation_extrapolation}
\Cref{thm:extrapolability_criteria} establishes the extrapolability of the  model classes \eqref{eq:model_classes} under some identifiability conditions and the assumption that the generative function $g$ is real analytic. To achieve extrapolation, we qualitatively evaluate the implementations of $\mathbf{M}_{\text{pre, post}}^{(1)}$, and $\mathbf{M}_{\text{pre, post}}^{(2)}$ with scaled sigmoid output. For $\mathbf{M}_{\text{pre, post}}^{(1)}$, by defining $g: \mathbb{R} \to \mathbb{R}$ as an unbounded scalar output and applying the modulo operator post-hoc during inference, the network is permitted to grow linearly outside the training support, which maps to a continuous `wrapping' behavior on the circular manifold. Theoretically, $\mathbf{M}_{\text{pre}}^{(1)}$ is not fully identifiable due to the $2k\pi$ equivalence class, but is extrapolable under the real analytic assumption (Theorem \ref{thm:extrapolability_criteria}).


\begin{figure}[!ht]
        \centering
        \includegraphics[width=\linewidth]{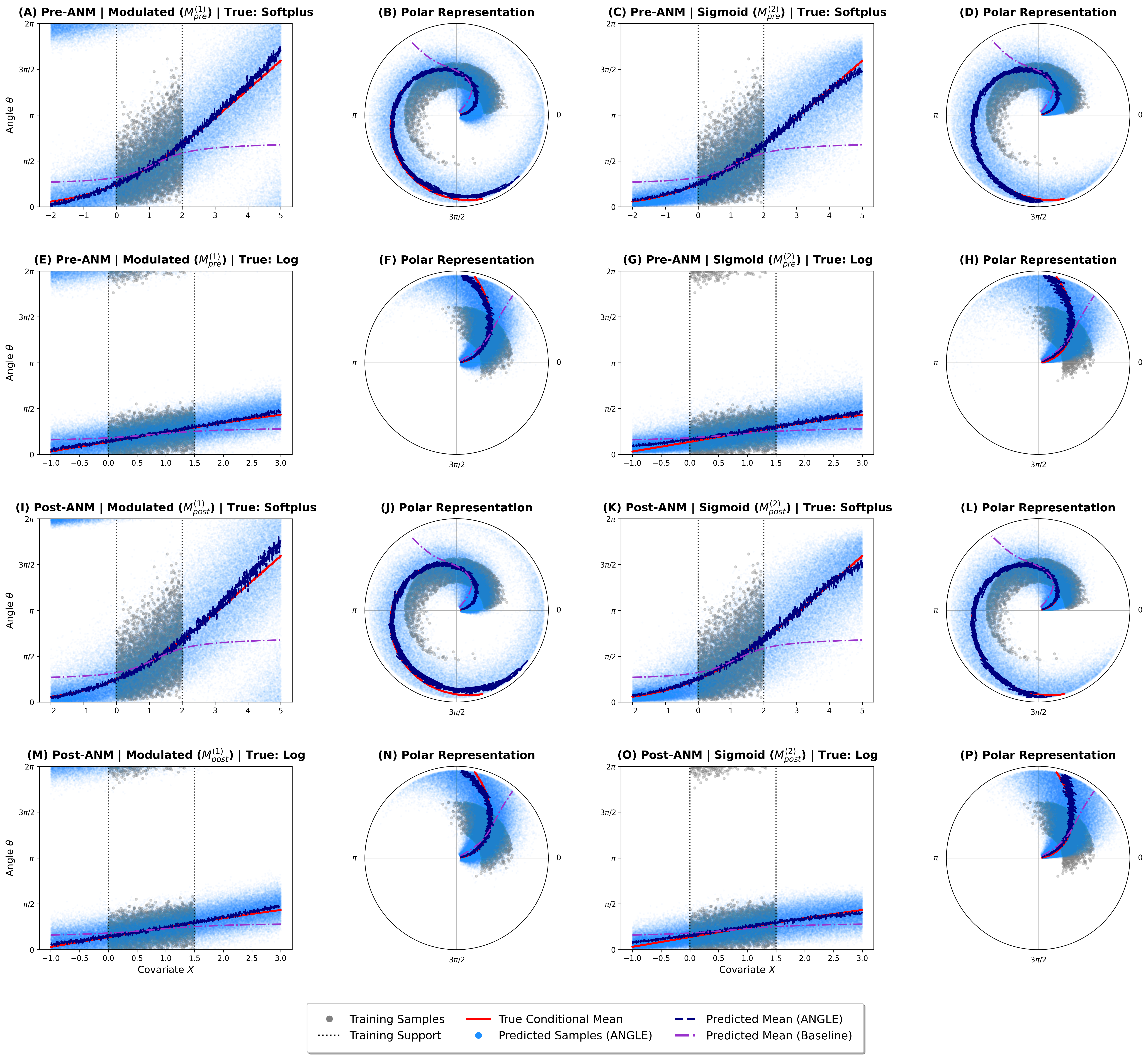}
        \caption{Simulation results for extrapolation with ANGLE: unbounded output with $2\pi$ modulation (first two columns) and output bounded in $[0,2\pi)$ with scaled sigmoid (last two columns). As a baseline, the kernel CLR is considered.}
        \label{fig:extrapolation_simulations}
\end{figure}


We demonstrate extrapolability empirically using simulated pre-ANM data with one predictor in Fig. \ref{fig:extrapolation_simulations}. 5000 training covariates $X\in\mathbb{R}$ are sampled uniformly from a bounded support and perturbed with pre-additive Gaussian noise, $\eta \sim \mathcal{N}(0, 1)$. These noisy inputs are passed through two base functions - linearly asymptotic $g = \mathrm{Softplus}$, and sub-linear $g = \mathrm{Log}$, and projected onto the circular manifold to generate the angular response $Y$. The true function $g$ is chosen to be strictly monotonic, i.e., not wrapping around the circle. To establish a ground-truth extrapolation baseline, the exact conditional circular mean is numerically approximated over an extended evaluation grid via Monte Carlo integration utilizing 10,000 noise samples per grid point. The training and evaluation supports for the $\mathrm{Softplus}$ base function are defined as $[0,2]$ and $[-2,5]$, respectively; for the logarithmic configuration, these boundaries are restricted to $[0, 1.5]$ and $[-1,3]$. We visualize the extrapolation fits by training ANGLE over the simulated datasets under both pre-ANM (panels (A)-(H)) and post-ANM (panels (I)-(P)) formulations. Within these plots, the training samples are denoted by gray scatter points bounded within the training support. To approximate the model's predicted conditional distribution, we generate 100 samples from the trained ANGLE at each evaluation coordinate, illustrated by the blue scatter clouds. The theoretical ground-truth conditional mean and ANGLE's predicted circular mean are superimposed as solid red and dashed navy lines, respectively. The dash-dotted indigo line represents the fit of kernel CLR, which we consider as a baseline.

It is visually evident that although the baseline fit (indigo) closely approximates the true conditional mean (red) within the training support, it asymptotes outside the support and fails to extrapolate. Conversely, the proposed ANGLE with the $\mathbf{M}_{\text{post,pre}}^{(1)}$-like architecture smoothly preserves the structural trend of the data, providing a better approximation of the true conditional mean, both inside and beyond the training support (first two columns in Fig. \ref{fig:extrapolation_simulations}). Finally, extrapolation capabilities of the $\mathbf{M}_{\text{post,pre}}^{(2)}$-like architecture with scaled sigmoid is presented in the last two columns in Fig. \ref{fig:extrapolation_simulations}, where the predicted circular mean (dashed navy line) closely aligns with the true conditional mean (red line).

\section{Real Data Applications}
\label{sec:applications}
To empirically validate the practical utility of our proposed framework for prediction with uncertainty quantification, sufficient dimension reduction, extrapolation, and testing equality of conditional distributions, we conduct comprehensive evaluations across both image and tabular data modalities. For prediction, we address two distinct tasks: object pose estimation from images and location-based wind direction prediction. For the former, we utilize the PASCAL3D+ benchmark dataset \citep{xiang2014beyond} comprising 36,697 images across 12 object categories with annotated ground-truth orientations. The proposed architecture for this task combines a fine-tuned visual encoder with a class-specific ANGLE head for targeted object pose estimation. For wind direction prediction, we evaluate the model using two distinct meteorological datasets: the German wind direction dataset previously analyzed by \cite{cohen2025bayesian} and district-level mean wind directions in India for December 2025. The Indian wind dataset is also used to showcase extrapolability of ANGLE. The proposed models are implemented in fully dense mode with $\mathrm{atan2}$ in the final layer. Point predictions are obtained by computing the circular mean of samples simulated from the estimated ANGLE object. To ensure the statistical robustness of our findings, all experimental setups are executed over 50 independent trials, with performance metrics reported as the empirical mean and standard deviation.

\subsection{Object Pose Estimation} \label{sec:object_pose_estimation}
Pose estimation is a fundamental task in computer vision, which deals with predicting the direction towards which an object in a given image or video is pointing. It is an important building block in systems that aim to understand complex scenes and arises in many applications such as estimation of head pose \citep{prokudin2018deep}, gaze direction \citep{nonaka2022dynamic}, object class orientation, and so on. We evaluate our approach for pose detection of objects on the PASCAL3D+ dataset \citep{xiang2014beyond}, a rigorous benchmark for pose estimation comprising 12 object categories (see Table \ref{table:pascal_statistics}) sourced from PASCAL VOC 2012 and ImageNet \citep{deng2009imagenet}. The dataset provides 3D annotations including angles for azimuth, elevation, and distance. We model the azimuth, which denotes the horizontal rotation angle defining the object's orientation relative to a $0^\circ$ reference. Consistent with prior literature \citep{prokudin2018deep}, we utilize the ImageNet and PASCAL training images for training, ImageNet validation images for validation, and PASCAL validation images for testing (data split statistics are provided in Table \ref{table:pascal_statistics}). Notably, we analyze the raw images without applying any prior preprocessing and data augmentation.

{\tiny 
\setlength{\tabcolsep}{5.5pt} 
\begin{longtable}[]{@{}lccccccccccccc@{}}
\caption{{\small Number of images per class in train, validation, and test splits of the PASCAL3D+ dataset.}}\label{table:pascal_statistics}\tabularnewline
\toprule\noalign{}
~ & aeroplane & bicycle & boat & bottle & bus & car & chair &
diningtable & motorbike & sofa & train & tvmonitor & Total \\
\midrule\noalign{}
\endhead

Train & 1467 & 1094 & 1843 & 1515 & 865 & 4043 & 1996 & 1533 & 1019 & 1150 & 989 & 1041 & 18555 \\
Validation & 979 & 658 & 1262 & 761 & 536 & 2780 & 514 & 1153 & 613 & 707 & 646 & 623 & 11232 \\
Test & 484 & 380 & 491 & 733 & 320 & 1173 & 1449 & 374 & 376 & 387 & 329 & 414 & 6910 \\
\midrule\noalign{}
Total class-wise & 2930 & 2132 & 3596 & 3009 & 1721 & 7996 & 3959 & 3060 & 2008 & 2244 & 1964 & 2078 & 36697 \\
\bottomrule\noalign{}
\endlastfoot
\end{longtable}
}

Since the proposed ANGLE expects data in the form of arrays or tensors, feature extraction is performed using state-of-the-art visual encoders: Inception-v3 \citep{szegedy2016rethinking} and ConvNeXt \citep{liu2022convnet}, initialized with pre-trained ImageNet weights from \textit{torchvision}\footnote{https://github.com/pytorch/vision}. As classification-based pre-training encourages rotation invariance, we fine-tune these models using PASCAL3D+ azimuth annotations to extract rotation-sensitive representations. Hence, the object pose estimation task is divided into two stages: (1) fine-tuning pre-trained visual encoder to perform as feature extractor and (2) fitting an ANGLE model on the extracted features. For the first stage, we replace the standard classification blocks of the image models with a custom multi-class biternion regression head. To circumvent the periodic discontinuity of angular targets at $0^\circ / 360^\circ$, we encode the ground-truth azimuth angle $\theta$ as a continuous 2D directional vector, or biternion, $v = [\cos\theta, \sin\theta]^\top$. The modified final linear layer maps the extracted image features to a tensor of size $C \times 2$, where $C=12$ is the total number of target object classes. A feature-wise $L_2$-normalization is then applied to the outputs to ensure the predicted 2D vectors lie on the unit circle. During fine-tuning, we employ a conditional routing strategy: the network extracts the predicted biternion $\hat{v}_c$ corresponding strictly to the object class $c$. The entire network is fine-tuned end-to-end by minimizing the cosine distance loss between the predicted and target biternions: $\mathcal{L} = 1 - \hat{v}_c \cdot v$. Optimization is performed for 20 epochs using Adam with a learning rate of $1 \times 10^{-4}$ and a weight decay of $0.01$. During training, we utilize the validation images for early stopping. 

A single global encoder processes images across all 12 classes, generating 2048-dimensional (Inception-v3) or 768-dimensional (ConvNeXt) feature vectors. In the second stage, these features are subsequently passed to the ANGLE framework. While feature extraction is global, we train independent, class-specific models to perform targeted pose estimation. We compare our approach against the parametric mixture models of \cite{prokudin2018deep} and standard baselines reported in Table \ref{table:baseline_models}. CLR and its kernel variant are excluded due to their inability to scale to high-dimensional regimes where features outnumber samples. Evaluation metrics include MAAD, CMDE, circular CRPS, $\text{Accuracy}_{\pi/6}$, and median error (MedErr; see Appendix \ref{appendix:metrics} for the mathematical formulations). Quantitative class-wise comparisons are presented in \Cref{table:object_pose_results} (Appendix \ref{appendix:tables_figures}), and the test-set predictions with 95\% intervals are visualized in Fig. \ref{fig:object_poses}. The prediction intervals are constructed by extracting the 0.025 and 0.975 circular quantiles using \eqref{eq:circular_quantile} from 100 predictive samples from the estimated conditional distribution of azimuth angles. Fig. \ref{fig:Objectpose_Acc_CRPS} illustrates the mean scores obtained by the models on $\text{Accuracy}_{\pi/6}$ and CRPS across 12 classes. The proposed model with ConvNeXt as the visual encoder (ConvNeXt-ANGLE) achieves the best mean scores across all 12 classes for all metrics, followed by Inception-v3-ANGLE.

\begin{figure}[!t]
        \centering
        \includegraphics[width=0.8\linewidth]{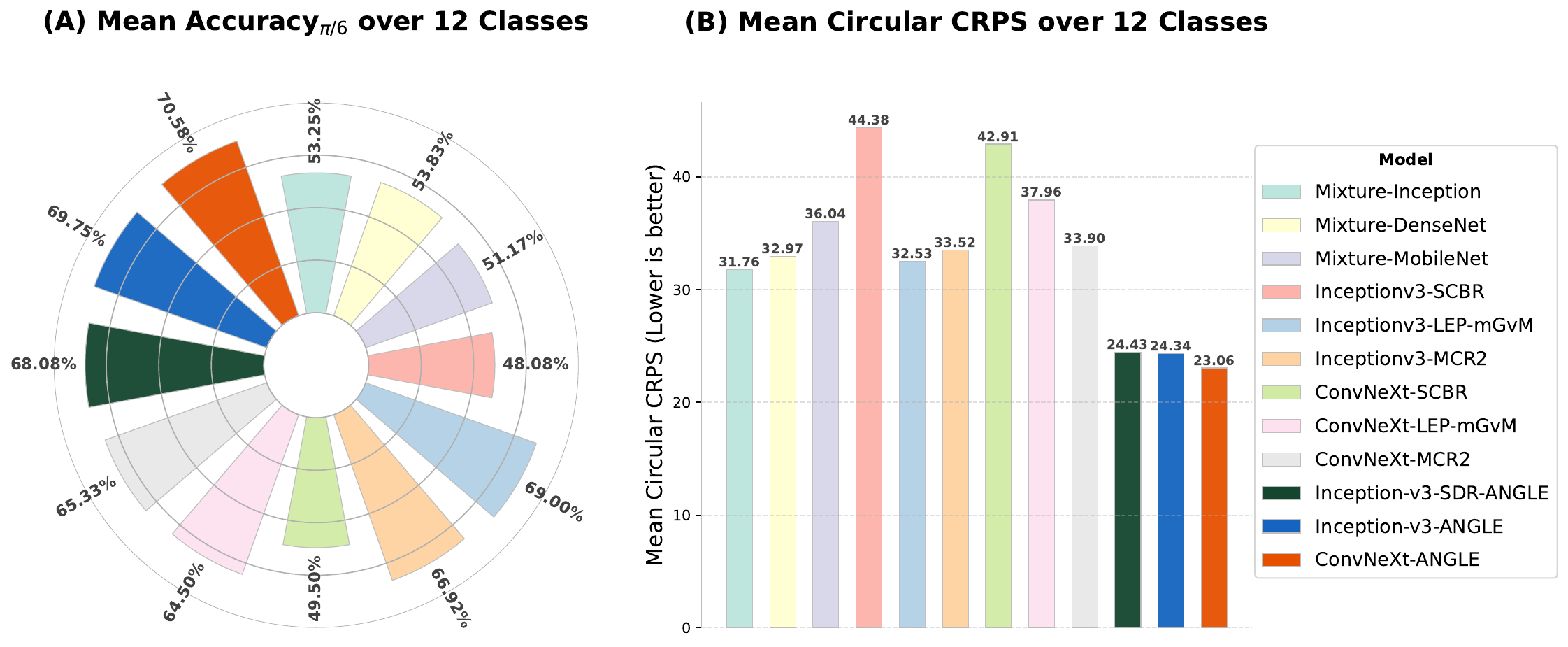}
        \caption{{\small Mean (A) $\text{Accuracy}_{\pi/6}$ and (B) CRPS obtained by the proposed and baseline models across all 12 classes on the PASCAL3D+ test set.}}
        \label{fig:Objectpose_Acc_CRPS}
\end{figure}

\begin{figure}[]
        \centering
        \includegraphics[width=\linewidth]{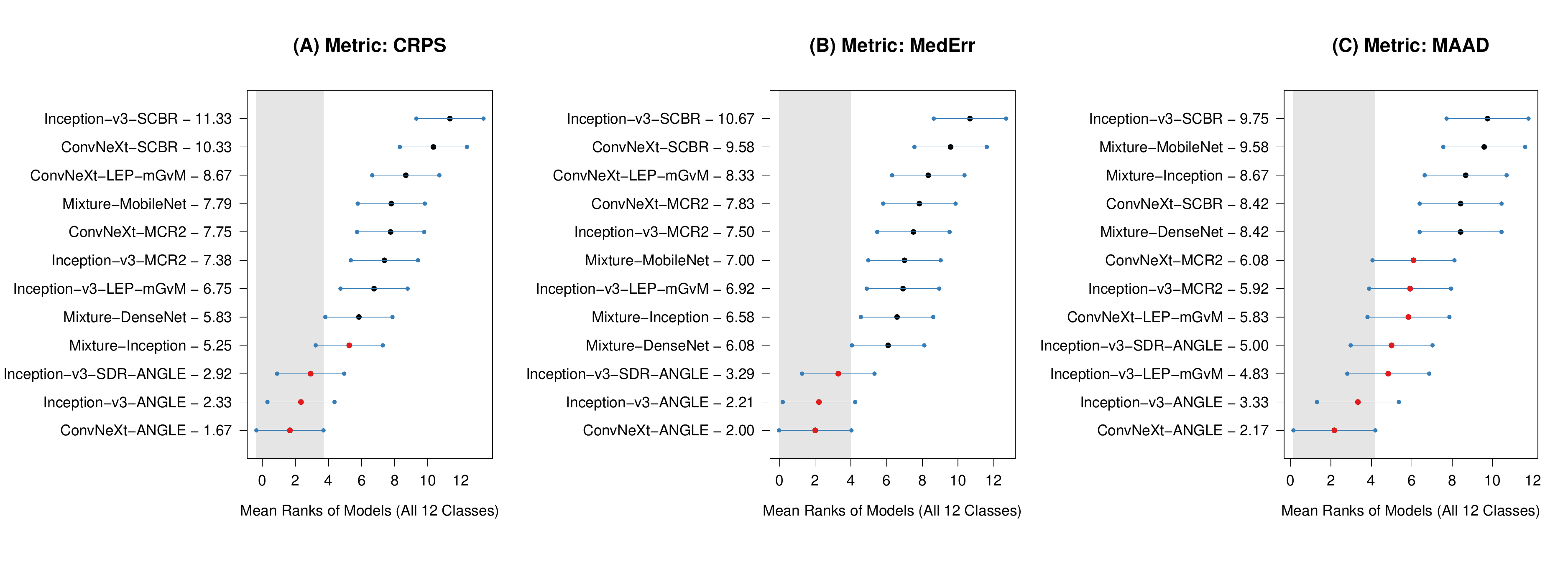}
        \caption{{\small MCB test results to statistically evaluate model performances in object pose estimation using (A) CRPS, (B) MedErr, and (C) MAAD metrics. The $y$-axis in each plot denotes the models and their average ranks (lower the better) corresponding to that metric.}}
        \label{fig:Objectpose_MCB}
\end{figure}

To statistically test the performance of the proposed models against baselines on this task, we conduct the non-parametric Multiple Comparisons with the Best (MCB) test \citep{koning2005m3}. The results presented in Fig. \ref{fig:Objectpose_MCB} indicate that ANGLE is statistically superior to the baselines and achieve the lowest average ranks with respect to CRPS, MedErr, and MAAD metrics.

\subsection{Wind Direction Prediction} \label{sec:wind_direction_pred}
The modeling of wind direction is a well-established problem in circular statistics. As a second application, we leverage spatial coordinates to predict wind direction at unobserved locations across two meteorological datasets. Predicting wind direction at unobserved locations is useful in several applications, including pollutant dispersion tracking, extreme weather forecasting, and the optimization of renewable wind energy infrastructure. The German dataset, sourced from the Deutscher Wetterdienst (DWD)\footnote{https://www.dwd.de}, utilizes a single cross-sectional observation from 260 stations on the final day of a calm weather period \citep{cohen2025bayesian}. Observations for this data are in the range $(-\pi, \pi]$, which is converted to the range of $[0, 2\pi)$ for modeling with ANGLE. However, modeling solely using spatial coordinates is rather simplistic and the procedure can be made more robust by incorporating external covariates that affect wind direction, such as temperature, pressure, and humidity. 

Another dataset considered in this study is the Indian wind data\footnote{https://iced.niti.gov.in/energy/fuel-sources/wind/speed}, which comprises the mean wind direction across 737 districts for December 2025. The observations lie in $[0, 2\pi)$, with the northern direction indicating 0 radians and angles increasing clockwise. The distribution of wind directions for both datasets is presented as rose plots in Fig. \ref{fig:wind_distributions} (Appendix \ref{appendix:tables_figures}). The German and Indian datasets were randomly partitioned into 208/52 and 552/185 train/test splits, respectively (Fig. \ref{fig:combined_wind_datasets} in Appendix \ref{appendix:tables_figures}). ANGLE model comprising 2 layers (128 hidden units, 64 noise dimension) was trained for 500 epochs on the German data, and a 3-layer architecture (100 hidden units, 64 noise dimension) was trained for 800 epochs on the Indian data. We benchmark predictive performance against the baselines listed in \Cref{table:baseline_models}. 

The proposed approach attains superior CRPS scores across both datasets (Fig. \ref{fig:wind_crps}), signifying better modeling of the target conditional distribution. Quantitative results are summarized in \Cref{table:wind_prediction_results} (Appendix \ref{appendix:tables_figures}). Notably, the Indian wind dataset presents a more challenging modeling task than the German dataset due to its larger volume and wider observational variance. As a result, although baseline models yield comparable results on the German data, the proposed ANGLE demonstrates a significant performance advantage on the Indian data. Furthermore, our proposal proves to be a remarkably computationally efficient conditional distribution estimator; while Bayesian vMQP requires 21,677 seconds for its sampling procedure on the German dataset, our models complete training and inference in just 1.96-2.2 seconds (close to CLR), representing a 99.99\% decrease in computational overhead.

\begin{figure}[!t]
        \centering
        \includegraphics[width=0.7\linewidth]{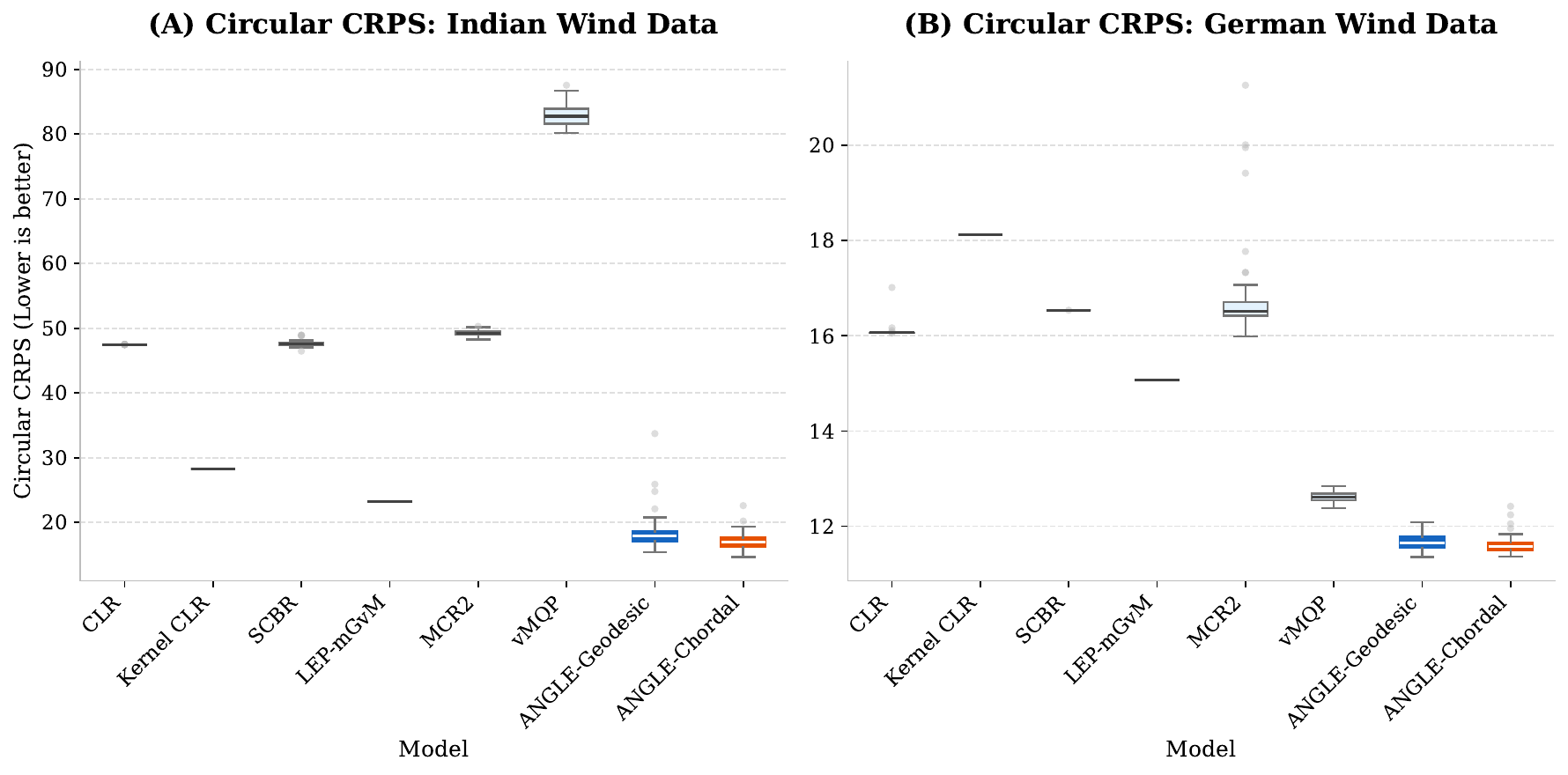}
        \caption{{\small CRPS obtained by the models on (A) Indian and (B) German wind datasets over 50 independent runs.}}
        \label{fig:wind_crps}
\end{figure}

\subsection{ANGLE for Sufficient Dimension Reduction}
We demonstrate the practical utility of sufficient dimension reduction (SDR) using ANGLE for the PASCAL3D+ dataset. In object pose estimation, feature extractors such as Inception-v3 generate high-dimensional representations (here, $d=2048$). To mitigate redundancy and isolate the predictive subspace that is `sufficient' to explain the angular responses, we estimate a $d \times b$ index parameter $\widehat{\beta}$ ($b \ll d$) through ANGLE, following \Cref{thm:central_subspace_recovery}. Setting $b=5$, we project the original features to obtain the reduced 5-dimensional covariates $X' = \widehat{\beta}^\top X$. A stage-two model, termed Inception-v3-SDR-ANGLE, is subsequently trained on $X'$. Despite a $99.7\%$ reduction in feature dimensionality, this approach yields highly competitive predictive performance and consistently outperforms all baselines (see \Cref{table:object_pose_results} and Fig. \ref{fig:Objectpose_Acc_CRPS}), achieving the third lowest average rank in the MCB test (Fig. \ref{fig:Objectpose_MCB}) based on MedErr and CRPS. The estimated 5-dimensional central subspace sufficiently captures the structural information required for azimuth prediction. Beyond computational efficiency, SDR naturally facilitates visual diagnostics and exploratory analysis.



\subsection{ANGLE for Extrapolation}
To evaluate the extrapolation capabilities of the ANGLE framework, we implement a targeted spatial block-splitting strategy on the Indian wind dataset. Specifically, the training set is restricted to observations whose latitude and longitude are both below their respective 85th percentiles, with the remaining data allocated to the test set. This configuration establishes an out-of-distribution (OOD) spatial block, confining the training data to a bounded south-western region of the coordinate space and forcing the model to extrapolate toward the northern and eastern peripheries. Hence, the northern and north-eastern states are isolated as the held-out test set (Fig. \ref{fig:train_test_splits_Extrapolability} in Appendix \ref{appendix:tables_figures}) with 517 training and 220 test observations.

        

{\tiny 
\setlength{\tabcolsep}{5pt} 
\begin{longtable}[]{ccccccccc}
\caption{{\small Model performance on Indian wind dataset under the extrapolation regime. The \underline{\textbf{best}} results are highlighted. MedErr, CRPS, and MAAD are reported in degrees.}}\label{table:extrapolation_wind_results}\tabularnewline
\toprule\noalign{}
Metric & vMQP & CLR & Kernel CLR & SCBR & LEP-MGvM & MCR2 & ANGLE-Geodesic & ANGLE-Chordal \\
\midrule\noalign{}
\endfirsthead
        Accuracy & $0.20\pm0.02$ & \underline{$\mathbf{0.36\pm0.00}$} & $0.22\pm0.00$ & $0.26\pm0.01$ & $0.22\pm0.00$ & $0.21\pm0.03$ & $0.30\pm0.04$ & $0.31\pm0.02$ \\ 
        MedErr & $89.14\pm1.45$ & $52.96\pm0.41$ & $78.72\pm0.00$ & $60.75\pm6.02$ & $73.00\pm0.00$ & $95.22\pm7.90$ & $51.27\pm9.56$ & \underline{$\mathbf{45.74\pm1.55}$} \\ 
        CRPS & $57.84\pm0.34$ & $59.16\pm0.25$ & $78.62\pm0.00$ & $65.15\pm3.68$ & $74.43\pm0.00$ & $91.96\pm3.61$ & $38.84\pm4.58$ & \underline{$\mathbf{37.81\pm0.67}$} \\ 
        MAAD & $88.71\pm0.61$ & $59.16\pm0.25$ & $78.62\pm0.00$ & $65.15\pm3.68$ & $74.43\pm0.00$ & $91.96\pm3.61$ & $58.25\pm7.18$ & \underline{$\mathbf{55.90\pm0.87}$} \\ 
        CMDE & $0.99\pm0.01$ & $0.60\pm0.00$ & $0.85\pm0.00$ & $0.65\pm0.06$ & $0.80\pm0.00$ & $1.03\pm0.05$ & $0.56\pm0.11$ & \underline{$\mathbf{0.52\pm0.01}$} \\  \bottomrule 
\end{longtable}
}

The predictive performance under this extrapolation paradigm is summarized in Table \ref{table:extrapolation_wind_results}. Notably, some conditions required for the theoretical guarantees of extrapolability established in Section \ref{sec:identifiability_and_extrapolation}, such as the monotonicity assumption, are violated in this empirical setting; broad-scale meteorological phenomena over the Indian subcontinent dictate complex circulating wind patterns that do not vary monotonically with latitudinal or longitudinal progression. Although the scores substantially degrade as compared to the prediction (interpolation) regime (Table \ref{table:wind_prediction_results}), the results demonstrate that ANGLE substantially outperforms the baseline models on this OOD task. 

\subsection{ANGLE for Testing Equality of Conditional Distributions}

To evaluate conditional distribution equality using ANGLE, we tested whether azimuth distributions given Inception-v3 features are preserved across the 12 PASCAL3D+ categories. For each class pair, we tested the null hypothesis ($H_0$) of identical conditional pose distributions. ANGLE was fitted on the larger class ($\mathcal{S}_1$). The smaller class ($\mathcal{S}_2$) was bisected: one half generated synthetic azimuths via the fitted model, while the other provided the observed azimuths. These samples were compared using Kuiper's test and Trigonometric MANOVA. The latter serves as our primary testing criterion due to its superior power and Type-I error control for circular data \citep{landler2021advice}.

\begin{figure}[!t]
\centering\includegraphics[width=\linewidth]{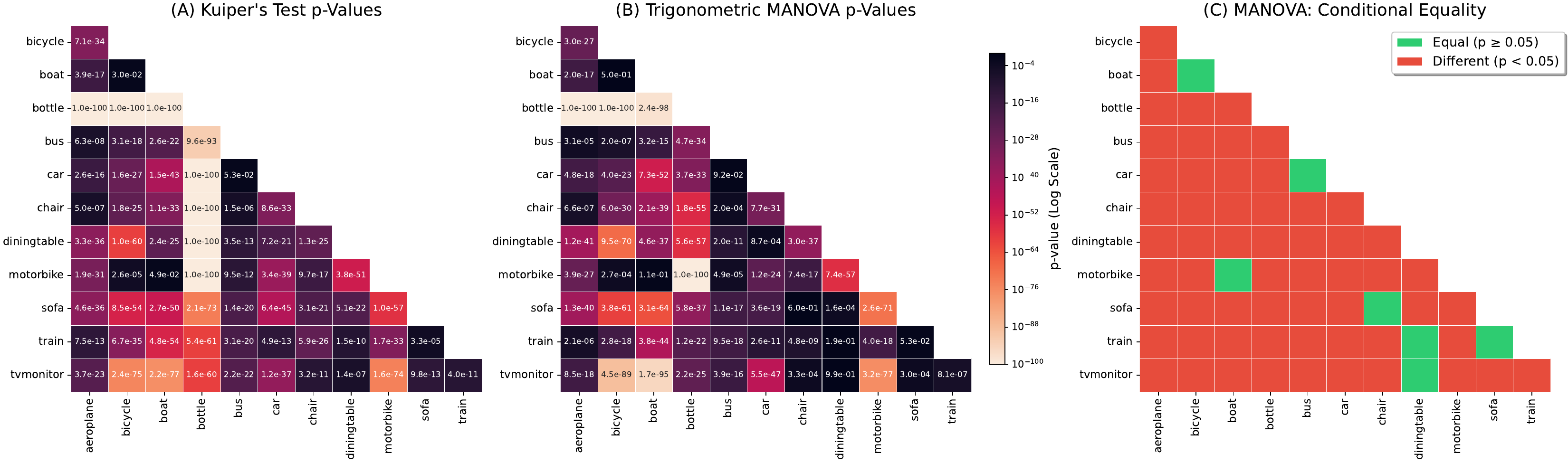}
\caption{{\small Testing pairwise equality of conditional distributions of 12 objects in PASCAL3D+ training set using ANGLE.}}\label{fig:conditional_equality_testing}
\end{figure}

Fig. \ref{fig:conditional_equality_testing} summarizes the exhaustive pairwise tests, with the symmetric upper triangle and diagonal masked. Panels (A) and (B) visualize $p$-values for the Kuiper and MANOVA tests, respectively; lighter shades denote severe generative mismatches ($p \ll 0.05$). Panel (C) binarizes the MANOVA results at $\alpha = 0.05$: green cells indicate statistically indistinguishable distributions (failing to reject $H_0$), while red cells denote significant divergence. Results indicate that structurally analogous objects (e.g., car/bus, sofa/chair) share equivalent conditional distributions. However, the prevalence of rejections highlights a high degree of class-specificity, confirming distinct feature-to-orientation mappings across most pairs. This divergence might be due to varying geometric symmetries, photographic framing biases (e.g., frontal cars versus lateral aeroplanes), or fundamental differences in the underlying neural feature representations.

\section{Ablation Study} \label{sec:ablation_study}
Implementation of the proposed ANGLE framework affords significant architectural flexibility. Within the GCES loss \eqref{eq:crps_angular}, the distance metric can be specified as either chordal or geodesic; for the latter, the function $\mathcal{F}$ may be the negative of any of those outlined in \citet[Table 1]{gneiting2013strictly} to ensure strict propriety. Architecturally, the network supports both single-, multiple-index, and fully dense modes, alongside three output head formulations: unbounded (modulated by $2\pi$ during inference), scaled sigmoid, or a biternion representation utilizing $\mathrm{atan2}$. To compare the contribution of these components in model performance, we conducted an ablation study on the Indian wind dataset utilizing a distinct train-test split from Sec. \ref{sec:wind_direction_pred}. By varying the distance metric, network mode, kernel, and output head, we compared 60 different model variants. Full results are detailed in Table \ref{table:ablation} (Appendix \ref{appendix:tables_figures}), and the top 10 configurations based on CRPS are highlighted in Fig. \ref{fig:ablation}. The results demonstrate that the fully dense mode consistently surpasses the single-index mode, capitalizing on the superior expressive power of fully connected layers. Moreover, models employing an unbounded output modulated by $2\pi$ exhibit high instability, characterized by large standard deviations. We observe that the dense variant combining the chordal distance with a scaled sigmoidal output head achieves the highest overall performance, while the $C^2$-Wendland kernel proves most effective among models utilizing the geodesic distance.

\begin{figure}[!h]
        \centering
        \includegraphics[width=0.8\linewidth]{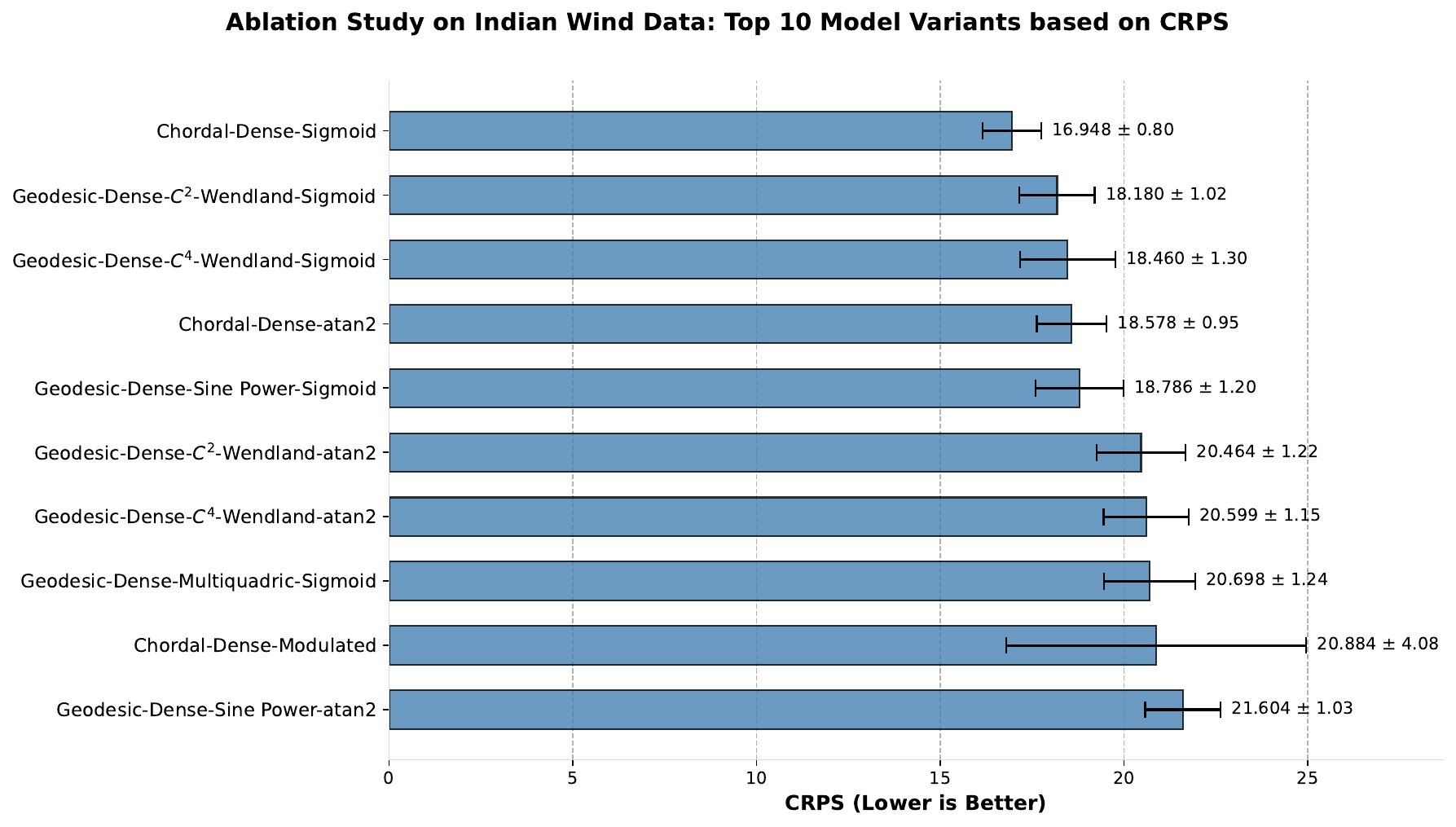}
        \caption{Top 10 ANGLE model variants based on CRPS in the ablation study. The mean and standard deviation of the CRPS scores are reported over 50 independent training and evaluation runs.}
        \label{fig:ablation}
\end{figure}

\section{Discussion}\label{sec:conclusion}

In this paper, we proposed a deep generative framework called ANGLE for distributional learning on circular data, that combines the flexibility of deep learning with the inferential capabilities of distributional modeling. The framework is nonparametric, computationally efficient, produces rotationally equivariant estimators, and offers extrapolation along with uncertainty quantification while naturally accommodating both Euclidean and circular covariates. Through simulations and real data applications to object pose estimation and wind direction prediction, we demonstrated its ability to recover complex conditional distributions and provide model-intrinsic uncertainty quantification. Beyond prediction, ANGLE yields a unified toolbox for supervised learning with angular responses, including sufficient dimension reduction and testing equality of conditional distributions. Several directions remain open for future research. Methodologically, extending the method for probabilistic forecasting of circular time series \citep{harvey2024modelling} and space-time data \citep{wang2014modeling,pathak2026deep} and to more general non-Euclidean response spaces, such as spheres, tori, and manifolds, would considerably broaden its scope. On the applied side, performance in computer vision tasks can be improved through stronger feature extractors and data augmentation. Extending the framework to real-time video analysis would enable applications such as dynamic gaze estimation \citep{nonaka2022dynamic} and trajectory tracking, further broadening the scope of ANGLE.

\renewcommand{\thefigure}{A.\arabic{figure}}
\renewcommand{\thetable}{A.\arabic{table}}

\newcommand{\ms}[2]{%
  \substack{#1\\[0.25ex]\scriptstyle (\pm #2)}%
}

\newcommand{\bms}[2]{%
  \substack{\mathbf{#1}\\[0.25ex] \scriptstyle \mathbf{(\pm #2)}}%
}

\newcommand{\ubms}[2]{%
  \substack{\underline{\mathbf{#1}}\\[0.25ex]
            \scriptstyle \underline{\mathbf{(\pm #2)}}}%
}
\newcommand{\md}[2]{%
  \substack{#1\\[0.25ex]\scriptstyle (\pm #2)}%
}
\def\spacingset#1{\renewcommand{\baselinestretch}%
{#1}\small\normalsize} \spacingset{1}

\clearpage
\appendix

\section*{Appendix}

\section*{Table of Contents} 

\startcontents[appendix]

\printcontents[appendix]{l}{1}{\setcounter{tocdepth}{2}}

\section{Mathematical Proofs}
\label{sec:proofs_of_main_results}

\subsection{Auxiliary Lemmas} \label{appendix:basic_def_lemmas}
\setcounter{lemma}{0}
\renewcommand{\thelemma}{\ref{appendix:basic_def_lemmas}.\arabic{lemma}}

\begin{lemma} \label{lemma:quantile_of_h}
    Let $h:[0,1]\to\mathbb{R}$ be a strictly increasing continuous function and $\varepsilon \sim \mathcal{U}(0,1)$. Then the $\alpha$-quantile of $h(\varepsilon)$ is $h(\alpha)$.
\end{lemma}
\begin{proof}
    Note that as $h$ is strictly increasing and continuous, it is injective. For any $\varepsilon\sim \mathcal{U}(0,1)$,
\begin{align*}
    F_{h(\varepsilon)}(t)=\mathbf{P}(h(\varepsilon)\leq t)=\mathbf{P}(\varepsilon\leq h^{-1}(t))=h^{-1}(t),
\end{align*}
where $F_X$ denotes the CDF of the random variable $X$. So the quantile function of $h(\varepsilon)$ is given as
\begin{align*}
    F^{-1}_{h(\varepsilon)}(\alpha)=\inf\left\{t\mid F_{h(\varepsilon)}(t)\geq \alpha\right\}=\inf\left\{t|h^{-1}(t)\geq \alpha\right\}=h(\alpha),
\end{align*}
since $h$ is strictly increasing. 
\end{proof}


\begin{lemma}[Strict propriety of the chordal circular energy score]
\label{lem:strict_properness_chordal}
Let $\mathcal{P}$ denote the class of all Borel probability measures on $[0,2\pi)$. Let $u:[0,2\pi)\to\mathbb{R}^2$ be the embedding $u(\theta)=(\cos\theta,\sin\theta)^\top$, and let $\mathrm{d}_c(\theta_1,\theta_2)=\norm{u(\theta_1)-u(\theta_2)}$ be the chordal distance. For $\mathbf{P}\in\mathcal{P}$ and $\theta\in[0,2\pi)$, define the chordal circular energy score loss by taking $\mathcal{F}$ to be the identity function and $\mathrm{d}$ to be $\mathrm{d}_c$ in \eqref{eq:crps_angular}. Then $\overset{\circ}{\mathrm{ES}}$ is a strictly proper scoring rule relative to $\mathcal{P}$, i.e., for any $\mathbf{Q}\in\mathcal{P}$,
\begin{align*}
    \mathbf{E}_{\Theta_0\sim \mathbf{Q}}\left[\overset{\circ}{\mathrm{ES}}(\mathbf{Q},\Theta_0)\right] \leq \mathbf{E}_{\Theta_0\sim \mathbf{Q}}\left[\overset{\circ}{\mathrm{ES}}(\mathbf{P},\Theta_0)\right],
\end{align*}
with equality if and only if $\mathbf{P}=\mathbf{Q}$.
\end{lemma}
\begin{proof}
Let $\mathbf{P},\mathbf{Q}\in\mathcal{P}$ be arbitrary. The map $u:[0,2\pi)\to\mathbb{R}^2$ is continuous and strictly injective on $[0,2\pi)$, hence Borel measurable. Define the pushforward measures $u_*\mathbf{P}$ and $u_*\mathbf{Q}$ on $\mathbb{S}^1$. Since $u$ is an injective Borel measurable map between standard Borel spaces, the pushforward operator is injective on Borel probability measures, so that $u_*\mathbf{P} = u_*\mathbf{Q}$ if and only if $\mathbf{P}= \mathbf{Q}$. Since the support of $u_*\mathbf{P}$ is contained in $\mathbb{S}^1$, every $X\sim u_*\mathbf{P}$  satisfies $\|X\| = \|u(\Theta)\| = 1$ almost surely. So, $\|X - Y\| \leq 2$ almost surely for any $Y \sim u_*\mathbf{Q}$, and hence $\mathbf{E}(\|X-Y\|) < \infty$. Substituting $\mathrm{d}_c(\theta_1,\theta_2)=\|u(\theta_1)-u(\theta_2)\|$ into \eqref{eq:crps_angular} under $\mathbf{Q}$ and applying the change of variables $X=u(\Theta)\sim u_*\mathbf{P}$, $X'=u(\Theta')\sim u_*\mathbf{P}$, and $Y=u(\Theta_0)\sim u_*\mathbf{Q}$, we obtain
\begin{align*}
    \mathbf{E}_{\Theta_0\sim Q}\left[\overset{\circ}{\mathrm{ES}}(\mathbf{P},\Theta_0)
    \right] 
    &= \mathbf{E}_{Y\sim u_*\mathbf{Q}}\mathbf{E}_{X\sim u_*\mathbf{P}}\left[\|X-Y\|
    \right] - \frac{1}{2}\mathbf{E}_{X,X'\sim u_*\mathbf{P}}\left[\|X-X'\|\right]\\
    &= \mathbf{E}_{Y\sim u_*\mathbf{Q}}\left[\mathrm{ES}(u_*\mathbf{P},Y)\right],
\end{align*}
where $\mathrm{ES}(\cdot,\cdot)$ denotes the standard multivariate energy score loss on $\mathbb{R}^2$ defined in \cite{shen2025engression}. Now due to \citet[Theorem 4.2]{joziak2015conditionally}, we know that the Euclidean norm is conditionally strictly negative definite on $\mathbb{R}^2$, and hence the same holds for $\mathbb{S}^1$. Then, using \citet[Theorem 4]{waghmare2025proper}, the associated energy score loss must be strictly proper. Therefore we can write,
\begin{align*}
    \mathbf{E}_{Y\sim u_*\mathbf{Q}}\left[\mathrm{ES}(u_*\mathbf{Q},Y)\right] 
    \leq 
    \mathbf{E}_{Y\sim u_*\mathbf{Q}}\left[\mathrm{ES}(u_*\mathbf{P},Y)\right] 
    = \mathbf{E}_{\Theta_0\sim \mathbf{Q}}\left[\overset{\circ}{\mathrm{ES}}
    (\mathbf{P},\Theta_0)\right],
\end{align*}
with equality if and only if $u_*\mathbf{P}=u_*\mathbf{Q}$. Since $u_*\mathbf{P}=u_*\mathbf{Q}$ if and only if $\mathbf{P}=\mathbf{Q}$, equality holds if and only if $\mathbf{P}=\mathbf{Q}$. This completes the proof.
\end{proof}

\begin{lemma} \label{lemma:compactness_condition}
    Let $\mathcal{X}$ be a compact subset of $\mathbb{R}^d$, and let $h: [0, 1] \to \mathbb{R}$ be a continuous function. For any vector $\beta \in \mathbb{R}^d$, the set $S = \left\{\beta^\top x + h(\varepsilon) \mid x \in \mathcal{X}, \varepsilon \in [0,1]\right\}$ is a compact subset of $\mathbb{R}$.
\end{lemma} 
\begin{proof}
Let $\mathcal{Z} = \mathcal{X} \times [0,1]$ denote the Cartesian product of the domain spaces. Because $\mathcal{X}$ is compact by assumption and the closed interval $[0,1]$ is compact, their product $\mathcal{Z}$ is a compact subset of $\mathbb{R}^{d+1}$ by Tychonoff's theorem. Now, define the mapping $g: \mathcal{Z} \to \mathbb{R}$ as $g(x, \varepsilon) = \beta^\top x + h(\varepsilon)$. The function $g$ can be decomposed into the sum of two functions: the linear projection $(x, \varepsilon) \mapsto \beta^\top x$, which is universally continuous, and the scalar function $(x, \varepsilon) \mapsto h(\varepsilon)$, which is continuous by the hypothesis on $h$. As the sum of two continuous functions, $g$ is continuous on $\mathcal{Z}$. By definition, the set $S$ is precisely the image of $\mathcal{Z}$ under $g$, meaning $S = g(\mathcal{Z})$. Since the continuous image of a compact topological space is compact, it follows that $S$ is a compact subset of $\mathbb{R}$.
\end{proof}

\subsection{Identifiability Analysis} \label{appendix:identifiability}

As identifiability of the model classes serve as a basis for extrapolability \citep{shen2025engression}, we first investigate the identifiability of \eqref{eq:model_classes}. For this, we impose some assumptions on the covariate  $X$, coefficient vectors $\beta\in\mathcal{B}$, and the function classes $\mathcal{G}$ and $\mathcal{H}$. The assumptions are stated in Sec. \ref{sec:identifiability_and_extrapolation}, where we take,
\begin{align*}
    & \mathcal{G}=\left\{g|g\text{ satisfies \Cref{assmp:g_related}}\right\}, \quad \mathcal{H}=\left\{h|h\text{ satisfies \Cref{assmp:h_related}}\right\}.
\end{align*}

The assumptions establish the necessary regularity conditions to guarantee model identifiability and extrapolability. \Cref{assmp:X_support_related} imposes a compact, well-behaved support for the covariates with a strictly positive density. This is a standard regularity condition in nonparametric estimation that naturally aligns with the bounded nature of real-world measurements, ensuring stable estimation across the domain. \Cref{assmp:g_related} ensures that the data-generating function $g$ is smooth and, crucially for pre-ANMs, strictly monotone. Following \citet{shen2025engression}, requiring monotonicity provides a structural middle ground. It avoids the highly restrictive assumption of strict linearity (which trivially guarantees identifiability) while imposing sufficient shape constraints to track the generative mechanism outside the training support. \Cref{assmp:h_related} characterizes the noise distribution, mapping a uniform base measure via the strictly increasing function $h$. The condition $h(1/2) = 0$ anchors the noise with a zero median. More importantly, for directional data, the strict bound $|h(1) - h(0)| < 2\pi$ ensures the noise support cannot span the $\mathbb{S}^1$ manifold entirely. Without this constraint, the noise would fully wrap around the circle, obscuring the structural signal and rendering the model fundamentally unidentifiable. Finally, \Cref{assmp:B_related} specifies the parameter space for the index vector $\beta$. Fixing the Euclidean norm ($\|\beta\| = 1$) and the sign of the first component ($\beta_{(1)} > 0$) resolves the inherent scale and sign ambiguities standard to single-index models. In practice, this structural constraint can be maintained during optimization via manifold parameterization or projected gradient techniques.

Through the subsequent results, we show that the classes $\mathbf{M}_{\text{post}}^{(2)}$ and $\mathbf{M}_{\text{pre}}^{(2)}$ are fully identifiable, i.e., two models in these classes that induce the same conditional distribution $\mathbf{P}_{Y|x}$ for all $x \in \mathcal{X}$ must necessarily share the same index $\beta$, the same regression function $g$ on the training support almost surely, and the same noise function $h$ almost surely on $[0,1]$. In contrast, the classes $\mathbf{M}_{\text{post}}^{(1)}$ and $\mathbf{M}_{\text{pre}}^{(1)}$ exhibit what we term as the \emph{$2k\pi$-identifiability trap}. Since the response is observed only modulo $2\pi$, shifting the regression function $g$ by any integer multiple of $2\pi$ leaves the conditional distribution of $Y|X$ unchanged. Consequently, parameter configurations differing only by such shifts are observationally indistinguishable, and $g$ is identifiable only up to an additive constant of the form $2k\pi$, where $k \in \mathbb{Z}$. This lack of full identifiability arises because $g$ is allowed to take values in all of $\mathbb{R}$ rather than being constrained to $[0,2\pi)$. The restriction $g:\mathbb{R}\rightarrow[0,2\pi)$ imposed in $\mathbf{M}_{\text{post}}^{(2)}$ and $\mathbf{M}_{\text{pre}}^{(2)}$ restores full identifiability. Intuitively, once $g$ is constrained to take values within a single period of the circle, the usual ambiguity caused by adding multiples of $2\pi$ disappears. Thus, the range restriction selects a unique representative among all wrapped equivalent functions, thereby restoring identifiability.

\begin{theorem}[Identifiability]
\label{thm:identifiability_all}
\leavevmode
\begin{enumerate}[label=(\Alph*)]

    \item[(A)] \label{thm:identifiability_post2}%
    (\textit{Identifiability of $\mathbf{M}_{\mathrm{post}}^{(2)}$}).\enspace
    Let $\beta_1,\beta_2\in \mathcal{B}$, $g_1,g_2\in \mathcal{G}$
    $(g_1, g_2: \mathbb{R} \to [0, 2\pi))$, and $h_1,h_2\in \mathcal{H}$
    such that
    \begin{align}
        \label{eq:equality_condition_thm1}
        \bigl(g_1\bigl(\beta_1^\top x\bigr)+h_1(\varepsilon)\bigr)
        \bmod 2\pi
        \;\overset{d}{=}\;
        \bigl(g_2\bigl(\beta_2^\top x\bigr)+h_2(\varepsilon)\bigr)
        \bmod 2\pi
    \end{align}
    for all $x\in \mathcal{X}$, $\varepsilon\in [0,1]$. Then
    $\beta_1=\beta_2=:\beta$, $g_1(\beta^\top x)=g_2(\beta^\top x)$
    almost surely in $\mathcal{X}$, and $h_1(\varepsilon)=h_2(\varepsilon)$
    almost surely in $[0,1]$.

    \item[(B)] \label{thm:identifiability_post1}%
    (\textit{$2k\pi$-identifiability of $\mathbf{M}_{\mathrm{post}}^{(1)}$}). \enspace
    Let $\beta_1,\beta_2\in \mathcal{B}$, $g_1,g_2\in \mathcal{G}$
    $(g_1, g_2: \mathbb{R} \to \mathbb{R})$, and $h_1,h_2\in \mathcal{H}$
    such that
    \begin{align}
        \label{eq:equality_condition_th1}
        \bigl(g_1\bigl(\beta_1^\top x\bigr)+h_1(\varepsilon)\bigr)
        \bmod 2\pi
        \;\overset{d}{=}\;
        \bigl(g_2\bigl(\beta_2^\top x\bigr)+h_2(\varepsilon)\bigr)
        \bmod 2\pi
    \end{align}
    for all $x\in \mathcal{X}$, $\varepsilon\in [0,1]$. Then
    $g_1(\beta_1^\top x)=g_2(\beta_2^\top x)+2k\pi$ for some $k\in \mathbb{Z}$
    almost surely in $\mathcal{X}$, and $h_1(\varepsilon)=h_2(\varepsilon)$
    almost surely in $[0,1]$. Hence $g$ is not uniquely identified, and
    the model class falls into a \emph{$2k\pi$-identifiability trap}.

    \item[(C)] \label{thm:identifiability_pre1}%
    (\textit{$2k\pi$-identifiability of $\mathbf{M}_{\mathrm{pre}}^{(1)}$}).\enspace
    Assume $g_1,g_2\in \mathcal{G}$ ($g_1, g_2:\mathbb{R}\to\mathbb{R}$),
    $h_1,h_2\in \mathcal{H}$, and $\beta_1,\beta_2\in \mathcal{B}$, with the
    map $\varepsilon\mapsto g(\beta^\top x+h(\varepsilon))$ supported on an
    interval of length strictly less than $2\pi$. Suppose
    \begin{align}
        \label{eq:equality_condition_th5}
        g_1\bigl(\beta_1^\top x+h_1(\varepsilon)\bigr)\bmod 2\pi
        \;\overset{d}{=}\;
        g_2\bigl(\beta_2^\top x+h_2(\varepsilon)\bigr)\bmod 2\pi
    \end{align}
    for all $x\in \mathcal{X}$, $\varepsilon\in [0,1]$. Then $\beta_1=\beta_2$, $h_1=h_2$, and
    $g_1(\beta_1^\top x + h_1(\varepsilon))
     = g_2(\beta_2^\top x + h_2(\varepsilon)) + 2k\pi$
    for some $k\in \mathbb{Z}$ and all $\varepsilon\in[0,1]$, so that
    $\mathbf{M}_{\mathrm{pre}}^{(1)}$ falls into a
    \emph{$2k\pi$-identifiability trap}.

    \item[(D)] \label{thm:identifiability_pre2}%
    (\textit{Identifiability of $\mathbf{M}_{\mathrm{pre}}^{(2)}$}).\enspace
    Assume $g_1,g_2\in \mathcal{G}$
    $\bigl(g_1, g_2: \mathbb{R}\to [0,2\pi)\bigr)$,
    $h_1,h_2\in \mathcal{H}$, and $\beta_1,\beta_2\in \mathcal{B}$.
    Suppose
    \begin{align}
        \label{eq:equality_condition_th3}
        g_1\bigl(\beta_1^\top x+h_1(\varepsilon)\bigr)
        \;\overset{d}{=}\;
        g_2\bigl(\beta_2^\top x+h_2(\varepsilon)\bigr)
    \end{align}
    for all $x\in \mathcal{X}$, $\varepsilon\in [0,1]$. Then
    $\beta_1=\beta_2=\beta$,
    $g_1(\beta^\top x+h_1(\varepsilon))=g_2(\beta^\top x+h_2(\varepsilon))$,
    and $h_1(\varepsilon)=h_2(\varepsilon)$ for all $\varepsilon\in[0,1]$.

\end{enumerate}
\end{theorem}

\begin{proof}
The proof is as follows.
    \begin{subproof}[Proof of (A)]
    For a fixed $x\in \mathcal{X}$, \eqref{eq:equality_condition_thm1} implies the equality of all trigonometric moments on the unit circle \citep{bell2024review}, which means for all $r\in \mathbb{Z}$:
    \begin{align}
        \notag & \mathbf{E}\left[\exp\left\{\iota r\left(g_1(\beta_1^\top x)+h_1(\varepsilon)\right)(\bmod \ 2\pi)\right\}\right]= \mathbf{E}\left[\exp\left\{\iota r\left(g_2(\beta_2^\top x)+h_2(\varepsilon)\right)(\bmod \ 2\pi)\right\}\right]\\
        \label{eq:1_thm1}
        \implies & \mathbf{E}\left[\exp\left\{\iota r\left(g_1(\beta_1^\top x)+h_1(\varepsilon)\right)\right\}\right]= \mathbf{E}\left[\exp\left\{\iota r\left(g_2(\beta_2^\top x)+h_2(\varepsilon)\right)\right\}\right]\\
        \notag \implies & \exp\left\{\iota rg_1(\beta_1^\top x)\right\}\psi_{h_1}(r)=\exp{\left\{\iota rg_2(\beta_2^\top x)\right\}}\psi_{h_2}(r).
    \end{align}
    Due to \Cref{assmp:h_related}, we must have $\psi_h(1) \neq 0$. Taking $r=1$ in \eqref{eq:1_thm1} allows us to write: 
\begin{align}
    \label{eq:2_thm1}
    \exp{\left\{\iota\left(g_1(\beta_1^\top x)-g_2(\beta_2^\top x)\right)\right\}}=\frac{\psi_{h_2}(1)}{\psi_{h_1}(1)}.
\end{align}
    Let $\delta_0 = \arg\left(\frac{\psi_{h_2}(1)}{\psi_{h_1}(1)}\right) \in (-\pi, \pi]$ be the principal argument of the constant complex fraction. Since the right-hand side of \eqref{eq:2_thm1} is independent of $x$, evaluating the complex exponential implies that the arguments must be equal up to an integer multiple of $2\pi$:
\begin{align}
    \label{eq:3_thm1}
    g_1(\beta_1^\top x)-g_2(\beta_2^\top x) = \delta_0 + 2\pi\kappa(x),
\end{align}
    where $\kappa(x) \in \mathbb{Z}$ is some integer-valued function of $x$. Substituting this relationship back into \eqref{eq:1_thm1} for an arbitrary integer $r \in \mathbb{Z}$ yields:
\begin{align*}
    \exp{\left\{\iota r(\delta_0 + 2\pi\kappa(x))\right\}}\psi_{h_1}(r) = \psi_{h_2}(r).
\end{align*}
    Because both $r$ and $\kappa(x)$ are integers, the term $\exp(\iota r 2\pi \kappa(x))$ is identically $1$ for all $x \in \mathcal{X}$. Hence,
\begin{align}
    \label{eq:4_thm1}
    \exp{(\iota r\delta_0)}\psi_{h_1}(r)=\psi_{h_2}(r),\quad \forall r\in \mathbb{Z}.
\end{align}
    By the uniqueness of trigonometric moments, \eqref{eq:4_thm1} implies that the wrapped distributions are globally identical:
    $$(h_2(\varepsilon)) \bmod 2\pi \overset{d}{=} (h_1(\varepsilon) + \delta_0) \bmod 2\pi.$$ 
    By \Cref{assmp:h_related}(A), $h$ is strictly increasing, so its image is the interval $[h(0), h(1)]$. By \Cref{assmp:h_related}(B), $|h(1) - h(0)| < 2\pi$, which ensures this image interval has a length strictly less than $2\pi$. Consequently, the mapping to the circle is injective and contains no overlapping probability mass. Therefore, the unwrapped distributions must be equal up to a single global integer shift $k \in \mathbb{Z}$:
\begin{align} \label{eq:5_thm1}
    h_2(\varepsilon) + 2k\pi\overset{d}{=} h_1(\varepsilon)+\delta_0.
\end{align}
    Because $\varepsilon \sim \mathcal{U}(0,1)$ and both $h_1, h_2$ are strictly monotonically increasing, the median of the distribution occurs exactly at $\varepsilon = 0.5$. From \Cref{assmp:h_related}(A), $h_1(0.5)=h_2(0.5)=0$. Taking the median of both sides of \eqref{eq:5_thm1} yields $2k\pi = \delta_0$. Because $\delta_0 \in (-\pi, \pi]$ is a principal argument, the only integer $k$ that satisfies this equality is $k=0$, which consequently forces $\delta_0=0$. Substituting this into \eqref{eq:5_thm1} reveals that $h_1(\varepsilon)$ and $h_2(\varepsilon)$ are identical in distribution. Because $h_1$ and $h_2$ are strictly monotonically increasing functions mapping from the same uniform random variable $\varepsilon$, they uniquely define the quantile functions of their distributions (see \Cref{lemma:quantile_of_h}). Identical distributions possess identical quantile functions, forcing pointwise equality:
    $$h_1(\varepsilon) = h_2(\varepsilon) \quad \text{almost surely in } [0,1].$$
    Finally, we substitute $\delta_0=0$ back into the spatial equation \eqref{eq:3_thm1} to obtain:
    $$g_1(\beta_1^\top x) - g_2(\beta_2^\top x) = 2\pi\kappa(x).$$
    By the definition of the model class $\mathbf{M}_{\mathrm{post}}^{(2)}$, both $g_1$ and $g_2$ strictly map to the codomain $[0, 2\pi)$. Therefore, their absolute difference is strictly bounded: $|g_1(\beta_1^\top x) - g_2(\beta_2^\top x)| < 2\pi$. This forces $|2\pi\kappa(x)| < 2\pi$. Since $\kappa(x)$ is an integer, the only valid solution is $\kappa(x) = 0$ for all $x \in \mathcal{X}$. Hence:
\begin{align}
\label{eq:9_thm1}
    g_1(\beta_1^\top x)=g_2(\beta_2^\top x).
\end{align}
    Due to \Cref{assmp:g_related}(A), \Cref{assmp:g_related}(B), and the normalization constraints in \Cref{assmp:B_related}, \eqref{eq:9_thm1} implies $\beta_1=\beta_2$ and $g_1=g_2$ almost surely on $\mathcal{X}$ by Theorem 1 of \citet{lin2007identifiability}.
\end{subproof}

\begin{subproof}[Proof of (B)]
    For a fixed $x\in \mathcal{X}$, \eqref{eq:equality_condition_th1} implies the equality of all trigonometric moments on the unit circle \citep{bell2024review}. Because the additive noise structure is identical to that of model class $\mathbf{M}_{\mathrm{post}}^{(2)}$, we can apply the exact same deductions utilized in the proof of \Cref{thm:identifiability_all}(A). By isolating the first trigonometric moments and unwrapping the bounded noise distributions, we immediately recover both the pointwise equality of the noise functions, $h_1(\varepsilon) = h_2(\varepsilon)$ almost surely on $[0,1]$, and the global spatial phase relationship:
    \begin{align*}
        g_1(\beta_1^\top x) = g_2(\beta_2^\top x) + 2k\pi, \quad \text{for some } k \in \mathbb{Z}.
    \end{align*}
    Now, unlike model class $\mathbf{M}_{\mathrm{post}}^{(2)}$, the functions $g_1, g_2 \in \mathbf{M}_{\mathrm{post}}^{(1)}$ map to the unbounded codomain $\mathbb{R}$. Because their outputs are not restricted to $[0, 2\pi)$, their absolute difference $|g_1(\beta_1^\top x) - g_2(\beta_2^\top x)|$ is not geometrically constrained to be strictly less than $2\pi$. Hence, $k=0$ is not forced, and the functions are only identifiable up to an integer multiple of $2\pi$, establishing that the model class falls into a $2k\pi$-identifiability trap.
\end{subproof}

\begin{subproof}[Proof of (C)]
    Define a push-forward probability measure on the target space $[0, 2\pi)$ as: 
    \begin{align*}
        \mu_j(A)=\mathcal{U}\left(\left\{a\in [0,1]\mid g_j\left(\beta_j^\top x+h_j(a)\right)\bmod 2\pi\in A\right\}\right),\quad \text{for any Borel set } A \in \mathscr{B}([0, 2\pi)).
    \end{align*}
    The equality in distribution \eqref{eq:equality_condition_th5} implies that the two random variables induce the identical probability measure on the target space $[0, 2\pi)$. Therefore, for every measurable set $A\in \mathscr{B}([0, 2\pi))$,
\begin{align}
\label{eq:1_thm5}
    \mathbf{P}\left[g_1\left(\beta_1^\top x+h_1(\varepsilon)\right)\bmod 2\pi\in A\right]=\mathbf{P}\left[g_2\left(\beta_2^\top x+h_2(\varepsilon)\right)\bmod 2\pi\in A\right].
\end{align}
    Because $\varepsilon\sim \mathcal{U}(0,1)$, we have $\mathbf{P}(\varepsilon\in B)=\mathcal{U}(B)$ for any measurable $B\subseteq [0,1]$. Thus,
\begin{align*}
    \mathbf{P}\left(g_j\left(\beta_j^\top x+h_j(\varepsilon)\right)\bmod 2\pi\in A\right) = \mu_j(A).
\end{align*}
    Substituting this into \eqref{eq:1_thm5} yields $\mu_1(A)=\mu_2(A)$ for all Borel sets $A$, meaning $\mu_1\equiv\mu_2$ as measures on $[0, 2\pi)$. 
    
\noindent Next, define an unwrapped push-forward measure on $\mathbb{R}$ as:
\begin{align}
\label{eq:nuj_def}
    \nu_j(B)=\mathcal{U}\left(\left\{a\in [0,1]\mid g_j\left(\beta_j^\top x+h_j(a)\right)\in B\right\}\right),\quad \text{for any Borel set } B \in \mathscr{B}(\mathbb{R}).
\end{align}
    Let $\omega: \mathbb{R} \to [0, 2\pi)$ denote the Borel measurable wrapping map $\omega(t) = t \bmod 2\pi$. By assumption, the mapping $\varepsilon \mapsto g_j(\beta_j^\top x+h_j(\varepsilon))$ is supported on an interval strictly shorter than $2\pi$. Let $I_{\nu_j} \subset \mathbb{R}$ denote this interval. The measure $\nu_j$ is therefore entirely supported on $I_{\nu_j}$. Because the length of $I_{\nu_j}$ is strictly less than $2\pi$, the restriction of the wrapping map to this interval, $\omega|_{I_{\nu_j}}$, is strictly injective. 
    
For any measurable $A \in \mathscr{B}([0, 2\pi))$,
\begin{align*}
    \notag \mu_j(A)& =\mathcal{U}\left(\left\{a\in [0,1]\mid g_j\left(\beta_j^\top x+h_j(a)\right)\in \omega^{-1}(A)\right\}\right)\\
    &= \nu_j(\omega^{-1}(A))=\omega_{*}\nu_j(A),
\end{align*}
    which establishes that $\mu_j$ is the push-forward of $\nu_j$ under $\omega$. Conversely, since $\omega$ is injective on $I_{\nu_j}$, we can write for any measurable $B\in \mathscr{B}(\mathbb{R})$:
\begin{align}
    \notag \nu_j(B)&= \nu_j(B\cap I_{\nu_j})\text{ (since $\nu_j$ is supported on $I_{\nu_j}$)}\\
    \notag &= \nu_j\left(\omega^{-1}\left(\omega(B\cap I_{\nu_j})\right)\right)\text{ (by injectivity of $\omega$ on $I_{\nu_j}$)}\\
    \label{eq:4_th3}
    &= \mu_j\left(\omega(B\cap I_{\nu_j})\right)\text{ (since $\mu_j=\omega_{*}\nu_j$)}.
\end{align}
    Because $\mu_1 \equiv \mu_2$, and both $\nu_1$ and $\nu_2$ are supported on intervals strictly shorter than $2\pi$, their supports can only differ by an exact integer multiple of $2\pi$. That is, there exists a unique $k \in \mathbb{Z}$ such that $I_{\nu_1} = I_{\nu_2} + 2\pi k$. Consequently, the wrapped images $\omega(I_{\nu_1})$ and $\omega(I_{\nu_2})$ perfectly coincide on $[0, 2\pi)$. Due to the $2\pi$-periodicity of $\omega$, for any measurable $B$:
\begin{align*}
    \omega(B\cap I_{\nu_1})=\omega\left(B\cap (I_{\nu_2}+2\pi k)\right) = \omega\left((B-2\pi k)\cap I_{\nu_2}\right).
\end{align*}
    Therefore, utilizing \eqref{eq:4_th3}, we obtain:
\begin{align}
    \notag \nu_1(B)=\mu_1(\omega(B\cap I_{\nu_1}))&=\mu_2(\omega(B\cap I_{\nu_1})) \\
    \notag &= \mu_2(\omega((B-2\pi k)\cap I_{\nu_2}))\\
    \label{eq:5_th5}
    &= \nu_2(B-2\pi k).
\end{align}
    Equation \eqref{eq:5_th5} establishes that $\nu_1$ and $\nu_2$ are identical measures up to a spatial shift of $2\pi k$. Returning to the definition of $\nu_j$ in \eqref{eq:nuj_def}, this implies that for some $k\in\mathbb{Z}$:
\begin{align*}
    g_1\left(\beta_1^\top x+h_1(\varepsilon)\right)\overset{d}{=}g_2\left(\beta_2^\top x+h_2(\varepsilon)\right) + 2\pi k.
\end{align*}

Fixing $x\in \mathcal{X}$, the composition $\varepsilon\mapsto g(\beta^\top x+h(\varepsilon))$ is strictly monotonically increasing in $\varepsilon$ due to \Cref{assmp:g_related}(C) and \Cref{assmp:h_related}(A). Because strictly increasing functions of a uniform variable uniquely define the quantile functions of their distributions, equality in distribution necessitates pointwise equality almost surely:
\begin{align} \label{eq:intermediate_with_beta}
    g_1(\beta_1^\top x + h_1(\varepsilon)) = g_2(\beta_2^\top x + h_2(\varepsilon)) + 2\pi k, \quad \forall \varepsilon\in[0,1].
    \end{align}

    \noindent By taking the gradient of both sides of this equation with respect to $x$, we obtain the vector equation:
\begin{align} \label{eq:thm5_gradients}
g_1'\bigl(\beta_1^\top x + h_1(\varepsilon)\bigr) \beta_1 = g_2'\bigl(\beta_2^\top x + h_2(\varepsilon)\bigr) \beta_2.
\end{align}
    Taking the Euclidean $\ell_2$-norm of both sides yields:
$$\abs{g_1'\bigl(\beta_1^\top x + h_1(\varepsilon)\bigr)} \norm{\beta_1}_2 = \abs{g_2'\bigl(\beta_2^\top x + h_2(\varepsilon)\bigr)} \norm{\beta_2}_2.$$
    By \Cref{assmp:B_related}, the coefficient vectors are normalized such that $\norm{\beta_1}_2 = \norm{\beta_2}_2 = 1$, reducing the equality to:
$$\abs{g_1'\bigl(\beta_1^\top x + h_1(\varepsilon)\bigr)} = \abs{g_2'\bigl(\beta_2^\top x + h_2(\varepsilon)\bigr)}.$$
    By \Cref{assmp:g_related}(C), $g_1$ and $g_2$ are strictly monotone, meaning their derivatives are strictly non-zero and maintain a constant sign. Consequently, we must have $g_1'\bigl(\beta_1^\top x + h_1(\varepsilon)\bigr) = \pm g_2'\bigl(\beta_2^\top x + h_2(\varepsilon)\bigr)$. Substituting this into \eqref{eq:thm5_gradients},
$$\pm g_2'\bigl(\beta_2^\top x + h_2(\varepsilon)\bigr) \beta_1 = g_2'\bigl(\beta_2^\top x + h_2(\varepsilon)\bigr) \beta_2.$$
    Because the derivative $g_2'(\cdot)$ is strictly non-zero, we can divide both sides by this scalar term, isolating the coefficient vectors to obtain $\pm \beta_1 = \beta_2.$ \Cref{assmp:B_related} further dictates that the first coordinate of any $\beta \in \mathcal{B}$ is strictly positive ($\beta_{(1)} > 0$). This strictly rules out the negative root $\beta_1 = -\beta_2$, forcing the conclusion that $\beta_1 = \beta_2 = \beta$.

    Having established $\beta_1 = \beta_2 = \beta$, we return to the intermediate equality \eqref{eq:intermediate_with_beta}. By \Cref{assmp:h_related}(A), the noise functions are anchored at the median such that $h_1(0.5) = h_2(0.5) = 0$. Evaluating \eqref{eq:intermediate_with_beta} at $\varepsilon = 0.5$ yields:
\begin{align*}
    g_1\bigl(\beta^\top x\bigr) = g_2\bigl(\beta^\top x\bigr) + 2k\pi.
\end{align*}
    Because this relationship holds for all $x \in \mathcal{X}$, it describes a global constant shift between the two functions. We can substitute this functional relationship $g_2(t) + 2k\pi = g_1(t)$ back into \eqref{eq:intermediate_with_beta} for an arbitrary $\varepsilon$:
\begin{align*}
    g_1\bigl(\beta^\top x + h_1(\varepsilon)\bigr) = g_1\bigl(\beta^\top x + h_2(\varepsilon)\bigr).
\end{align*}
    Because $g_1$ is strictly monotone (\Cref{assmp:g_related}(C)), it is strictly injective. The equality of the evaluated outputs therefore guarantees the equality of their inputs:
\begin{align*}
    \beta^\top x + h_1(\varepsilon) &= \beta^\top x + h_2(\varepsilon) \\
    \implies h_1(\varepsilon) &= h_2(\varepsilon).
\end{align*}
    This confirms that $h_1 = h_2$ for all $\varepsilon \in [0, 1]$. We have thus established $\beta_1 = \beta_2$, $h_1 = h_2$, and $g_1(\cdot) = g_2(\cdot) + 2k\pi$, proving that $\mathbf{M}_{\mathrm{pre}}^{(1)}$ falls into the $2k\pi$-identifiability trap.
\end{subproof}
\begin{subproof}[Proof of (D)]
Fix $x\in \mathcal{X}$. Due to \Cref{assmp:g_related}(C) and \Cref{assmp:h_related}(A), the composition $\varepsilon\mapsto g(\beta^\top x+h(\varepsilon))$ is strictly monotone in $\varepsilon$. By Lemma~\ref{lemma:quantile_of_h}, the $a$-th quantile of $h(\varepsilon)$ is $h(a)$, so that
\begin{align}
    \label{eq:7_th3}
    g_j\left(\beta_j^\top x+Q_a(h_j(\varepsilon))\right)=g_j\left(\beta_j^\top x+h_j(a)\right),
\end{align}
where $Q_a(\cdot)$ denotes the $a$-level quantile function. Now \eqref{eq:equality_condition_th3} implies that all the quantiles of the random variables on both sides should be equal, and hence using \eqref{eq:7_th3}, we must have 
\begin{align}
    \label{eq:8_th3}
    g_1\left(\beta_1^\top x+h_1(a)\right)=g_2\left(\beta_2^\top x+h_2(a)\right)
\end{align}
for all $x\in \mathcal{X},a\in [0,1]$. 
    By \Cref{assmp:h_related}(A), $h_1(0.5) = h_2(0.5) = 0$. Evaluating \eqref{eq:7_th3} at $a = 0.5$ yields:
\begin{align*}
    g_1\bigl(\beta_1^\top x\bigr) = g_2\bigl(\beta_2^\top x\bigr).
\end{align*}
    Because this holds for all $x \in \mathcal{X}$, under \Cref{assmp:X_support_related}, \Cref{assmp:g_related}(B), and the coefficient constraints in \Cref{assmp:B_related}, we can invoke Theorem 1 of \citet{lin2007identifiability} to conclude that
\begin{align*}
    \beta_1 = \beta_2 = \beta \quad \text{and} \quad g_1 = g_2.
\end{align*}
    Having established the equality of the index parameters and the spatial functions, we substitute $\beta_1 = \beta_2 = \beta$ and $g_2 = g_1$ back into the pointwise equality \eqref{eq:8_th3} for an arbitrary $\varepsilon\in [0,1]$ to get
\begin{align*}
    g_1\bigl(\beta^\top x+h_1(\varepsilon)\bigr) = g_1\bigl(\beta^\top x+h_2(\varepsilon)\bigr).
\end{align*}
    Because $g_1$ is strictly monotone (\Cref{assmp:g_related}(C)), it is injective. The equality of the evaluated outputs therefore guarantees the equality of their inputs:
\begin{align*}
    \beta^\top x+h_1(\varepsilon) &= \beta^\top x+h_2(\varepsilon) \\
    \implies h_1(\varepsilon) &= h_2(\varepsilon).
\end{align*}
    This confirms that $h_1(\varepsilon) = h_2(\varepsilon)$ for all $\varepsilon\in[0,1]$. We have thus established $\beta_1=\beta_2$, $g_1=g_2$, and $h_1=h_2$, completing the identifiability proof for $\mathbf{M}_{\mathrm{pre}}^{(2)}$.
\end{subproof}

\end{proof}

\subsection{Proof of \Cref{thm:population_guarantee_engression}}

\begin{proof}
    Given any model $g\in \mathcal{M}$ with $g\sim \mathbf{P}_{\tr}(y|x)$, assume for the sake of contradiction that there exists a subset $\mathcal{X}_1\subseteq\mathcal{X}$ with $\mathbf{P}(X\in \mathcal{X}_1)>0$, such that $\mathbf{P}_{\widetilde g}(y|x)\neq \mathbf{P}_{\tr}(y|x)$ for all $x\in \mathcal{X}_1$. Now, since \eqref{eq:crps_angular} is a strictly proper scoring rule, we must have for all $x\in \mathcal{X}_1$,
    \begin{align}
        \label{eq:1_thm7}
        \mathbf{E}_{Y\sim \mathbf{P}_{\tr}(y|x)}\left[\overset{\circ}{\mathrm{ES}}\left(\mathbf{P}_{\tr}(y|x),Y\right)\right] < \mathbf{E}_{Y\sim \mathbf{P}_{\tr}(y|x)}\left[\overset{\circ}{\mathrm{ES}}\left(\mathbf{P}_{\widetilde g}(y|x),Y\right)\right],\quad \forall x\in \mathcal{X}_1.
    \end{align}
Since we have $g \sim\mathbf{P}_{\tr}(y|x)$, \eqref{eq:1_thm7} can be written as
\begin{align}
    \label{eq:2_th7}
     \mathbf{E}_{Y\sim \mathbf{P}_{\tr}(y|x)}\left[\overset{\circ}{\mathrm{ES}}\left(\mathbf{P}_{g}(y|x),Y\right)\right] < \mathbf{E}_{Y\sim \mathbf{P}_{\tr}(y|x)}\left[\overset{\circ}{\mathrm{ES}}\left(\mathbf{P}_{\widetilde g}(y|x),Y\right)\right],\quad \forall x\in \mathcal{X}_1.
\end{align}
Taking expectation on both sides of \eqref{eq:2_th7} with respect to $\mathbf{P}_{\tr}$ (since $\mathcal{X}_1$ has non-zero measure) gives
\begin{align*}
    \mathbf{E}_{\mathbf{P}_{\tr}}\left[\overset{\circ}{\mathrm{ES}}\left(\mathbf{P}_{g}(y|X),Y\right)\right]< \mathbf{E}_{\mathbf{P}_{\tr}}\left[\overset{\circ}{\mathrm{ES}}\left(\mathbf{P}_{\widetilde g}(y|X),Y\right)\right],
\end{align*}
which is a contradiction since $\widetilde g$ is the minimizer in \eqref{eq:engression_population_version}. So, we must have $\mathbf{P}(X\in \mathcal{X}_1)=0$, which means that $\widetilde g\sim \mathbf{P}_{\tr}(y|x)$. This completes the proof.
\end{proof}

\subsection{Proof of \Cref{lem:point_predictions_without_origin}}
\begin{proof}

By the population level guarantee (Theorem \ref{thm:population_guarantee_engression}) and our correct specification assumption, the minimizers satisfy $\widetilde g(x,\varepsilon)\sim \mathbf{P}_{\tr}(\cdot|x)$ and $\widetilde g_\alpha(x,\varepsilon)\sim \mathbf{P}_{\tr}^\alpha(\cdot|x)$ almost everywhere. By definition, $\mathbf{P}_{\tr}^\alpha(\cdot|x)$ is the distribution of $Y^\alpha=Y\oplus\alpha$, where $Y\sim \mathbf{P}_{\tr}(\cdot|x)$. This means that $\mathbf{P}_{\tr}^\alpha$ is the push-forward measure of $\mathbf{P}_{\tr}$ under the map $T_\alpha:\theta\mapsto \theta\oplus\alpha$, i.e., $\mathbf{P}_{\tr}^\alpha=(T_\alpha)_{*}\mathbf{P}_{\tr}$. Using the change of variables for push-forward measures, we can write
\begin{align}
    \label{eq:lem1_eq1}
    \mathbf{E}_{\Theta\sim (T_{\alpha})_{*}\mathbf{P}_{\tr}}[\phi(\Theta)]=\mathbf{E}_{\Theta\sim \mathbf{P}_{\tr}}[\phi(T_\alpha(\Theta))]=\mathbf{E}_{\Theta\sim \mathbf{P}_{\tr}}[\phi(\Theta\oplus\alpha)],
\end{align}
for any Borel measurable function $\phi$. Note that the map $T_\alpha$ is piecewise continuous on $[0, 2\pi)$, making it measurable with respect to the Borel $\sigma$-algebra on $[0, 2\pi)$.  Thus, the push-forward measure is well-defined for any Borel measurable $\phi$.

\begin{subproof}[Proof of (A)]
    Since $\widetilde g_\alpha\sim \mathbf{P}_{\tr}^\alpha$, we can write
    \begin{align}
        \label{eq:lem1_eq2}
        \mathbf{E}_\varepsilon[\sin \widetilde g_\alpha(x,\varepsilon)]=\mathbf{E}_{\Theta\sim \mathbf{P}_{\tr}^\alpha(\cdot|x)}[\sin \Theta],\quad  \mathbf{E}_\varepsilon[\cos \widetilde g_\alpha(x,\varepsilon)]=\mathbf{E}_{\Theta\sim \mathbf{P}_{\tr}^\alpha(\cdot|x)}[\cos \Theta].
    \end{align}

Applying \eqref{eq:lem1_eq1} with $\phi(\theta)=\sin \theta$ and $\phi(\theta)=\cos \theta$ to \eqref{eq:lem1_eq2}, and noting that the $2\pi$-periodicity of the trigonometric functions allows us to drop the modulo operator ($\phi(\Theta \oplus \alpha) = \phi(\Theta + \alpha)$), we obtain
\begin{align}
        \label{eq:lem1_eq3}
        \mathbf{E}_{\Theta\sim \mathbf{P}_{\tr}^\alpha(\cdot|x)}[\sin \Theta]= \mathbf{E}_{\Theta\sim \mathbf{P}_{\tr}(\cdot|x)}[\sin (\Theta+\alpha)],\quad \mathbf{E}_{\Theta\sim \mathbf{P}_{\tr}^\alpha(\cdot|x)}[\cos \Theta]= \mathbf{E}_{\Theta\sim \mathbf{P}_{\tr}(\cdot|x)}[\cos (\Theta+\alpha)].
    \end{align}
    
    Using the sum identities $\sin(\theta+\alpha)=\sin\theta\cos\alpha+\cos\theta\sin\alpha$ and $\cos(\theta+\alpha)=\cos\theta\cos\alpha-\sin\theta\sin\alpha$, and denoting $\bar s=\mathbf{E}_{\Theta\sim \mathbf{P}_{\tr}(\cdot|x)}[\sin \Theta]$ and $\bar c=\mathbf{E}_{\Theta\sim \mathbf{P}_{\tr}(\cdot|x)}[\cos \Theta]$, \eqref{eq:lem1_eq3} expands to
    \begin{align}\label{eq:lem1_eq4}
        \mathbf{E}_{\Theta\sim \mathbf{P}_{\tr}^\alpha(\cdot|x)}[\sin \Theta]=\bar s\cos\alpha+\bar c\sin \alpha = \bar s_\alpha,\quad  \mathbf{E}_{\Theta\sim \mathbf{P}_{\tr}^\alpha(\cdot|x)}[\cos\Theta]=\bar c\cos\alpha-\bar s\sin \alpha = \bar c_\alpha.
    \end{align}

   Using \eqref{eq:lem1_eq4}, we can write
\begin{align*}
    \bar c_\alpha+\iota \bar s_\alpha& =(\bar c\cos\alpha-\bar s\sin\alpha)+\iota(\bar s\cos\alpha+\bar c\sin\alpha)\\
    &= (\bar c+\iota\bar s)(\cos\alpha+\iota\sin\alpha)=(\bar c+\iota\bar s)\exp(\iota \alpha),
\end{align*}
and hence,
    \begin{align*}
        \widetilde \mu_\alpha(x)=\mathrm{atan2}(\bar s_\alpha,\bar c_\alpha)=\mathrm{atan2}(\bar s,\bar c)\oplus\alpha=\widetilde \mu(x)\oplus\alpha,
    \end{align*}
    which completes the proof. 
\end{subproof}

\begin{subproof}[Proof of (B)] 
For the sake of simplicity, assume that the circular median is unique. However, the proof holds even if the set of minimizers contains multiple values. Note that we can write for any $\theta_1,\theta_2\in [0,2\pi)$,
\begin{align}
    \label{eq:dmetricD}
    \mathrm{d}(\theta_1,\theta_2)=\mathrm{D}(\theta_1\ominus \theta_2),
\end{align}
where $\mathrm{D}(u)=\pi-\abs{\pi-u}$ for geodesic metric and $\mathrm{D}(u)=2\abs{\sin\left(\frac{u}{2}\right)}$ for chordal distance metric (both metrics are rotationally invariant). Applying \eqref{eq:lem1_eq1} with $\phi(\theta)=\mathrm{d}(\theta,\psi)$ for any fixed $\psi\in [0,2\pi)$, we have 
\begin{align*}
    \mathbf{E}_\varepsilon[\mathrm{d}(\widetilde g_\alpha(x,\varepsilon),\psi)]&= \mathbf{E}_{\Theta\sim \mathbf{P}_{\tr}^\alpha(\cdot|x)}[\mathrm{d}(\Theta,\psi)]\\
    &= \mathbf{E}_{\Theta\sim \mathbf{P}_{\tr}(\cdot|x)}[\mathrm{d}(\Theta\oplus \alpha,\psi)]\\
    &= \mathbf{E}_{\Theta\sim \mathbf{P}_{\tr}(\cdot|x)}[\mathrm{d}(\Theta,\psi\ominus \alpha)]\text{ (using $\mathrm{d}(\theta_1,\theta_2)=\mathrm{D}(\theta_1\ominus\theta_2)$)}\\
    &= \mathbf{E}_{\varepsilon}[\mathrm{d}(\widetilde g(x,\varepsilon),\psi\ominus \alpha)]\text{ (since $\widetilde g(x,\varepsilon)\sim \mathbf{P}_{\tr}(\cdot|x)$)}.
\end{align*}
Taking $\argmin$ over $[0,2\pi)$ on both sides, we get 
\begin{align*}
    \argmin_{\psi\in [0,2\pi)}\left\{\mathbf{E}_\varepsilon[\mathrm{d}(\widetilde g_\alpha(x,\varepsilon),\psi)]\right\}&= \argmin_{\psi\in [0,2\pi)}\left\{\mathbf{E}_{\varepsilon}[\mathrm{d}(\widetilde g(x,\varepsilon),\psi\ominus \alpha)]\right\}\\
    &= \argmin_{\psi\in [0,2\pi)}\left\{\mathbf{E}_{\varepsilon}[\mathrm{d}(\widetilde g(x,\varepsilon),\psi)]\right\}\oplus \alpha,
\end{align*}
or equivalently, $m_{\widetilde g_\alpha}(x)=m_{\widetilde g}(x)\oplus \alpha$ $\mathbf{P}_X$-almost everywhere. This completes the proof.
\end{subproof}
 

\begin{subproof}[Proof of (C)] We have
\begin{align}
    \notag \sigma_{\widetilde g_\alpha}(x)&=\textbf{Med}_\varepsilon[\mathrm{d}(\widetilde g_\alpha(x,\varepsilon),m_{\widetilde g_\alpha}(x))]\\ \notag
    &= \textbf{Med}_{\Theta\sim \mathbf{P}_{\tr}^\alpha(\cdot|x)}\left[\mathrm{d}(\Theta,m_{\widetilde g_\alpha}(x))\right]\text{ (since $\widetilde g_\alpha\sim \mathbf{P}_{\tr}^\alpha(\cdot|x)$)}\\ \notag
    &= \textbf{Med}_{\Theta\sim \mathbf{P}_{\tr}(\cdot|x)}\left[\mathrm{d}(\Theta\oplus \alpha,m_{\widetilde g_\alpha}(x))\right]\text{ (using \eqref{eq:lem1_eq1} with $\phi(\theta)=\mathrm{d}(\theta,m_{\widetilde g_\alpha}(x)))$}\\ \label{eq:last_step_dispersion_req}
    &= \textbf{Med}_{\Theta\sim \mathbf{P}_{\tr}(\cdot|x)}\left[\mathrm{d}(\Theta\oplus\alpha,m_{\widetilde g}(x)\oplus\alpha)\right]\text{ (using part (B))}
    \end{align}

Due to \eqref{eq:dmetricD}, we must have
$\mathrm{d}(\Theta\oplus\alpha, m_{\widetilde g}(x)\oplus\alpha) = \mathrm{D}((\Theta\oplus\alpha) \ominus (m_{\widetilde g}(x)\oplus\alpha))$.
Using the properties of modulo $2\pi$ arithmetic over $[0, 2\pi)$, the $\alpha$ terms cancel, yielding $(\Theta\oplus\alpha) \ominus (m_{\widetilde g}(x)\oplus\alpha) = \Theta\ominus m_{\widetilde g}(x)$. Therefore, $$\mathrm{d}(\Theta\oplus\alpha, m_{\widetilde g}(x)\oplus\alpha) = \mathrm{D}(\Theta\ominus m_{\widetilde g}(x)) = \mathrm{d}(\Theta, m_{\widetilde g}(x)).$$
Substituting this back into our median expression, we obtain:
\begin{align*}
    \sigma_{\widetilde g_\alpha}(x)&= \textbf{Med}_{\Theta\sim \mathbf{P}_{\tr}(\cdot|x)}\left[\mathrm{d}(\Theta,m_{\widetilde g}(x))\right]\\
    &= \textbf{Med}_{\varepsilon}\left[\mathrm{d}(\widetilde g(x,\varepsilon),m_{\widetilde g}(x))\right]=\sigma_{\widetilde g}(x),
\end{align*}
which holds $\mathbf{P}_X$-almost everywhere. This completes the proof.
\end{subproof}

\begin{subproof}[Proof of (D)]
By definition of the expected trigonometric moments, we can write:
\begin{align*}
    \psi_{\widetilde g}(r|x)&=\mathbf{E}_{\varepsilon}\left[\exp(\iota r\widetilde g(x,\varepsilon))\right]=\mathbf{E}_{\Theta\sim \mathbf{P}_{\tr}(\cdot|x)}\left[\exp(\iota r\Theta)\right],\\
    \psi_{\widetilde g_\alpha}(r|x)&=\mathbf{E}_{\varepsilon}\left[\exp(\iota r\widetilde g_\alpha(x,\varepsilon))\right]=\mathbf{E}_{\Theta\sim \mathbf{P}_{\tr}^\alpha(\cdot|x)}\left[\exp(\iota r\Theta)\right].
\end{align*}
Applying the push-forward change of variables \eqref{eq:lem1_eq1} with $\phi(\theta)=\exp(\iota r\theta)$, we have:
\begin{align*}
    \psi_{\widetilde g_\alpha}(r|x)& =\mathbf{E}_{\Theta\sim \mathbf{P}_{\tr}^\alpha(\cdot|x)}[\exp(\iota r\Theta)]=\mathbf{E}_{\Theta\sim \mathbf{P}_{\tr}(\cdot|x)}[\exp(\iota r(\Theta\oplus \alpha))].
\end{align*}
Because the complex exponential function $\exp(\iota r \cdot)$ is $2\pi$-periodic for any integer $r\in \mathbb{Z}$, the modulo addition is equivalent to standard addition in the exponent. Thus, $\exp(\iota r(\Theta\oplus\alpha))=\exp(\iota r(\Theta+\alpha))=\exp(\iota r\alpha)\exp(\iota r\Theta)$. Factoring out the constant, we have
\begin{align*}
    \psi_{\widetilde g_\alpha}(r|x) = \exp(\iota r\alpha)\mathbf{E}_{\Theta\sim \mathbf{P}_{\tr}(\cdot|x)}[\exp(\iota r\Theta)] = \exp(\iota r\alpha)\psi_{\widetilde g}(r|x).
\end{align*}
Substituting this into the truncated density equation, we obtain:
\begin{align*}
    \mathcal{P}_{\widetilde g_\alpha}(\theta|x)&=(2\pi)^{-1}\sum_{r=-R}^{R}{\psi_{\widetilde g_\alpha}(r|x)\exp(-\iota r\theta)}\\
    &= (2\pi)^{-1}\sum_{r=-R}^{R}{\psi_{\widetilde g}(r|x)\exp(\iota r\alpha)\exp(-\iota r\theta)}\\
    &= (2\pi)^{-1}\sum_{r=-R}^{R}{\psi_{\widetilde g}(r|x)\exp(-\iota r(\theta - \alpha))}.
\end{align*}
Since the exponential function evaluates evaluating the phase modulo $2\pi$, we can replace $(\theta - \alpha)$ with $(\theta \ominus \alpha)$, resulting in $\mathcal{P}_{\widetilde g}(\theta\ominus\alpha|x)$ in the RHS. This completes the proof.
\end{subproof}


\begin{subproof}[Proof of (E)] 
For the sake of simplicity, assume the conditional mode is unique. Using part (D), we have:
\begin{align*}
    \bar{\upsilon}_{\widetilde g_\alpha}(x) = \argmax_{\theta \in [0, 2\pi)} \left\{\mathcal{P}_{\widetilde g_\alpha}(\theta|x)\right\} = \argmax_{\theta \in [0, 2\pi)} \left\{\mathcal{P}_{\widetilde g}(\theta \ominus \alpha|x)\right\}.
\end{align*}
Let $\theta' = \theta \ominus \alpha$. Because the mapping $\theta \mapsto \theta \ominus \alpha$ is a bijection on $[0, 2\pi)$, evaluating the maximum over $\theta$ is equivalent to evaluating the maximum over $\theta'$ and applying the inverse transformation $\theta = \theta' \oplus \alpha$. Therefore:
\begin{align*}
    \argmax_{\theta \in [0, 2\pi)} \left\{\mathcal{P}_{\widetilde g}(\theta \ominus \alpha|x)\right\} &= \left( \argmax_{\theta' \in [0, 2\pi)} \left\{\mathcal{P}_{\widetilde g}(\theta'|x)\right\} \right) \oplus \alpha \\
    &= \bar{\upsilon}_{\widetilde g}(x) \oplus \alpha.
\end{align*}
This establishes that $\bar{\upsilon}_{\widetilde g_\alpha}(x) = \bar{\upsilon}_{\widetilde g}(x) \oplus \alpha$. The proof for the conditional antimode, $\underline{\upsilon}_{\widetilde g_\alpha}(x) = \underline{\upsilon}_{\widetilde g}(x) \oplus \alpha$, follows identically by replacing $\argmax$ with $\argmin$. This completes the proof.
\end{subproof}

\end{proof}

\subsection{Proof of \Cref{thm:central_sigma_field_exists}}
\begin{proof}
    Let $\Pi_y=\mathbf{P}_{X|Y}(\cdot|y)$ and $\overset{\circ}{\mathbf{P}}=\{\Pi_y|y\in \mathbb{S}^1\}$. Since $\overset{\circ}{\mathbf{P}}$ is dominated by a $\sigma$-finite measure, it contains a countable subset $\overset{\circ}{\mathbf{Q}}=\{Q_k|k=1,2,\ldots\}$ such that $\overset{\circ}{\mathbf{P}}\equiv \overset{\circ}{\mathbf{Q}}$ (we use $\equiv$ to denote that one measure dominates the other). Let $\{c_k\}_{k=1,2,\ldots}$ be a sequence of positive numbers adding up to $1$ and let $\overset{\circ}{\mathbf{Q}_0}=\sum_{k=1}^{\infty}{c_kQ_k}$. Then  $\overset{\circ}{\mathbf{Q}_0}$ is a probability measure on $\mathcal{X}$ such that $\{\overset{\circ}{\mathbf{Q}_0}\}\equiv \overset{\circ}{\mathbf{Q}}\equiv\overset{\circ}{\mathbf{P}}$. Let $\pi_y=\frac{d\Pi_y}{d\overset{\circ}{\mathbf{Q}_0}}$ and let $\mathfrak{G}$ be a sub-sigma field of $\sigma(X)$. We claim that $Y\independent X|\mathfrak{G}$ if and only if $\pi_y$ is essentially measurable with respect to $\mathfrak{G}$ for all $y\in \mathbb{S}^1$ modulo $\overset{\circ}{\mathbf{Q}_0}$. To show the necessary part, let $B\in \sigma(X)$. Then, 
    \begin{align}
       \notag \mathbf{E}_{\overset{\circ}{\mathbf{Q}}_0}\left[\pi_y(X)\mathbf{1}_{\{X\in B\}}\right]&= \mathbf{E}_{\Pi_y}[\mathbf{1}_{\{X\in B\}}]\\ \notag
        &= \mathbf{E}_{\Pi_y}\left[\mathbf{E}_{\Pi_y}\left(\mathbf{1}_{\{X\in B\}}|\mathfrak{G}\right)\right]\\ \label{eq:existence_eq1}
        &= \mathbf{E}_{\overset{\circ}{\mathbf{Q}}_0}\left[\pi_y(X)\mathbf{E}_{\Pi_y}\left(\mathbf{1}_{\{X\in B\}}|\mathfrak{G}\right)\right].
    \end{align} 
    Now, since $Y\independent X|\mathfrak{G}$, $\Pi_y(B|\mathfrak{G})$ is same for all $y\in \mathbb{S}^1$. Hence $\Pi_y(B|\mathfrak{G})=Q_k(B|\mathfrak{G})$ for all $k$, implying that $\Pi_y(B|\mathfrak{G})=\overset{\circ}{\mathbf{Q}_0}(B|\mathfrak{G})$. So, we can write the RHS of \eqref{eq:existence_eq1} as 
    \begin{align*}
        \mathbf{E}_{\overset{\circ}{\mathbf{Q}}_0}\left[\pi_y(X)  \mathbf{E}_{\overset{\circ}{\mathbf{Q}}_0}\left(\mathbf{1}_{\{X\in B\}}|\mathfrak{G}\right)\right]=   \mathbf{E}_{\overset{\circ}{\mathbf{Q}}_0}\left[\mathbf{1}_{\{X\in B\}}  \mathbf{E}_{\overset{\circ}{\mathbf{Q}}_0}\left(\pi_y(X)|\mathfrak{G}\right)\right].
    \end{align*}
    Using this in \eqref{eq:existence_eq1}, we can write
    \begin{align}\label{eq:existence_eq2}
         \mathbf{E}_{\overset{\circ}{\mathbf{Q}}_0}\left[\pi_y(X)\mathbf{1}_{\{X\in B\}}\right]=\mathbf{E}_{\overset{\circ}{\mathbf{Q}}_0}\left[\mathbf{1}_{\{X\in B\}}  \mathbf{E}_{\overset{\circ}{\mathbf{Q}}_0}\left(\pi_y(X)|\mathfrak{G}\right)\right],
    \end{align}
    which implies $\pi_y(X)=\mathbf{E}_{\overset{\circ}{\mathbf{Q}}_0}\left[\pi_y(X)|\mathfrak{G}\right]$-a.s. $\overset{\circ}{\mathbf{Q}_0}$, since \eqref{eq:existence_eq2} holds for all $B\in \sigma(X)$. Next, we prove the sufficiency part. Assume that $\pi_y$ is essentially measurable with respect to $\mathfrak{G}$ for all $y\in \mathbb{S}^1$ modulo $\overset{\circ}{\mathbf{Q}_0}$, which is equivalent to 
    \begin{align}
        \label{eq:existence_eq3}
        \mathbf{E}_{\overset{\circ}{\mathbf{Q}_0}}\left[\pi_y(X)|\mathfrak{G}\right]=\pi_y(X).
    \end{align}
    Then for any $A\in \mathfrak{G}$, we can write 
    \begin{align*}
        \mathbf{E}_{\pi_y}\left[\mathbf{1}_{\{X\in A\}}\mathbf{E}_{\overset{\circ}{\mathbf{Q}_0}}\left(\mathbf{1}_{\{X\in B\}}|\mathfrak{G}\right)\right]& = \mathbf{E}_{\overset{\circ}{\mathbf{Q}}_0}\left[\mathbf{1}_{\{X\in A\}}\pi_y(X)\mathbf{E}_{\overset{\circ}{\mathbf{Q}}_0}\left(\mathbf{1}_{\{X\in B\}}|\mathfrak{G}\right)\right]\\
        &= \mathbf{E}_{\overset{\circ}{\mathbf{Q}_0}}\left[\mathbf{1}_{\{X\in A\}}\mathbf{1}_{\{X\in B\}}\mathbf{E}_{\overset{\circ}{\mathbf{Q}_0}}\left(\pi_y(X)|\mathfrak{G}\right)\right]\\
        &= \mathbf{E}_{\overset{\circ}{\mathbf{Q}_0}}\left[\mathbf{1}_{\{X\in A\}}\mathbf{1}_{\{X\in B\}}\pi_y(X)\right]\text{ (due to \eqref{eq:existence_eq3})}\\
        &= \mathbf{E}_{\pi_y}[\mathbf{1}_{\{X\in A\}}\mathbf{1}_{\{X\in B\}}]=\Pi_y(X\in A \cap B).
    \end{align*}
    Therefore, we can conclude that $\mathbf{E}_{\overset{\circ}{\mathbf{Q}}_0}\left[\mathbf{1}_{\{X\in B\}}|\mathfrak{G}\right]=\overset{\circ}{\mathbf{Q}_0}(B|\mathfrak{G})$ is the conditional probability $\Pi_y(B|\mathfrak{G})$, which means that $\Pi_y(B|\mathfrak{G})$ does not depend on $y$, hence $Y\independent X|\mathfrak{G}$. Now let $\mathfrak{G}^*$ be the intersection of all SDR sigma fields $\mathfrak{G}$. Then $\mathfrak{G}^*$ is itself a sigma field. Moreover, since $\pi_y$ is measurable with respect to all SDR sigma fields for all $y\in \mathbb{S}^1$, it should also be measurable with respect to $\mathfrak{G}^*$ for all $y\in \mathbb{S}^1$. Hence, $\mathfrak{G}^*$ is itself an SDR sigma field, which implies that it is also the smallest SDR sigma field. If $\mathfrak{G}_0$ is another smallest SDR sigma field then we know that $\mathfrak{G}^*\subseteq \mathfrak{G}_0$ and $\mathfrak{G}_0\subseteq\mathfrak{G}^*$. Therefore, $\mathfrak{G}^*$ is unique. 
\end{proof} 

\subsection{Proof of \Cref{thm:central_subspace_finding_engression}}
\begin{proof}
   Since $\varepsilon$ is independent of $X$, under correct specification, $\mathbf{P}_{Y|X}$ depends on $X$ only through $\beta_0^\top X$. That is, for any Borel measurable set $B\subset [0,2\pi)$, $\mathbf{P}(Y\in B|X=x)=\mathbf{P}(Y\in B|\beta_0^\top X=\beta_0^\top x)$, which means that  we should have $Y\independent X|\sigma(\beta_0^\top X)$. This means that $\sigma(\beta_0^\top X)$ is always an SDR sigma field and therefore, we must have $\mathfrak{G}^*_{Y|X}\subseteq \sigma(\beta_0^\top X)$. Next, we need to show that $\sigma(\beta_0^\top X)\subseteq \mathfrak{G}^*_{Y|X}$. Suppose for the sake of contradiction that $\sigma(\beta_0^\top X)\nsubseteq \mathfrak{G}^*_{Y|X}$. We can write $\mathfrak{G}^*_{Y|X}=\sigma(R(X))$ for some measurable function $R$ such that $\sigma(\beta_0^\top X)\nsubseteq \sigma(R(X))$. Now since $\sigma(R(X))$ is an SDR sigma field, the conditional distribution $\mathbf{P}_{Y|X}$ depends on $X$ only through $R(X)$, i.e.,
\begin{align*} 
    \mathbf{P}_{Y|X=x}=\mathbf{P}_{Y|R(X)=R(x)},\quad \mathbf{P}_X\text{-a.s.}.
\end{align*}
Therefore, we can write
\begin{align*}
    R(x_1)=R(x_2)\implies \mathbf{P}_{Y|R(X)=R(x_1)}=\mathbf{P}_{Y|R(X)=R(x_2)}&\implies \mathbf{P}_{Y|X=x_1}=\mathbf{P}_{Y|X=x_2}\\ 
    &\implies \beta_0^\top x_1=\beta_0^\top x_2,
\end{align*}
where the last step holds because the map $t\mapsto \mathbf{P}_{Y|\beta_0^\top X=t}$ is injective (note that injectivity of $f_0$ guarantees the injectivity of the induced measure $\mathbf{P}_{Y|\cdot}$). The above chain means that $\beta_0^\top X$ is a function of $R(X)$ almost surely. Hence, using Doob-Dynkin's lemma, $\beta_0^\top X$ is measurable with respect to $\sigma(R(X))$, i.e., $\sigma(\beta_0^\top X)\subseteq \sigma(R(X))$, which is a contradiction. Thus, we must have $\sigma(\beta_0^\top X)=\mathfrak{G}^*_{Y|X}$, which completes the proof. 
\end{proof}

\subsection{Proof of \Cref{thm:central_subspace_recovery}}
\begin{proof}
 By \Cref{thm:population_guarantee_engression}, under correct specification, the population minimizer satisfies $\widetilde g\left(\widetilde\beta^\top x,\widetilde h(\varepsilon)\right)\sim \mathbf{P}_{Y|x}$ for all $x\in \mathcal{X}$. Because the conditional distribution of $Y$ given $X=x$ depends on $x$ strictly through the linear combination $\widetilde\beta^\top x$, it follows that $Y \independent X \mid \widetilde\beta^\top X$, meaning that $\sigma(\widetilde\beta^\top X)$ is an SDR sigma field. From \Cref{thm:central_subspace_finding_engression}, the unique central subspace satisfies $\mathfrak{G}^*_{Y|X}=\sigma(\beta_0^\top X)$. By the minimality of the central subspace, any other sub-sigma field that satisfies the SDR condition must contain it. Therefore, we must have $\sigma\left(\beta_0^\top X\right)\subseteq \sigma\left(\widetilde \beta^\top X\right)$. By the Doob-Dynkin factorization lemma, this inclusion implies the existence of a Borel measurable function $\psi:\mathbb{R}\rightarrow\mathbb{R}$ such that $\beta_0^\top X=\psi\left(\widetilde\beta^\top X\right)$ almost surely.  Suppose for the sake of contradiction that $\beta_0$ and $\widetilde\beta$ are linearly independent vectors in $\mathbb{R}^d$. Define the linear transformation matrix $B=\left[\widetilde\beta,\beta_0\right]^\top$, which by assumption has full row rank $2$. Because $X$ is absolutely continuous with respect to the Lebesgue measure on $\mathbb{R}^d$, its projection onto this full-rank two-dimensional subspace, given by the random vector $Z=\left(\widetilde\beta^\top X,\beta_0^\top X\right)^\top$, must be absolutely continuous with respect to the two-dimensional Lebesgue measure $\lambda_2$ on $\mathbb{R}^2$. Consequently, for any Borel set $A\subset\mathbb{R}^2$ with $\lambda_2(A)=0$, we must have $\mathbf{P}(Z\in A)=0$. 
 Define $\Psi=\{(u,v)\in \mathbb{R}^2 \mid v=\psi(u)\}$ as the graph of the function $\psi$. Because $\psi$ is Borel measurable, Fubini's Theorem guarantees that its graph has a 2D Lebesgue measure of zero ($\lambda_2(\Psi)=0$). Utilizing the absolute continuity of $Z$, this implies $\mathbf{P}\left[\left(\widetilde\beta^\top X, \beta_0^\top X\right)^\top \in \Psi \right] = 0$. However, our previous deduction states that $\beta_0^\top X=\psi\left(\widetilde\beta^\top X\right)$ almost surely, meaning the random vector resides on the graph with probability one: $\mathbf{P}\left[\left(\widetilde\beta^\top X, \beta_0^\top X\right)^\top \in \Psi \right] = 1$. This is a contradiction. We conclude that $\beta_0$ and $\widetilde\beta$ cannot be linearly independent. Since neither vector is the zero vector (assuming non-trivial dimension reduction), they must be collinear. Thus, we can write $\widetilde\beta=c\beta_0$ for some non-zero scalar $c$. This implies that $\sigma(\widetilde\beta^\top X) = \sigma(\beta_0^\top X)$. Therefore, $\mathfrak{G}^*_{Y|X}=\sigma(\widetilde\beta^\top X)$, completing the proof. 
\end{proof}

\subsection{Proof of \Cref{thm:extrapolability_criteria}}
\begin{proof}
    We illustrate the proof for part (B); the proof for part (A) follows identically by dropping the additive noise $h(\varepsilon)$ from the argument of $g$. Fix $\varepsilon\in[0,1]$, and let $M_1, M_2 \in \mathbf{M}_{\text{pre}}^{(2)}$ be two models parameterized by $(g_1, \beta_1, h_1)$ and $(g_2, \beta_2, h_2)$ that induce the identical conditional distribution on the training support:
    \begin{align}
        \label{eq:1_th6}
        g_1\left(\beta_1^\top x+h_1(\varepsilon)\right)\overset{d}{=}g_2\left(\beta_2^\top x+h_2(\varepsilon)\right),\quad \forall x\in \mathcal{X}.
    \end{align}
    By the identifiability guarantee in \Cref{thm:identifiability_all}(D), it follows that $\beta_1 = \beta_2 = \beta$, $h_1=h_2=h$ on $[0,1]$, and $g_1(\beta^\top x+h(\varepsilon))=g_2(\beta^\top x+h(\varepsilon))$ almost surely for $x \in \mathcal{X}$. 
    
    Because $\varepsilon\sim \mathcal{U}(0,1)$ and $h$ is continuous and strictly increasing (\Cref{assmp:h_related}), the image $h([0,1])$ is a closed interval $[a,b]$ for some $a<b$. Fix any interior point $x_0\in \mathcal{X}$. We can write $g_1(\beta^\top x_0+t)=g_2(\beta^\top x_0+t)$ for all $t\in [a,b]$. This implies that $g_1$ and $g_2$ perfectly agree on the closed interval $[\beta^\top x_0+a, \beta^\top x_0+b]$. Because a closed interval contains infinitely many accumulation points, and both $g_1$ and $g_2$ are real analytic on the connected open set $\mathcal{I}$, the identity theorem for real analytic functions \citep{krantz2002primer} dictates that $g_1=g_2$ everywhere on $\mathcal{I}$.
    
    Next, we establish the topological boundary for extrapolation. Define the projected parameter space for this specific model as $\mathfrak{T}_{pre} = \{\beta^\top x+h(\varepsilon) \mid x\in \mathcal{X}, \varepsilon \in [0,1]\}$. Because $\mathcal{X}$ is compact, $h$ is continuous, and the projection is linear, $\mathfrak{T}_{pre}$ is a compact subset of $\mathbb{R}$ (see \Cref{lemma:compactness_condition}). By assumption, $\mathfrak{T}_{pre}\subset \mathcal{I}$. Define the set distance to the boundary of the analytic domain:
    \begin{align*}
        \Delta = \mathrm{d}_s\left(\mathfrak{T}_{pre}, \mathbb{R}\setminus\mathcal{I}\right) = \inf_{s\in \mathfrak{T}_{pre}, u\in \mathbb{R}\setminus\mathcal{I}} \abs{s-u}.
    \end{align*}
    Because $\mathcal{I}$ is open, its complement $\mathbb{R}\setminus\mathcal{I}$ is closed. The distance between a compact set ($\mathfrak{T}_{pre}$) and a disjoint closed set ($\mathbb{R}\setminus\mathcal{I}$) is strictly positive. Thus, $\Delta > 0$. 
    
    Now, let $x^*$ be an extrapolation point satisfying $\inf_{x\in \mathcal{X}} \norm{x-x^*}_2 \le \delta$ for some $\delta > 0$. Let $x_c \in \mathcal{X}$ be the closest point to $x^*$ in the support. By the Cauchy-Schwarz inequality and the constraint $\norm{\beta} = 1$ (\Cref{assmp:B_related}), the distance in the projected scalar space is bounded by $\delta$:
    \begin{align*}
        \abs{\left(\beta^\top x^*+h(\varepsilon)\right)-\left(\beta^\top x_c+h(\varepsilon)\right)} = \abs{\beta^\top(x^* - x_c)} \le \norm{\beta} \norm{x^* - x_c} \le \delta.
    \end{align*}
    Therefore, the projected point $\beta^\top x^*+h(\varepsilon)$ lies within a $\delta$-neighborhood of $\mathfrak{T}_{pre}$. By choosing $\delta_0 < \Delta$, any $\delta \le \delta_0$ guarantees that $\beta^\top x^*+h(\varepsilon)$ strictly remains inside the analytic domain $\mathcal{I}$ for all $\varepsilon \in [0,1]$. Because we have already established that $g_1=g_2$ everywhere on $\mathcal{I}$, it follows that:
    \begin{align*}
        g_1(\beta^\top x^*+h(\varepsilon))=g_2(\beta^\top x^*+h(\varepsilon)),\quad \forall \varepsilon\in [0,1].
    \end{align*}
    This establishes that $\mathrm{D}(\mathbf{P}_1(y|x^*),\mathbf{P}_2(y|x^*))=0$. Because this holds for any $x^*$ within a $\delta$-neighborhood of $\mathcal{X}$, taking the supremum yields $U_{\mathbf{M}_{\text{pre}}^{(2)}}(\delta)=0$ for all $\delta \le \delta_0$. This completes the proof for the fully identifiable model classes. 

    \begin{subproof}[Proof for models under $2k\pi$-identifiability trap ($\mathbf{M}_{\text{pre}}^{(1)}$ and $\mathbf{M}_{\text{post}}^{(1)}$)]
    We detail the proof for $\mathbf{M}_{\text{pre}}^{(1)}$; the logic applies identically to $\mathbf{M}_{\text{post}}^{(1)}$ by modifying the position of the additive noise. 
Let $M_1, M_2 \in \mathbf{M}_{\text{pre}}^{(1)}$ be two models that induce the identical conditional distribution on the training support. By the $2k\pi$-identifiability guarantee (\Cref{thm:identifiability_all}(C)), the underlying parameters are not uniquely identified; instead, they satisfy $h_1 = h_2 = h$, $\beta_1 = \beta_2 = \beta$, and $g_1(\beta^\top x + h(\varepsilon)) = g_2(\beta^\top x + h(\varepsilon)) + 2k\pi$ for some integer $k \in \mathbb{Z}$, almost surely for $x \in \mathcal{X}$. Using the exact same topological arguments established previously, $g_1$ and $g_2$ must maintain a constant difference of $2k\pi$ over a closed interval in the projected space. Because $g_1$ and $g_2$ are real analytic on the connected open set $\mathcal{I}$, the function $f(z) = g_1(z) - g_2(z) - 2k\pi$ is also real analytic. Since $f(z) = 0$ on a closed interval, the identity theorem \citep{krantz2002primer} requires that $f(z) = 0$ everywhere on $\mathcal{I}$. Therefore, $g_1(z) = g_2(z) + 2k\pi$ across the entire analytic domain $\mathcal{I}$. Now, consider an extrapolation point $x^*$ such that $\beta^\top x^* + h(\varepsilon) \in \mathcal{I}$. The generated circular responses for the two models at this extrapolation point are given by:
    \begin{align*}
        Y_1^* &= g_1(\beta^\top x^* + h(\varepsilon)) \bmod 2\pi \\
        Y_2^* &= g_2(\beta^\top x^* + h(\varepsilon)) \bmod 2\pi
    \end{align*}
    Substituting the analytic continuation $g_1(z) = g_2(z) + 2k\pi$ into the first equation yields:
    \begin{align*}
        Y_1^* &= \left(g_2(\beta^\top x^* + h(\varepsilon)) + 2k\pi \right) \bmod 2\pi.
    \end{align*}
    By the modulo $2\pi$ arithmetic, the additive $2k\pi$ shift vanishes, and we get
    \begin{align*}
        Y_1^* &= g_2(\beta^\top x^* + h(\varepsilon)) \bmod 2\pi = Y_2^*.
    \end{align*}
    Because $Y_1^* = Y_2^*$ deterministically for all $\varepsilon \in [0,1]$, the conditional distributions match: $\mathbf{P}_1(y|x^*) = \mathbf{P}_2(y|x^*)$. Consequently, $\mathrm{D}(\mathbf{P}_1(y|x^*),\mathbf{P}_2(y|x^*))=0$. Taking the supremum over the $\delta$-neighborhood confirms $U_{\mathbf{M}_{\text{pre}}^{(1)}}(\delta)=0$, proving that the model is extrapolable despite being not fully identifiable.
\end{subproof}
\end{proof}

\begin{remark}
    Identifiability is not a strictly necessary condition for distributional extrapolability. For example, \cite{shen2025engression} note that the class of linear pre-ANMs is not identifiable, but distributionally extrapolable.
\end{remark}

\section{Baseline Models and Implementation} \label{appendix:Baseline_models}
In this section, we briefly describe the baseline models used for comparison along with their implementation strategies. As discussed in the main text, we do not consider Bayesian implementations for all models except vMQP. This is primarily due to two reasons: to avoid the excessive computational overhead of Bayesian training paradigms, and the non-availability of original open-source codes for most of the frameworks. The task-specific usage of each model for comparison in this study is mentioned in \Cref{table:baseline_models} in the main paper.

\begin{enumerate}
    \item \textbf{Circular Linear Regression (CLR):} 
    The projected multivariate linear model proposed by \cite{presnell1998projected} models a circular response variable by treating it as the radial projection of an unobserved, latent bivariate normal random vector onto the unit circle. The mean of this latent Euclidean vector is formulated as a linear combination of the input covariates, and the corresponding regression weights are efficiently estimated via an Expectation-Maximization (EM) algorithm. The code implementation is taken from a public repository\footnote{https://gist.github.com/ahwillia/2941fb64e8bfe999c66291257601dfb4}. Natively, the framework is designed to map purely linear (Euclidean) covariates to a directional response. However, to keep a fair comparison strategy, we embed the circular predictors (if any) trigonometrically into continuous sine and cosine components prior to incorporating them into the model's design matrix.

    \item \textbf{Kernel CLR:} 
    Nonparametric kernel regression for a circular response \citep{taylor2012non, Alonso_Pena_2025} extends standard local smoothing techniques to accommodate directional data. Because angular responses cannot be averaged arithmetically without artificially depending on the zero-direction, the circular response is first decomposed into its orthogonal sine and cosine components. The model then computes a locally weighted average for both components, where the assigned weights are determined by a kernel function that measures the proximity of the evaluation point to the training covariates. The final predicted mean direction is subsequently reconstructed by applying the two-argument arctangent function ($\mathrm{atan2}$) to the smoothed sine and cosine estimates. We implement this nonparametric framework to natively process both linear and circular covariates simultaneously. This is achieved by utilizing a multivariate product kernel. Specifically, the joint localized weight for any given data point is constructed by multiplying individual feature kernels: a standard Euclidean kernel (such as a Gaussian kernel governed by a bandwidth parameter $h$) is applied to the continuous linear features, while a directional kernel (such as a von Mises kernel governed by a concentration parameter $\kappa$) is applied to the circular features. This formulation allows the model to natively map mixed-domain inputs directly to a circular response without requiring artificial embedding of the circular predictors prior to fitting.

    \item \textbf{Skew-Circular Beta Regressor (SCBR):}
    The Skew-Circular Beta (SCB) regression model \citep{hassanzadeh2021smoothing} provides a parametric framework for predicting multimodal and asymmetric angular responses. The model governs the circular distribution through a global location parameter $\mu \in [0, 2\pi)$ and two global shape parameters $a, b > 0.5$ which control the concentration and multimodality. To capture complex, observation-specific dynamics, the model dynamically predicts a directional skewness parameter $\lambda_i \in [-1, 1]$ conditioned on the input covariates. This is achieved using a linear predictor $\eta_i = \mathbf{x}_i^\top \beta + \text{bias}$, which is subsequently squeezed into the valid mathematical boundaries via an arctangent link function, $\lambda_i = \frac{2}{\pi} \arctan(\eta_i)$. The SCB architecture aggregates features strictly through a standard Euclidean linear predictor, and does not natively process circular covariates. We feed the sin-cos embeddings of the circular features to the model for a consistent evaluation strategy. Although \cite{hassanzadeh2021smoothing} presents a Bayesian scheme to estimate the model parameters, we consider two distinct non-Bayesian training paradigms: Maximum Likelihood Estimation (MLE) and Mean Circular Prediction Error (MCPE) minimization, which the author uses as a model selection criterion. Under the MLE setting, the network minimizes the negative log-likelihood (NLL) of the SCB distribution, employing strict numerical clamping within the logarithmic terms to prevent gradient explosion when angular differences approach $-1$. Conversely, in the MCPE setting, the model operates as a geometric point-predictor by minimizing the cosine distance between the ground truth and the theoretical mean direction derived via the distribution's trigonometric moments.

    \item \textbf{LEP-mGvM:}
    The Lifted Expectation Propagation for the multivariate Generalised von Mises (LEP-mGvM; \cite{wu2019probabilistic}) model performs approximate Bayesian inference by mapping the $N$-dimensional circular target space into a relaxed $2N$-dimensional Euclidean space (representing the cosine and sine components) using Gaussian site approximations. While the theoretical LEP-mGvM framework dictates that the EP moment-matching step must compute the exact trigonometric moments of the GvM distribution (a process relying on unstable infinite series of modified Bessel functions), our implementation adopts a numerically robust simplification. Specifically, we utilize a first-order von Mises approximation by discarding the second harmonic parameter ($\kappa_2$) and employing exponentially scaled Bessel ratios, which maintains computational stability while preserving the strict annealing and damping schedules mandatory for EP convergence. Additionally, for out-of-sample predictions, the implemented model approximates the full $2N$ joint Gaussian posterior conditioning with a streamlined kernel projection of the learned site means. Regarding input compatibility, the theoretical Gaussian Process formulation natively ingests mixed linear and circular covariates through the construction of composite covariance functions. However, because our implementation relies exclusively on a standard Euclidean squared-exponential kernel, circular covariates must first be exogenously embedded into continuous orthogonal components (sine and cosine) prior to passing them to the model.

    \item \textbf{Multiple Circular Regression II (MCR2):}
    The Multiple Circular Regression II (MCR2) model \citep{jha2017multiple} is a geometric framework that predicts a circular response from multiple circular covariates by employing Möbius transformations. Theoretically, it linearly combines complex-valued circular predictors using convex mixing weights and rotation parameters, projects the aggregate onto the unit circle, applies a Möbius transformation, and performs a final global rotation to yield the predicted mean direction. In our PyTorch implementation, the convex weights are enforced via a softmax activation, and the first covariate's rotation is strictly anchored to zero to guarantee theoretical identifiability. To ensure stable gradient-based optimization, the complex arithmetic of the Möbius transformation is decoupled into explicit real and imaginary components, and the network is trained by minimizing the circular distance, an objective equivalent to maximizing the von Mises log-likelihood. The MCR2 architecture is rigidly formulated for angular inputs and does not natively process continuous linear covariates.

    \item \textbf{vMQP:}
    The von Mises Quasi-Process (vMQP) is a non-parametric Bayesian regression model specifically designed for circular data \citep{cohen2025bayesian}. It is formulated by taking a Gaussian process over two-dimensional Euclidean random functions and conditioning it on the unit circle. Unlike alternative Gaussian process-based adaptations for directional data that rely on wrapping or radial marginalization, the vMQP yields a simple, maximum-entropy density. Because the vMQP is not mathematically consistent under marginalization, it operates within a transductive learning framework where test locations are fixed during training, which motivates a fully Bayesian approach to parameter learning. To handle posterior inference efficiently, the model introduces a Stratonovich-like augmentation that linearizes the trigonometric dependencies, allowing for fast Markov Chain Monte Carlo (MCMC) sampling. We utilize the original Bayesian implementation of vMQP\footnote{https://github.com/Yarden231/vMQP} for comparison in the wind direction prediction task, where it is observed to incur the maximum training time. We thus omit it for simulation experiments and object pose estimation task.

\end{enumerate}

\section{Evaluation Metrics} \label{appendix:metrics}
In this section, we describe the metrics tailored for circular data used in model evaluation. Let $n$ denote the total number of observations, $\theta_i \in [0, 2\pi)$ be the ground-truth circular response for the $i$-th observation, and $\widehat{\theta}_i$ be the corresponding predicted mean direction. Let $\{\widehat{\theta}_{i,m}\}_{m=1}^M$ represent an ensemble of $M$ predicted angles for the $i$-th observation generated from the trained model. The mathematical formulations for the metrics are given as follows.

\begin{enumerate}
    \item \textbf{Mean Absolute Angular Deviation (MAAD):} MAAD computes the average magnitude of the shortest angular path separating the predictions from the ground truth. To handle phase wrapping at the $[0, 2\pi)$ boundary, the shortest angular distance $\mathrm{d}_a(\cdot, \cdot)$ is computed using the two-argument arctangent function:
    \begin{align} \label{eq:shortest_distance}
        \mathrm{d}_a\left(\theta_i, \widehat{\theta}_i\right) = \left| \mathrm{atan2}\left(\sin\left(\theta_i - \widehat{\theta}_i\right), \cos\left(\theta_i - \widehat{\theta}_i)\right) \right)\right|.
    \end{align}
    The MAAD is then obtained by averaging this absolute wrapped difference over all samples \citep{prokudin2018deep}:
    $$\text{MAAD} = \frac{1}{n} \sum_{i=1}^n \mathrm{d}_a\left(\theta_i, \widehat{\theta}_i\right).$$
    We report this metric in degrees for enhanced interpretability.

    \item \textbf{Circular Mean Directional Error (CMDE):} CMDE evaluates the alignment between the predicted and true angles by leveraging the cosine of the angular difference. This metric avoids the artificial discontinuities associated with linear distance metrics evaluated on a circle. It is bounded between $0$ (indicating perfect alignment) and $2$ (indicating perfectly opposite directions). The formulation is given by:
    $$\text{CMDE} = 1 - \frac{1}{n} \sum_{i=1}^n \cos\left(\theta_i - \widehat{\theta}_i\right).$$

    \item \textbf{Circular Continuous Ranked Probability Score (CRPS):} To evaluate the sharpness of the predictive distributions (or ensembles), we adapt the CRPS for circular topologies. Using the shortest angular distance function \eqref{eq:shortest_distance}, the empirical circular CRPS for a single observation $i$ approximated by an ensemble of size $M$ is defined as \citep{gneiting2007strictly}:
    $$\text{CRPS}_i = \frac{1}{M} \sum_{m=1}^M \mathrm{d}_a\left(\widehat{\theta}_{i,m}, \theta_i\right) - \frac{1}{2M(M-1)} \sum_{m=1}^M \sum_{m'=1, m'\neq m}^M \mathrm{d}_a\left(\widehat{\theta}_{i,m}, \widehat{\theta}_{i,m'}\right).$$
    The first term quantifies the average angular distance between the ensemble members and the true observation (rewarding accuracy), while the second term quantifies the expected angular distance among the ensemble members themselves. The overall Mean Circular CRPS is simply the average across all observations: $$\text{Mean Circular CRPS} = \frac{1}{n} \sum_{i=1}^n \text{CRPS}_i.$$

    \item \textbf{Median Error (MedErr):} The MedErr metric \citep{prokudin2018deep} provides a robust measure of central tendency for the prediction errors, mitigating the influence of extreme outliers. It is defined as the median of the shortest angular distances across all $n$ samples. For reporting purposes, the final value is converted from radians to degrees: $$\text{MedErr} = \text{median}_{i \in \{1, \dots, N\}} \left(\mathrm{d}_a\left(\theta_i, \widehat{\theta}_i\right)\right),$$
    where $\mathrm{d}_a$ denotes the shortest angular distance given in \eqref{eq:shortest_distance}.

    \item $\text{\textbf{Accuracy}}_{\mathbf{\widetilde{\theta}}}$: The `Threshold Accuracy' metric quantifies the proportion of predictions that fall within an acceptable margin of error, defined by a predetermined threshold $\widetilde{\theta}$. It is formulated as the mean of an indicator function over the dataset:
    $$\text{Acc}_{\widetilde{\theta}} = \frac{1}{N} \sum_{i=1}^{N} \mathbb{1}\left(\mathrm{d}_a\left(\theta_i, \widehat{\theta}_i\right) \leq \widetilde{\theta}\right),$$
    where $\mathbb{1}(\cdot)$ represents the indicator function, evaluating to 1 if the condition $\mathrm{d}_a \leq \widetilde{\theta}$ is satisfied, and 0 otherwise. In our standard evaluation protocol, we set the threshold to $\widetilde{\theta} = \pi/6$ (30 degrees), following prior works \citep{prokudin2018deep, tulsiani2015viewpoints}.
\end{enumerate}

\section{Software Demonstration}
\label{sec:software_demonstration}
We provide a highly modular open-source software implementation for ANGLE utilizing a PyTorch backend in Python. The source code is publicly available\footnote{\url{https://github.com/anglepy}} and we offer the software through our Python package \texttt{anglepy}\footnote{Documentation: \url{https://anglepy.readthedocs.io}}. The primary entry point for model training is the \texttt{ANGLE} wrapper function, which instantiates and trains a \texttt{CircularEngressor} object. As mentioned in \Cref{sec:ablation_study}, our implementation provides extensive flexibility regarding the network architecture, the projection of the output space, the formulation of the GCES loss, and the underlying noise injection mechanism. The input data is expected to be array-like in the form of PyTorch tensors, and training and inference can be performed in batches on both GPU and CPU.

\subsection{Architectural Configurations and Sufficient Dimension Reduction}
The model supports standard fully connected dense networks, which can be augmented with batch normalization (\texttt{add\_bn=True}) and residual skip connections (\texttt{resblock=True}). Furthermore, the software provides native support for sufficient dimension reduction (SDR) within the neural network framework. Setting \texttt{sdr=True} restricts the first layer to act as a linear projection. Setting \texttt{reduced\_dim=1} trains a single-index model, whereas \texttt{reduced\_dim=b} trains a multiple-index model.  Upon convergence, the fitted projection matrix $\beta \in \mathbb{R}^{d \times b}$ can be extracted via the \texttt{beta\_proj} attribute. This allows the user to explicitly compute the SDR subspace $\beta^\top X$ for downstream analysis.

\subsection{Output Heads}
To map unconstrained neural network outputs to the unit circle $\mathbb{S}^1 \cong [0, 2\pi)$, the software accommodates three distinct output modes, governed by the \texttt{unbounded} and \texttt{circular\_projection} arguments:
\begin{itemize}
    \item \textbf{Biternion head (\texttt{unbounded=False, circular\_projection='atan2'}):} The network's final layer outputs a two-dimensional vector $(u, v) \in \mathbb{R}^2$. This vector is then transformed to an angular coordinate in $[0, 2\pi)$ using the two-argument arctangent function, $\mathrm{atan2}(v, u)$.
    \item \textbf{Scaled sigmoid (\texttt{unbounded=False, circular\_projection='sigmoid'}):} The network applies a scaled sigmoid activation at the final layer to bound the outputs strictly within $[0, 2\pi)$.
    \item \textbf{Unbounded modulo (\texttt{unbounded=True}):} The network predicts an unconstrained scalar value $y \in \mathbb{R}$. Before loss computation and during inference, the output is mapped to the circle via the modulo operation $y \pmod{2\pi}$.
\end{itemize}

\subsection{Distance Metrics in GCES Loss}
The core optimization objective relies on strictly proper scoring rules, specifically the GCES, which requires the specification of a distance metric on the circle. The \texttt{dist\_method} argument allows toggling between two fundamental distances. \texttt{dist\_method='chordal'} utilizes the Euclidean distance between points embedded on the unit circle. \texttt{dist\_method='geodesic'} computes the shortest path along the manifold. When using the geodesic distance, the user can supply a custom (negative) kernel function via \texttt{kernel\_func}. The \texttt{kernels.py} module implements the strictly positive definite kernels on the circle proposed by \citet[Table 1]{gneiting2013strictly}.

\subsection{Noise Injection}
As a generative model, ANGLE relies on random noise to model conditional distributions. The framework allows precise control over this stochasticity. The base distribution of the noise is set via \texttt{noise\_dist} (accepting \texttt{'gaussian'} or \texttt{'uniform'}), with its scale modulated by the \texttt{noise\_std} parameter. The \texttt{noise\_all\_layer} boolean dictates the structural causal assumption. Setting \texttt{noise\_all\_layer=True} implements a pre-ANM framework by injecting noise into the early layers or all layers. Conversely, setting it to \texttt{False} strictly injects noise only at the final hidden layer, yielding a post-ANM formulation.

\subsection{Usage Example: Training and Inference}
The following Python code snippet demonstrates how to instantiate a realistic testing scenario, train the model using a subset of the aforementioned features (e.g., SDR, biternion projection, and chordal distance), and perform both point prediction and conditional distribution estimation via sampling.
\clearpage
\begin{lstlisting}[language=Python, caption={Demonstration of training and inference using the \texttt{ANGLE} API.}, label={lst:engression_code}]
import torch
from anglepy import ANGLE

# Assume X_train (N x d), y_train (N x 1), X_test (M x d) are torch.Tensors
device = torch.device("cuda" if torch.cuda.is_available() else "cpu")

# Train the ANGLE Model
model = ANGLE(
    x=X_train, 
    y=y_train, 
    num_layer=3, 
    hidden_dim=100, 
    noise_dim=64,
    noise_std=1.0,
    noise_dist='gaussian',
    noise_all_layer=True,          # Pre-ANM noise injection
    add_bn=True, 
    resblock=False, 
    sdr=True,                      # Enable SDR
    reduced_dim=1,                 # Single-index model
    unbounded=False,
    circular_projection='atan2',   # Biternion output head
    dist_method='chordal',         # Chordal distance for GCES loss
    gamma=1.0,                     # Energy score power parameter
    lr=0.05, 
    num_epochs=500, 
    standardize=True, 
    device=device,
    verbose=False
)

# Extract the SDR projection matrix (d x 1)
beta_matrix = model.beta_proj

# Inference: Point Predictions and Probabilistic Sampling
# Point estimates (e.g., conditional circular mean, median, or quantiles from 100 samples)
y_pred = model.predict(X_test, target='mean', sample_size=100)

# Draw 100 conditional samples per test observation
# Returns tensor of shape (M, 100) capturing the conditional distribution
y_pred_ensemble = model.sample(X_test, sample_size=100)
\end{lstlisting}

\section{Literature Review on Estimation of Circular Conditional Functionals} \label{appendix:review_point_prediction}
\subsection{Conditional Mean Function}
The mean regression task is the most explored in the literature for circular response. One of the earliest approaches involve modeling the circular response using a parametric link function \citep{fisher1992regression}, where the conditional distribution of the circular response is assumed to be von Mises distribution and the conditional mean function is related to the covariate through a parametric link. A relatively more popular approach was proposed by \cite{presnell1998projected}, which treats the circular response as a projection of unobserved response from a multivariate linear model onto $\mathbb{S}^1$. Different parametrization and estimation of such projected normal distributions can be found in \cite{hernandez2017general, zou2025dirichlet} among others. In a parallel approach, \cite{lund2002tree} devised the regression function through a tree-based approach, and it was utilized by \cite{lang2020circular} for wind direction forecasting. Another very common method of circular regression was proposed by \cite{kato2008circular}, where the regression curve is expressed as a form of M\"{o}bius circle transformation. Later, the approach was further extended to various circular data applications by works like \cite{jha2018circular,biswas2025semi}. Recently, \cite{fernandez2024regression} utilized a non-negative trigonometric sums (NNTS) model for circular data regression where the idea is to map the circular response to a vector of its first few trigonometric moments, and then exploit the hypersphere geometry of this space to convert the problem to a linear regression problem on a real line. Several other Bayesian approaches can be found in \cite{nur2024bayesian,ye2026bayesian}, and some machine learning-based modeling of the mean regression can be found in \cite{laha2022angular,bruns2024single}. A comprehensive review of parametric models for circular variables can be found in \cite{kim2016regressions,mohammad2021review,jha2022regression}. Nonparametric smoothing based regression approaches have also gained much popularity in estimating the conditional mean function in the mean time. In \cite{di2013non}, the regression function is defined as the minimizer of the expected cosine distance between the observed response and the fitted value. Under this formulation, the regression function can be expressed as the arctangent of the ratio of two conditional trigonometric moments, which are subsequently estimated through local averaging of the sine and cosine transforms of the observed responses. Subsequently, \cite{di2014nonparametric} extended this framework using local polynomial fitting techniques for spherical data. More generally, a common nonparametric strategy in the circular data literature is to model the conditional sine and cosine components separately via smoothing procedures and then combine the resulting estimates to recover the conditional direction; see, for example, \cite{nguyen2023adaptive,meilan2024nonparametric,woolsey2025nonparametric}. Several general additive and spline based models are also available in the literature, see \cite{mcmillan2013two,hassanzadeh2021smoothing} among others. In more recent contributions, \cite{gottard2026quasi} proposed modeling the circular response through the quasi-likelihood score instead of assuming the full conditional distribution, and \cite{francisco2026analyzing} utilized the Nadaraya Watson local constant estimator for a setup with circular response and a mix of continuous and categorical predictors. A comprehensive review of nonparametric regression for circular responses can be found in \cite{alonso2025review}.

\subsection{Conditional Mode and Concentration Functions} Most proposed circular regression models are primarily concerned with modeling the conditional mean direction, with \cite{fisher1992regression} being a notable exception in additionally allowing the conditional dispersion to depend on the covariates. Likewise, comparatively little attention has been devoted to estimating the conditional modal direction, despite the fact that the mode often provides a more meaningful notion of location in multimodal or highly diffuse settings, where the mean direction may correspond to a direction that is rarely observed in practice. Applications where the estimation of the conditional mode and concentration are particularly relevant include modeling the flight orientation of nocturnal migratory songbirds as a function of wind direction and flight altitude \citep{sjoberg2015nocturnal}, modeling wind direction at noon using early morning measurements \citep{kato2010family}, modeling the onset date of acute primary angle-closure glaucoma as a function of patient age \citep{gao2006application}, and studying the timing of tropical tree growth relative to annual climatic cycles \citep{hogan2019drought}. Motivated by such applications, \cite{alonso2023analyzing,perez2026nonparametric} developed nonparametric estimators for the conditional modal direction. More recently, \cite{ameijeiras2026semiparametric} proposed semiparametric modeling of the conditional concentration and modal direction through a flexible unimodal family of circular densities introduced by \cite{ameijeiras2022family}.

\subsection{CDF, Quantiles, and Density Functions} Estimating cumulative distribution functions and quantiles for circular data presents several challenges that do not arise in the Euclidean setting. The fundamental difficulty stems from the periodic geometry of the circle: unlike the classical CDF on the real line, a circular CDF requires an arbitrary choice of origin or lower integration limit, and this arbitrariness propagates to downstream tasks such as quantile estimation. Consequently, circular quantiles must be defined in a manner that is invariant to the choice of origin and applicable across general circular distributions. Early contributions in this area were made by \cite{di2012smooth}, who introduced nonparametric estimation of the circular CDF and proposed quantile estimation via inversion of the estimated CDF and through a Parzen-type smoothing approach. Later, \cite{di2016nonparametric} proposed circular quantile regression using either inversion of a double-kernel CDF estimator or minimization of a smooth circular check loss analogous to the pinball loss for Euclidean quantile regression \citep{steinwart2011estimating}. In a related work, \cite{di2016note} studied nonparametric estimation of the conditional circular density. More recently, \cite{klar2024nonparametric} developed estimators of circular distribution and density functions using Fejér polynomials and additionally established validity under error-in-covariate settings. Alternative notions of circular quantiles have also been proposed through angular depth \citep{ley2014new,hauch2024quantiles} and optimal transport formulations \citep{hallin2024nonparametric,bercu2024regularized}.

\newpage
\section{Tables and Figures} \label{appendix:tables_figures}

{\footnotesize 
\setlength{\tabcolsep}{6pt} 
\begin{longtable}[]{@{}ccccc@{}}
\caption{{\small Parameter configurations for the synthetic circular data generation mechanisms. The table details the number of linear ($p_L$) and circular ($p_C$) covariates, along with their corresponding true regression coefficient vectors ($\beta$), across all 16 experimental settings.}}\label{table:simulation_parameter_settings}\tabularnewline
\toprule 
Approach & Setting & \makecell{\# Linear\\covariates} & \makecell{\# Circular\\covariates} & Parameters \\
\midrule 
\endhead
\makecell{1: Arctangent link,\\linear relation} & 1.1 & 2 & 0 & \makecell{$\beta_1 = [-5.3, -7.9],$\\
            $\beta_2 = [-2.8, 4.1]$} \\ 
        ~ & 1.2 & 0 & 2 & \makecell{$\beta_1 = [13.5, -6.6, 8.1, 6.7],$\\
            $\beta_2 = [-5.2, 7.4, 6.9, -5.5]$} \\ 
        ~ & 1.3 & 2 & 2 & \makecell{$\beta_1 = [7.3, 6.6, 2.5, 1.1, 10.3, 5.9],$\\
            $\beta_2 = [6.7, -8.1, -3.4, 2.7, -6.5, 5.7]$} \\ 
        ~ & 1.4 & 75 & 25 & \makecell{$\beta_1, \beta_2 \in \mathbb{R}^{125}$ \\
          $\beta_{1,j}, \beta_{2,j} \sim \mathcal{U}(-10, 10)$} \\  \midrule
        \makecell{2: Arctangent link,\\nonlinear relation} & 2.1 & 2 & 0 & \makecell{$\beta_1 = [-5.3, -7.9],$\\
            $\beta_2 = [-2.8, 4.1]$} \\ 
        ~ & 2.2 & 0 & 2 & \makecell{$\beta_1 = [13.5, -6.6, 8.1, 6.7],$\\
            $\beta_2 = [-5.2, 7.4, 6.9, -5.5]$} \\ 
        ~ & 2.3 & 2 & 2 & \makecell{$\beta_1 = [7.3, 6.6, 2.5, 1.1, 10.3, 5.9],$\\
            $\beta_2 = [6.7, -8.1, -3.4, 2.7, -6.5, 5.7]$} \\ 
        ~ & 2.4 & 75 & 25 & \makecell{$\beta_1, \beta_2 \in \mathbb{R}^{125}$ \\
          $\beta_{1,j}, \beta_{2,j} \sim \mathcal{U}(-10, 10)$} \\ \midrule
        \makecell{3: Projected normal distribution,\\linear relation} & 3.1 & 2 & 0 & \makecell{$\beta_1 = [8.3, 4.6],$\\
            $\beta_2 = [6.5, -5.1]$} \\  
        ~ & 3.2 & 0 & 2 & \makecell{$\beta_1 = [13.5, -8.6, 12.1, 6.9],$\\
            $\beta_2 = [9.2, -8.1, -5.3, 2.2]$} \\ 
        ~ & 3.3 & 2 & 2 & \makecell{$\beta_1 = [8.3, 4.6, 2.5, 1.1, 12.3, 5.2],$\\
            $\beta_2 = [6.5, -5.1, -3.4, 2.7, -6.5, 4.7]$} \\
        ~ & 3.4 & 75 & 25 & \makecell{$\beta_1, \beta_2 \in \mathbb{R}^{125}$ \\
          $\beta_{1,j}, \beta_{2,j} \sim \mathcal{U}(-10, 10)$} \\ \midrule
        \makecell{4: Projected normal distribution,\\nonlinear relation} & 4.1 & 2 & 0 & \makecell{$\beta_1 = [8.3, 4.6],$\\
            $\beta_2 = [6.5, -5.1]$} \\  
        ~ & 4.2 & 0 & 2 & \makecell{$\beta_1 = [13.5, -8.6, 12.1, 6.9],$\\
            $\beta_2 = [9.2, -8.1, -5.3, 2.2]$} \\ 
        ~ & 4.3 & 2 & 2 & \makecell{$\beta_1 = [8.3, 4.6, 2.5, 1.1, 12.3, 5.2],$\\
            $\beta_2 = [6.5, -5.1, -3.4, 2.7, -6.5, 4.7]$} \\
        ~ & 4.4 & 75 & 25 & \makecell{$\beta_1, \beta_2 \in \mathbb{R}^{125}$ \\
          $\beta_{1,j}, \beta_{2,j} \sim \mathcal{U}(-10, 10)$} \\
\bottomrule\noalign{}
\endlastfoot
\end{longtable}
}

{\scriptsize 
\renewcommand{\arraystretch}{1.85}
\setlength{\LTleft}{0pt}
\setlength{\LTright}{0pt}
\setlength{\LTcapwidth}{\textwidth}
\setlength{\tabcolsep}{1.6pt} 
\begin{longtable}[]{@{}lcccccccccccc@{}}
\caption{{\small Model performances for prediction on simulated data: mean and standard deviation of three circular metrics across all 16 data generation settings. The data generation mechanisms are summarized in \Cref{table:simulation_parameter_settings}. The \underline{\textbf{best}} and \textbf{second best} results are highlighted. The mean and standard deviation obtained across 50 independent training-inference runs are reported. MAAD and CRPS are reported in degrees.}}\label{table:simulation_results}\tabularnewline
\toprule\noalign{}
 \multirow{1}{*}{Model} & MAAD & CMDE & CRPS & MAAD & CMDE & CRPS & MAAD & CMDE & CRPS & MAAD & CMDE & CRPS \\
\midrule
\endfirsthead

\multicolumn{13}{c}{
\tablename\ \thetable\ -- continued from previous page
}\\
\toprule
 \multirow{1}{*}{Model} & MAAD & CMDE & CRPS & MAAD & CMDE & CRPS & MAAD & CMDE & CRPS & MAAD & CMDE & CRPS \\
\midrule
\endhead

\midrule
\multicolumn{13}{r}{Continued on next page}\\
\endfoot

\bottomrule
\endlastfoot
        ~ & \multicolumn{3}{c}{Setting 1.1} & \multicolumn{3}{c}{Setting 1.2} & \multicolumn{3}{c}{Setting 1.3} & \multicolumn{3}{c}{Setting 1.4} \\
        \cmidrule(lr){2-4} \cmidrule(lr){5-7} \cmidrule(lr){8-10} \cmidrule(lr){11-13} 
        ANGLE-Geodesic & $\bms{3.370}{0.507}$ & $\bms{0.007}{0.003}$ & $\bms{3.152}{0.655}$ & $\bms{4.253}{0.731}$ & $\ubms{0.010}{0.003}$ & $\bms{3.609}{0.867}$ & $\md{3.511}{0.770}$ & $\md{0.007}{0.004}$ & $\md{3.172}{0.829}$ & $\md{4.345}{0.540}$ & $\ubms{0.012}{0.004}$ & $\ubms{3.187}{0.422}$ \\ 
        ANGLE-Chordal & $\ubms{3.221}{0.318}$ & $\bms{0.007}{0.002}$ & $\ubms{2.725}{0.399}$ & $\ubms{3.939}{0.377}$ & $\ubms{0.010}{0.001}$ & $\ubms{2.996}{0.382}$ & $\bms{3.194}{0.320}$ & $\bms{0.005}{0.002}$ & $\ubms{2.652}{0.315}$ & $\bms{4.308}{0.525}$ & $\ubms{0.012}{0.004}$ & $\bms{3.238}{0.438}$ \\ 
        CLR & $\md{3.634}{0.020}$ & $\ubms{0.006}{0.000}$ & $\md{3.634}{0.020}$ & $\md{5.044}{0.032}$ & $\bms{0.017}{0.000}$ & $\md{5.044}{0.032}$ & $\ubms{2.948}{0.001}$ & $\ubms{0.004}{0.000}$ & $\bms{2.948}{0.001}$ & $\ubms{3.813}{0.003}$ & $\bms{0.013}{0.000}$ & $\md{3.813}{0.003}$ \\ 
        Kernel CLR & $\md{5.321}{0.000}$ & $\md{0.010}{0.000}$ & $\md{5.321}{0.000}$ & $\md{31.444}{0.000}$ & $\md{0.244}{0.000}$ & $\md{31.444}{0.000}$ & $\md{19.400}{0.000}$ & $\md{0.123}{0.000}$ & $\md{19.400}{0.000}$ & $\md{64.079}{0.000}$ & $\md{0.643}{0.000}$ & $\md{64.079}{0.000}$ \\ 
        SCB-NLL & $\md{33.028}{0.069}$ & $\md{0.230}{0.000}$ & $\md{33.028}{0.069}$ & $\md{37.151}{1.123}$ & $\md{0.277}{0.014}$ & $\md{37.151}{1.123}$ & $\md{46.040}{0.687}$ & $\md{0.373}{0.011}$ & $\md{46.040}{0.687}$ & $\md{44.505}{0.647}$ & $\md{0.363}{0.009}$ & $\md{44.505}{0.647}$ \\ 
        SCB-MCPE & $\md{33.352}{0.686}$ & $\md{0.238}{0.008}$ & $\md{33.352}{0.686}$ & $\md{34.360}{1.237}$ & $\md{0.244}{0.010}$ & $\md{34.360}{1.237}$ & $\md{46.441}{0.540}$ & $\md{0.378}{0.005}$ & $\md{46.441}{0.540}$ & $\md{46.441}{0.540}$ & $\md{0.378}{0.005}$ & $\md{46.441}{0.540}$ \\ 
        Lifted EP-MGvM & $\md{6.781}{0.000}$ & $\md{0.012}{0.000}$ & $\md{6.781}{0.000}$ & $\md{25.661}{0.000}$ & $\md{0.175}{0.000}$ & $\md{25.661}{0.000}$ & $\md{21.666}{0.000}$ & $\md{0.126}{0.000}$ & $\md{21.666}{0.000}$ & $\md{67.727}{0.000}$ & $\md{0.696}{0.000}$ & $\md{67.727}{0.000}$ \\ 
        MCR-II & $\md{41.277}{0.000}$ & $\md{0.352}{0.000}$ & $\md{41.277}{0.000}$ & $\md{30.682}{0.000}$ & $\md{0.209}{0.000}$ & $\md{30.682}{0.000}$ & $\md{46.074}{1.152}$ & $\md{0.419}{0.013}$ & $\md{46.074}{1.152}$ & $\md{43.826}{2.180}$ & $\md{0.389}{0.027}$ & $\md{43.826}{2.180}$ \\ \midrule
        ~ & \multicolumn{3}{c}{Setting 2.1} & \multicolumn{3}{c}{Setting 2.2} & \multicolumn{3}{c}{Setting 2.3} & \multicolumn{3}{c}{Setting 2.4} \\ 
    \cmidrule(lr){2-4} \cmidrule(lr){5-7} \cmidrule(lr){8-10} \cmidrule(lr){11-13}        
        ANGLE-Geodesic & $\bms{8.586}{3.017}$ & $\bms{0.032}{0.025}$ & $\bms{7.128}{1.942}$ & $\bms{8.104}{1.184}$ & $\bms{0.038}{0.008}$ & $\bms{6.342}{0.876}$ & $\bms{8.194}{1.392}$ & $\bms{0.047}{0.014}$ & $\bms{6.214}{0.886}$ & $\md{4.647}{0.587}$ & $\bms{0.015}{0.005}$ & $\ubms{3.424}{0.464}$ \\ 
        ANGLE-Chordal & $\ubms{5.718}{1.060}$ & $\ubms{0.015}{0.007}$ & $\ubms{4.607}{0.783}$ & $\ubms{6.703}{0.784}$ & $\ubms{0.029}{0.006}$ & $\ubms{5.111}{0.597}$ & $\ubms{6.411}{1.428}$ & $\ubms{0.030}{0.014}$ & $\ubms{4.775}{0.855}$ & $\bms{4.578}{0.537}$ & $\ubms{0.014}{0.004}$ & $\bms{3.478}{0.462}$ \\ 
        CLR & $\md{19.281}{0.000}$ & $\md{0.144}{0.000}$ & $\md{19.281}{0.000}$ & $\md{10.769}{0.009}$ & $\md{0.073}{0.000}$ & $\md{10.769}{0.009}$ & $\md{10.422}{0.000}$ & $\md{0.072}{0.000}$ & $\md{10.422}{0.000}$ & $\ubms{3.865}{0.002}$ & $\bms{0.015}{0.000}$ & $\md{3.865}{0.002}$ \\ 
        Kernel CLR & $\md{10.617}{0.000}$ & $\md{0.045}{0.000}$ & $\md{10.617}{0.000}$ & $\md{31.988}{0.000}$ & $\md{0.267}{0.000}$ & $\md{31.988}{0.000}$ & $\md{19.437}{0.000}$ & $\md{0.128}{0.000}$ & $\md{19.437}{0.000}$ & $\md{64.089}{0.000}$ & $\md{0.646}{0.000}$ & $\md{64.089}{0.000}$ \\ 
        SCB-NLL & $\md{28.335}{0.179}$ & $\md{0.189}{0.002}$ & $\md{28.335}{0.179}$ & $\md{32.441}{1.343}$ & $\md{0.219}{0.012}$ & $\md{32.441}{1.343}$ & $\md{33.611}{0.477}$ & $\md{0.235}{0.005}$ & $\md{33.611}{0.477}$ & $\md{37.087}{1.055}$ & $\md{0.273}{0.012}$ & $\md{37.087}{1.055}$ \\ 
        SCB-MCPE & $\md{27.815}{3.268}$ & $\md{0.186}{0.041}$ & $\md{27.815}{3.268}$ & $\md{29.781}{0.220}$ & $\md{0.197}{0.003}$ & $\md{29.781}{0.220}$ & $\md{33.259}{3.232}$ & $\md{0.229}{0.040}$ & $\md{33.259}{3.232}$ & $\md{41.401}{3.365}$ & $\md{0.325}{0.045}$ & $\md{41.401}{3.365}$ \\ 
        Lifted EP-MGvM & $\md{16.934}{0.000}$ & $\md{0.115}{0.000}$ & $\md{16.934}{0.000}$ & $\md{27.560}{0.000}$ & $\md{0.216}{0.000}$ & $\md{27.560}{0.000}$ & $\md{24.392}{0.000}$ & $\md{0.172}{0.000}$ & $\md{24.392}{0.000}$ & $\md{67.626}{0.000}$ & $\md{0.697}{0.000}$ & $\md{67.626}{0.000}$ \\ 
        MCR-II & $\md{41.435}{0.000}$ & $\md{0.346}{0.000}$ & $\md{41.435}{0.000}$ & $\md{28.118}{0.002}$ & $\md{0.181}{0.000}$ & $\md{28.118}{0.002}$ & $\md{41.959}{0.244}$ & $\md{0.357}{0.003}$ & $\md{41.959}{0.244}$ & $\md{43.495}{0.828}$ & $\md{0.383}{0.011}$ & $\md{43.495}{0.828}$ \\ \midrule
        ~ & \multicolumn{3}{c}{Setting 3.1} & \multicolumn{3}{c}{Setting 3.2} & \multicolumn{3}{c}{Setting 3.3} & \multicolumn{3}{c}{Setting 3.4} \\ 
   \cmidrule(lr){2-4} \cmidrule(lr){5-7} \cmidrule(lr){8-10} \cmidrule(lr){11-13}          
        ANGLE-Geodesic & $\bms{7.247}{0.340}$ & $\bms{0.025}{0.002}$ & $\bms{5.296}{0.374}$ & $\bms{5.825}{0.319}$ & $\ubms{0.013}{0.001}$ & $\bms{4.386}{0.427}$ & $\md{5.602}{0.274}$ & $\bms{0.016}{0.001}$ & $\bms{4.203}{0.339}$ & $\md{3.770}{0.366}$ & $\bms{0.007}{0.003}$ & $\ubms{2.720}{0.236}$ \\ 
        ANGLE-Chordal & $\ubms{7.175}{0.300}$ & $\ubms{0.024}{0.002}$ & $\ubms{5.093}{0.273}$ & $\ubms{5.722}{0.234}$ & $\ubms{0.013}{0.001}$ & $\ubms{4.184}{0.232}$ & $\bms{5.514}{0.295}$ & $\bms{0.016}{0.001}$ & $\ubms{3.956}{0.279}$ & $\bms{3.688}{0.344}$ & $\bms{0.007}{0.003}$ & $\bms{2.726}{0.265}$ \\ 
        CLR & $\md{7.306}{0.002}$ & $\md{0.026}{0.000}$ & $\md{7.306}{0.002}$ & $\md{6.686}{0.022}$ & $\bms{0.017}{0.000}$ & $\md{6.686}{0.022}$ & $\ubms{5.341}{0.002}$ & $\ubms{0.015}{0.000}$ & $\md{5.341}{0.002}$ & $\ubms{3.381}{0.002}$ & $\ubms{0.006}{0.000}$ & $\md{3.381}{0.002}$ \\ 
        Kernel CLR & $\md{7.345}{0.000}$ & $\ubms{0.024}{0.000}$ & $\md{7.345}{0.000}$ & $\md{25.815}{0.000}$ & $\md{0.172}{0.000}$ & $\md{25.815}{0.000}$ & $\md{20.692}{0.000}$ & $\md{0.130}{0.000}$ & $\md{20.692}{0.000}$ & $\md{68.379}{0.000}$ & $\md{0.709}{0.000}$ & $\md{68.379}{0.000}$ \\ 
        SCB-NLL & $\md{36.185}{0.440}$ & $\md{0.266}{0.006}$ & $\md{36.185}{0.440}$ & $\md{30.673}{0.143}$ & $\md{0.212}{0.001}$ & $\md{30.673}{0.143}$ & $\md{44.309}{2.560}$ & $\md{0.360}{0.028}$ & $\md{44.309}{2.560}$ & $\md{50.878}{0.782}$ & $\md{0.445}{0.012}$ & $\md{50.878}{0.782}$ \\ 
        SCB-MCPE & $\md{36.183}{1.476}$ & $\md{0.270}{0.018}$ & $\md{36.183}{1.476}$ & $\md{28.712}{0.048}$ & $\md{0.193}{0.000}$ & $\md{28.712}{0.048}$ & $\md{41.536}{0.966}$ & $\md{0.325}{0.013}$ & $\md{41.536}{0.966}$ & $\md{52.842}{0.977}$ & $\md{0.477}{0.014}$ & $\md{52.842}{0.977}$ \\ 
        Lifted EP-MGvM & $\md{8.228}{0.000}$ & $\md{0.027}{0.000}$ & $\md{8.228}{0.000}$ & $\md{24.407}{0.000}$ & $\md{0.130}{0.000}$ & $\md{24.407}{0.000}$ & $\md{21.155}{0.000}$ & $\md{0.116}{0.000}$ & $\md{21.155}{0.000}$ & $\md{85.440}{0.000}$ & $\md{0.935}{0.000}$ & $\md{85.440}{0.000}$ \\ 
        MCR-II & $\md{41.303}{0.772}$ & $\md{0.347}{0.010}$ & $\md{41.303}{0.772}$ & $\md{56.908}{13.448}$ & $\md{0.559}{0.185}$ & $\md{56.908}{13.448}$ & $\md{41.422}{0.549}$ & $\md{0.350}{0.005}$ & $\md{41.422}{0.549}$ & $\md{44.769}{1.672}$ & $\md{0.397}{0.022}$ & $\md{44.769}{1.672}$ \\ \midrule
        ~ & \multicolumn{3}{c}{Setting 4.1} & \multicolumn{3}{c}{Setting 4.2} & \multicolumn{3}{c}{Setting 4.3} & \multicolumn{3}{c}{Setting 4.4} \\ 
    \cmidrule(lr){2-4} \cmidrule(lr){5-7} \cmidrule(lr){8-10} \cmidrule(lr){11-13}          
        ANGLE-Geodesic & $\bms{8.017}{0.478}$ & $\bms{0.034}{0.004}$ & $\bms{6.194}{0.513}$ & $\bms{7.801}{0.374}$ & $\ubms{0.033}{0.002}$ & $\bms{5.794}{0.369}$ & $\bms{6.004}{0.715}$ & $\bms{0.014}{0.005}$ & $\bms{4.951}{0.529}$ & $\md{3.750}{0.357}$ & $\bms{0.007}{0.003}$ & $\ubms{2.696}{0.238}$ \\ 
        ANGLE-Chordal & $\ubms{7.775}{0.527}$ & $\ubms{0.033}{0.004}$ & $\ubms{5.668}{0.417}$ & $\ubms{7.576}{0.325}$ & $\ubms{0.033}{0.002}$ & $\ubms{5.478}{0.274}$ & $\ubms{5.704}{0.375}$ & $\ubms{0.012}{0.002}$ & $\ubms{4.336}{0.370}$ & $\bms{3.678}{0.357}$ & $\bms{0.007}{0.003}$ & $\bms{2.711}{0.260}$ \\ 
        CLR & $\md{8.404}{0.003}$ & $\md{0.038}{0.000}$ & $\md{8.404}{0.003}$ & $\md{9.143}{0.018}$ & $\bms{0.038}{0.000}$ & $\md{9.143}{0.018}$ & $\md{6.472}{0.000}$ & $\md{0.017}{0.000}$ & $\md{6.472}{0.000}$ & $\ubms{3.163}{0.002}$ & $\ubms{0.005}{0.000}$ & $\md{3.163}{0.002}$ \\ 
        Kernel CLR & $\md{8.356}{0.000}$ & $\md{0.039}{0.000}$ & $\md{8.356}{0.000}$ & $\md{27.293}{0.000}$ & $\md{0.216}{0.000}$ & $\md{27.293}{0.000}$ & $\md{19.893}{0.000}$ & $\md{0.123}{0.000}$ & $\md{19.893}{0.000}$ & $\md{68.752}{0.000}$ & $\md{0.714}{0.000}$ & $\md{68.752}{0.000}$ \\ 
        SCB-NLL & $\md{28.505}{0.293}$ & $\md{0.188}{0.005}$ & $\md{28.505}{0.293}$ & $\md{30.240}{0.176}$ & $\md{0.213}{0.002}$ & $\md{30.240}{0.176}$ & $\md{34.841}{0.422}$ & $\md{0.248}{0.005}$ & $\md{34.841}{0.422}$ & $\md{48.466}{0.849}$ & $\md{0.407}{0.013}$ & $\md{48.466}{0.849}$ \\ 
        SCB-MCPE & $\md{28.628}{0.306}$ & $\md{0.189}{0.004}$ & $\md{28.628}{0.306}$ & $\md{30.246}{0.547}$ & $\md{0.217}{0.005}$ & $\md{30.246}{0.547}$ & $\md{35.607}{4.107}$ & $\md{0.259}{0.048}$ & $\md{35.607}{4.107}$ & $\md{53.406}{1.074}$ & $\md{0.485}{0.015}$ & $\md{53.406}{1.074}$ \\ 
        Lifted EP-MGvM & $\md{10.016}{0.000}$ & $\md{0.048}{0.000}$ & $\md{10.016}{0.000}$ & $\md{27.938}{0.000}$ & $\md{0.175}{0.000}$ & $\md{27.938}{0.000}$ & $\md{22.360}{0.000}$ & $\md{0.132}{0.000}$ & $\md{22.360}{0.000}$ & $\md{85.366}{0.000}$ & $\md{0.935}{0.000}$ & $\md{85.366}{0.000}$ \\ 
        MCR-II & $\md{39.088}{0.000}$ & $\md{0.318}{0.000}$ & $\md{39.088}{0.000}$ & $\md{58.865}{4.375}$ & $\md{0.587}{0.061}$ & $\md{58.865}{4.375}$ & $\md{39.400}{0.440}$ & $\md{0.315}{0.007}$ & $\md{39.400}{0.440}$ & $\md{44.962}{1.724}$ & $\md{0.401}{0.023}$ & $\md{44.962}{1.724}$ \\

\bottomrule\noalign{}
\end{longtable}
}

{\scriptsize 
\setlength{\tabcolsep}{4.2pt} 
\renewcommand{\arraystretch}{1.85}
\setlength{\LTleft}{0pt}
\setlength{\LTright}{0pt}
\setlength{\LTcapwidth}{\textwidth}
\begin{longtable}{cccccccccc} 
\caption{{\small Results on the wind direction prediction task. The \underline{\textbf{best}} and \textbf{second best} results are highlighted. The mean and standard deviation obtained across 50 independent training-inference runs are reported. For the proposed model using energy score loss with geodesic distance, the $C^2$-Wendland kernel is used. MedErr, CRPS, and MAAD are reported in degrees.}}\label{table:wind_prediction_results}\\
\toprule
Dataset & Metric & vMQP & CLR & Kernel CLR & SCBR & Lifted EP-MGvM & MCR2 & ANGLE-Geodesic & ANGLE-Chordal \\
\midrule

\multirow{5}{*}{\rotatebox[origin=c]{90}{Germany Wind}} & Accuracy
& $\md{0.85}{0.01}$
& $\ubms{0.90}{0.00}$
& $\ubms{0.90}{0.00}$
& $\ubms{0.90}{0.00}$
& $\ubms{0.90}{0.02}$
& $\bms{0.89}{0.02}$
& $\ubms{0.90}{0.03}$
& $\bms{0.89}{0.02}$ \\

~ & MedErr
& $\md{12.74}{0.47}$
& $\md{12.66}{0.18}$
& $\md{16.05}{0.00}$
& $\md{13.01}{0.02}$
& $\md{13.18}{2.87}$
& $\md{13.28}{0.49}$
& $\ubms{12.52}{0.78}$
& $\bms{12.63}{0.87}$ \\

~ & CRPS
& $\md{12.61}{0.11}$
& $\md{16.06}{0.01}$
& $\md{18.12}{0.00}$
& $\md{16.54}{0.00}$
& $\md{15.59}{1.53}$
& $\md{16.95}{0.76}$
& $\ubms{11.47}{0.20}$
& $\bms{11.61}{0.14}$ \\

~ & MAAD
& $\md{16.49}{0.21}$
& $\md{16.06}{0.01}$
& $\md{18.12}{0.00}$
& $\md{16.54}{0.00}$
& $\ubms{15.59}{1.53}$
& $\md{16.95}{0.76}$
& $\bms{15.86}{0.41}$
& $\md{16.29}{0.27}$ \\

~ & CMDE
& $\ubms{0.06}{0.00}$
& $\ubms{0.06}{0.00}$
& $\bms{0.07}{0.00}$
& $\bms{0.07}{0.00}$
& $\bms{0.07}{0.01}$
& $\bms{0.07}{0.01}$
& $\ubms{0.06}{0.00}$
& $\ubms{0.06}{0.00}$ \\
\midrule

\multirow{5}{*}{\rotatebox[origin=c]{90}{India wind}}& Accuracy
& $\md{0.03}{0.00}$
& $\md{0.35}{0.00}$
& $\md{0.68}{0.00}$
& $\md{0.41}{0.03}$
& $\bms{0.77}{0.00}$
& $\md{0.33}{0.01}$
& $\md{0.72}{0.04}$
& $\ubms{0.78}{0.03}$ \\

~ & MedErr
& $\md{130.35}{0.73}$
& $\md{43.15}{0.01}$
& $\md{17.83}{0.00}$
& $\md{41.41}{1.06}$
& $\bms{13.17}{0.00}$
& $\md{41.25}{1.25}$
& $\md{15.76}{1.74}$
& $\ubms{13.04}{1.43}$ \\

~ & CRPS
& $\md{82.63}{2.25}$
& $\md{47.49}{0.00}$
& $\md{28.24}{0.00}$
& $\md{47.63}{0.46}$
& $\md{23.24}{0.00}$
& $\md{49.24}{0.40}$
& $\md{19.14}{1.95}$
& $\ubms{17.13}{1.32}$ \\

~ & MAAD
& $\md{122.83}{0.27}$
& $\md{47.49}{0.00}$
& $\md{28.24}{0.00}$
& $\md{47.63}{0.46}$
& $\bms{23.24}{0.00}$
& $\md{49.24}{0.40}$
& $\md{25.81}{2.50}$
& $\ubms{23.18}{1.60}$ \\

~ & CMDE
& $\md{1.46}{0.00}$
& $\md{0.41}{0.00}$
& $\md{0.20}{0.00}$
& $\md{0.43}{0.01}$
& $\ubms{0.15}{0.00}$
& $\md{0.43}{0.01}$
& $\md{0.18}{0.03}$
& $\ubms{0.15}{0.02}$ \\

\bottomrule
\endlastfoot
\end{longtable}}

\clearpage

\begingroup
{\tiny
\setlength{\tabcolsep}{2.6pt}
\renewcommand{\arraystretch}{1.85}
\setlength{\LTleft}{0pt}
\setlength{\LTright}{0pt}
\setlength{\LTcapwidth}{\textwidth}
\begin{longtable}{@{}llcccccccccccc|c@{}}
\caption{Model performances on the object pose detection task.
The \underline{\textbf{best}} and \textbf{second-best} results are highlighted.
The mean and standard deviation obtained across 50 independent
training--inference runs are reported. MedErr, CRPS, and MAAD are
reported in degrees.} \label{table:object_pose_results}\\

\toprule
Metric & Model & aeroplane & bicycle & boat & bottle & bus & car &
chair & table & motorbike & sofa & train & tvmonitor & Mean \\
\midrule
\endfirsthead

\multicolumn{15}{c}{
\tablename\ \thetable\ -- continued from previous page
}\\
\toprule
Metric & Model & aeroplane & bicycle & boat & bottle & bus & car &
chair & table & motorbike & sofa & train & tvmonitor & Mean \\
\midrule
\endhead

\midrule
\multicolumn{15}{r}{Continued on next page}\\
\endfoot

\bottomrule
\endlastfoot

        \multirow{12}{*}{\rotatebox[origin=c]{90}{$\text{Accuracy}_{\pi/6}$}} & Mixture-Inception & $0.28$ & $0.48$ & $0.38$ & \underline{$\mathbf{0.99}$} & $0.7$ & $0.33$ & $0.36$ & $0.65$ & $0.3$ & $0.57$ & $0.66$ & $0.69$ & $0.53$ \\
        ~ & Mixture-DenseNet & $0.53$ & $0.27$ & $0.4$ & \underline{$\mathbf{0.99}$} & $0.61$ & $0.39$ & $0.39$ & $0.63$ & $0.31$ & $0.65$ & $0.65$ & $0.64$ & $0.54$ \\
        ~ & Mixture-MobileNet & $0.16$ & $0.29$ & $0.41$ & \underline{$\mathbf{0.99}$} & $0.63$ & $0.39$ & $0.36$ & $0.63$ & $0.31$ & $0.64$ & $0.66$ & $0.67$ & $0.51$ \\
        ~ & Inception-v3-SCBR & $\ms{0.33}{0.00}$ & $\ms{0.27}{0.01}$ & $\ms{0.36}{0.01}$ & $\ubms{0.99}{0.00}$ & $\ms{0.20}{0.04}$ & $\ms{0.49}{0.00}$ & $\ms{0.39}{0.00}$ & $\ms{0.60}{0.01}$ & $\ms{0.31}{0.00}$ & $\ms{0.62}{0.00}$ & $\ms{0.55}{0.02}$ & $\ms{0.66}{0.00}$ & $0.48$ \\
        ~ & Inception-v3-LEP-mGvM & $\bms{0.68}{0.00}$ & $\bms{0.60}{0.00}$ & $\bms{0.48}{0.00}$ & $\ubms{0.99}{0.00}$ & $\ubms{0.80}{0.00}$ & $\ubms{0.68}{0.00}$ & $\ms{0.58}{0.00}$ & $\ms{0.65}{0.00}$ & $\ms{0.69}{0.00}$ & $\ms{0.73}{0.00}$ & $\ms{0.67}{0.00}$ & $\ms{0.73}{0.00}$ & $0.69$ \\
        ~ & Inception-v3-MCR2 & $\ms{0.60}{0.01}$ & $\ms{0.57}{0.01}$ & $\ms{0.44}{0.01}$ & $\ubms{0.99}{0.00}$ & $\ms{0.77}{0.01}$ & $\ms{0.65}{0.01}$ & $\ms{0.58}{0.00}$ & $\bms{0.66}{0.00}$ & $\ms{0.68}{0.01}$ & $\ms{0.74}{0.00}$ & $\ms{0.62}{0.01}$ & $\ms{0.73}{0.01}$ & $0.67$ \\
        ~ & ConvNeXt-SCBR & $\ms{0.34}{0.00}$ & $\ms{0.28}{0.00}$ & $\ms{0.40}{0.01}$ & $\ubms{0.99}{0.00}$ & $\ms{0.20}{0.05}$ & $\ms{0.50}{0.00}$ & $\ms{0.39}{0.00}$ & $\bms{0.66}{0.00}$ & $\ms{0.31}{0.00}$ & $\ms{0.63}{0.00}$ & $\ms{0.58}{0.01}$ & $\ms{0.66}{0.00}$ & $0.5$ \\
        ~ & ConvNeXt-LEP-mGvM & $\ms{0.55}{0.00}$ & $\ms{0.54}{0.00}$ & $\bms{0.48}{0.00}$ & $\ubms{0.99}{0.00}$ & $\ubms{0.80}{0.00}$ & $\ms{0.63}{0.00}$ & $\ms{0.43}{0.00}$ & $\bms{0.66}{0.00}$ & $\ms{0.53}{0.00}$ & $\ms{0.70}{0.00}$ & $\ubms{0.71}{0.00}$ & $\ms{0.72}{0.00}$ & $0.65$ \\
        ~ & ConvNeXt-MCR2 & $\ms{0.64}{0.02}$ & $\ms{0.55}{0.04}$ & $\ms{0.46}{0.01}$ & $\ubms{0.99}{0.00}$ & $\ms{0.70}{0.02}$ & $\bms{0.67}{0.01}$ & $\ms{0.58}{0.01}$ & $\ms{0.63}{0.01}$ & $\ms{0.57}{0.07}$ & $\ms{0.72}{0.01}$ & $\ms{0.59}{0.02}$ & $\bms{0.74}{0.01}$ & $0.65$ \\
        ~ & Inception-v3-SDR-ANGLE & $\ms{0.66}{0.01}$ & $\ms{0.59}{0.01}$ & $\ms{0.45}{0.01}$ & $\ubms{0.99}{0.00}$ & $\bms{0.79}{0.01}$ & $\ms{0.65}{0.00}$ & $\bms{0.59}{0.00}$ & $\ms{0.64}{0.04}$ & $\ms{0.66}{0.04}$ & $\ms{0.75}{0.01}$ & $\bms{0.70}{0.01}$ & $\ms{0.70}{0.00}$ & $0.68$ \\
        ~ & Inception-v3-ANGLE & $\ms{0.67}{0.01}$ & $\bms{0.60}{0.01}$ & $\ms{0.47}{0.01}$ & $\ubms{0.99}{0.00}$ & $\ubms{0.80}{0.01}$ & $\ms{0.66}{0.00}$ & $\ubms{0.60}{0.00}$ & $\ubms{0.67}{0.01}$ & $\ubms{0.70}{0.01}$ & $\bms{0.76}{0.00}$ & $\ubms{0.71}{0.01}$ & $\bms{0.74}{0.01}$ & $\mathbf{0.70}$ \\
        ~ & ConvNeXt-ANGLE & $\ubms{0.70}{0.01}$ & $\ubms{0.61}{0.01}$ & $\ubms{0.49}{0.01}$ & $\ubms{0.99}{0.00}$ & $\ubms{0.80}{0.00}$ & $\ubms{0.68}{0.00}$ & $\ubms{0.60}{0.00}$ & $\ubms{0.67}{0.01}$ & $\ubms{0.71}{0.01}$ & $\ubms{0.77}{0.01}$ & $\ms{0.69}{0.01}$ & $\ubms{0.76}{0.01}$ & \underline{$\mathbf{0.71}$} \\ \midrule
        \multirow{12}{*}{\rotatebox[origin=c]{90}{MedErr}} & Mixture-Inception & $69.89$ & $32.55$ & $48.24$ & $\mathbf{0.96}$ & $8.9$ & $70.64$ & $53.62$ & $8.51$ & $58.83$ & $25.38$ & $18.52$ & $16.31$ & $34.36$ \\
        ~ & Mixture-DenseNet & $27.63$ & $67.03$ & $50.35$ & \underline{$\mathbf{0.62}$} & $19.88$ & $59.3$ & $42.16$ & $\mathbf{5.9}$ & $59.84$ & $16.42$ & $17.63$ & $17.45$ & $32.02$ \\
        ~ & Mixture-MobileNet & $68.13$ & $65.1$ & $50.89$ & $1.49$ & $19.04$ & $52.01$ & $53.67$ & \underline{$\mathbf{5.59}$} & $59.92$ & $13.79$ & $17.68$ & $18.66$ & $35.50$ \\
        ~ & Inception-v3-SCBR & $\ms{51.40}{0.29}$ & $\ms{57.07}{0.43}$ & $\ms{65.84}{0.52}$ & $\ms{2.91}{0.00}$ & $\ms{54.85}{1.51}$ & $\ms{50.77}{0.21}$ & $\ms{54.97}{0.20}$ & $\ms{36.05}{0.25}$ & $\ms{52.10}{0.46}$ & $\ms{29.22}{0.12}$ & $\ms{45.45}{0.62}$ & $\ms{31.93}{0.30}$ & $44.38$ \\
        ~ & Inception-v3-LEP-mGvM & $\ms{34.80}{0.00}$ & $\ms{40.72}{0.00}$ & $\ms{52.65}{0.00}$ & $\ms{4.34}{0.00}$ & $\ms{24.99}{0.00}$ & $\ms{37.56}{0.00}$ & $\ms{39.66}{0.00}$ & $\ms{27.83}{0.00}$ & $\ms{32.95}{0.00}$ & $\ms{24.63}{0.00}$ & $\ms{43.35}{0.00}$ & $\ms{26.92}{0.00}$ & $32.53$ \\
        ~ & Inception-v3-MCR2 & $\ms{37.50}{0.88}$ & $\ms{41.99}{0.53}$ & $\ms{54.47}{0.40}$ & $\ms{3.86}{0.10}$ & $\ms{28.13}{0.46}$ & $\ms{39.33}{0.70}$ & $\ms{38.99}{0.12}$ & $\ms{29.33}{0.10}$ & $\ms{32.98}{0.42}$ & $\ms{24.91}{0.66}$ & $\ms{42.41}{0.47}$ & $\ms{28.29}{0.35}$ & $33.52$ \\
        ~ & ConvNeXt-SCBR & $\ms{48.71}{0.33}$ & $\ms{55.70}{0.38}$ & $\ms{60.01}{0.63}$ & $\ms{2.89}{0.00}$ & $\ms{56.20}{2.37}$ & $\ms{47.66}{0.41}$ & $\ms{53.50}{0.19}$ & $\ms{34.79}{0.13}$ & $\ms{50.33}{0.37}$ & $\ms{28.49}{0.16}$ & $\ms{45.15}{0.23}$ & $\ms{31.55}{0.11}$ & $42.92$ \\
        ~ & ConvNeXt-LEP-mGvM & $\ms{47.54}{0.00}$ & $\ms{50.85}{0.00}$ & $\ms{54.67}{0.00}$ & $\ms{3.39}{0.00}$ & $\ms{24.56}{0.00}$ & $\ms{46.52}{0.00}$ & $\ms{57.75}{0.00}$ & $\ms{26.07}{0.00}$ & $\ms{51.45}{0.00}$ & $\ms{24.85}{0.00}$ & $\ms{37.70}{0.00}$ & $\ms{30.15}{0.00}$ & $37.96$ \\
        ~ & ConvNeXt-MCR2 & $\ms{34.05}{1.29}$ & $\ms{43.65}{1.65}$ & $\ms{51.87}{0.25}$ & $\ms{5.63}{0.39}$ & $\ms{30.71}{1.16}$ & $\ms{35.47}{0.40}$ & $\ms{36.77}{0.51}$ & $\ms{30.84}{0.28}$ & $\ms{38.53}{3.96}$ & $\ms{26.29}{0.34}$ & $\ms{44.32}{0.83}$ & $\ms{28.67}{0.33}$ & $33.90$ \\
        ~ & Inception-v3-SDR-ANGLE & $\ms{18.59}{0.45}$ & $\ms{21.24}{0.64}$ & $\ms{36.29}{1.55}$ & $\ms{2.19}{0.31}$ & $\bms{6.95}{0.22}$ & $\ms{15.25}{0.21}$ & $\ms{23.10}{0.30}$ & $\ms{18.13}{4.24}$ & $\ms{18.02}{2.81}$ & $\ms{13.16}{1.06}$ & $\bms{11.52}{0.22}$ & $\ms{15.69}{0.39}$ & $16.68$ \\
        ~ & Inception-v3-ANGLE & $\ubms{17.39}{0.43}$ & $\ubms{19.96}{0.56}$ & $\bms{34.65}{0.89}$ & $\ms{2.46}{0.30}$ & $\bms{6.95}{0.28}$ & $\bms{14.62}{0.37}$ & $\ubms{22.09}{0.20}$ & $\ms{15.27}{0.57}$ & $\ubms{16.31}{0.49}$ & $\ubms{11.57}{0.43}$ & $\ms{12.05}{0.76}$ & $\bms{14.67}{0.28}$ & $\mathbf{15.67}$ \\
        ~ & ConvNeXt-ANGLE & $\bms{17.54}{0.47}$ & $\bms{21.03}{0.50}$ & $\ubms{30.72}{1.06}$ & $\ms{1.64}{0.19}$ & $\ubms{6.12}{0.18}$ & $\ubms{14.32}{0.33}$ & $\bms{22.80}{0.31}$ & $\ms{17.66}{0.50}$ & $\bms{17.11}{0.45}$ & $\bms{11.97}{0.36}$ & $\ubms{11.05}{0.62}$ & $\ubms{12.61}{0.32}$ & \underline{$\mathbf{15.47}$} \\ \midrule
        \multirow{12}{*}{\rotatebox[origin=c]{90}{CRPS}} & Mixture-Inception & $43.38$ & $36.25$ & $41.75$ & $2.82$ & $29.49$ & $42.96$ & $40.77$ & $22.32$ & $41.99$ & $23.35$ & $\mathbf{33.54}$ & $22.47$ & $31.76$ \\
        ~ & Mixture-DenseNet & $31$ & $44.35$ & $42.32$ & $2.56$ & $32.11$ & $60.1$ & $35.63$ & $24.58$ & $43.19$ & $22.84$ & \underline{$\mathbf{33.14}$} & $23.82$ & $32.97$ \\
        ~ & Mixture-MobileNet & $43.8$ & $53.33$ & $42.22$ & $2.95$ & $33.03$ & $49.58$ & $52.75$ & $23.86$ & $44.06$ & $24.91$ & $35.64$ & $26.35$ & $36.04$ \\
        ~ & Inception-v3-SCBR & $\ms{51.40}{0.29}$ & $\ms{57.07}{0.43}$ & $\ms{65.84}{0.52}$ & $\ms{2.91}{0.00}$ & $\ms{54.85}{1.51}$ & $\ms{50.77}{0.21}$ & $\ms{54.97}{0.20}$ & $\ms{36.05}{0.25}$ & $\ms{52.10}{0.46}$ & $\ms{29.22}{0.12}$ & $\ms{45.45}{0.62}$ & $\ms{31.93}{0.30}$ & $44.38$ \\
        ~ & Inception-v3-LEP-mGvM & $\ms{34.80}{0.00}$ & $\ms{40.72}{0.00}$ & $\ms{52.65}{0.00}$ & $\ms{4.34}{0.00}$ & $\ms{24.99}{0.00}$ & $\ms{37.56}{0.00}$ & $\ms{39.66}{0.00}$ & $\ms{27.83}{0.00}$ & $\ms{32.95}{0.00}$ & $\ms{24.63}{0.00}$ & $\ms{43.35}{0.00}$ & $\ms{26.92}{0.00}$ & $32.53$ \\
        ~ & Inception-v3-MCR2 & $\ms{37.50}{0.88}$ & $\ms{41.99}{0.53}$ & $\ms{54.47}{0.40}$ & $\ms{3.86}{0.10}$ & $\ms{28.13}{0.46}$ & $\ms{39.33}{0.70}$ & $\ms{38.99}{0.12}$ & $\ms{29.33}{0.10}$ & $\ms{32.98}{0.42}$ & $\ms{24.91}{0.66}$ & $\ms{42.41}{0.47}$ & $\ms{28.29}{0.35}$ & $33.52$ \\
        ~ & ConvNeXt-SCBR & $\ms{48.71}{0.33}$ & $\ms{55.70}{0.38}$ & $\ms{60.01}{0.63}$ & $\ms{2.89}{0.00}$ & $\ms{56.20}{2.37}$ & $\ms{47.66}{0.41}$ & $\ms{53.50}{0.19}$ & $\ms{34.79}{0.13}$ & $\ms{50.33}{0.37}$ & $\ms{28.49}{0.16}$ & $\ms{45.15}{0.23}$ & $\ms{31.55}{0.11}$ & $42.92$ \\
        ~ & ConvNeXt-LEP-mGvM & $\ms{47.54}{0.00}$ & $\ms{50.85}{0.00}$ & $\ms{54.67}{0.00}$ & $\ms{3.39}{0.00}$ & $\ms{24.56}{0.00}$ & $\ms{46.52}{0.00}$ & $\ms{57.75}{0.00}$ & $\ms{26.07}{0.00}$ & $\ms{51.45}{0.00}$ & $\ms{24.85}{0.00}$ & $\ms{37.70}{0.00}$ & $\ms{30.15}{0.00}$ & $37.96$ \\
        ~ & ConvNeXt-MCR2 & $\ms{34.05}{1.29}$ & $\ms{43.65}{1.65}$ & $\ms{51.87}{0.25}$ & $\ms{5.63}{0.39}$ & $\ms{30.71}{1.16}$ & $\ms{35.47}{0.40}$ & $\ms{36.77}{0.51}$ & $\ms{30.84}{0.28}$ & $\ms{38.53}{3.96}$ & $\ms{26.29}{0.34}$ & $\ms{44.32}{0.83}$ & $\ms{28.67}{0.33}$ & $33.90$ \\
        ~ & Inception-v3-SDR-ANGLE & $\ms{23.03}{0.16}$ & $\bms{27.54}{0.22}$ & $\bms{36.14}{0.42}$ & $\bms{2.53}{0.07}$ & $\bms{18.71}{0.26}$ & $\bms{28.96}{0.50}$ & $\bms{28.94}{0.26}$ & $\ms{23.81}{1.76}$ & $\ms{24.60}{1.64}$ & $\ms{18.67}{1.10}$ & $\ms{36.82}{0.12}$ & $\ms{23.45}{0.39}$ & $24.43$ \\
        ~ & Inception-v3-ANGLE & $\bms{22.74}{0.15}$ & $\ms{28.50}{0.22}$ & $\ms{39.24}{0.27}$ & $\ubms{2.51}{0.09}$ & $\ms{20.23}{0.21}$ & $\ms{29.46}{0.28}$ & $\ms{30.02}{0.18}$ & $\bms{21.82}{0.12}$ & $\ubms{22.54}{0.17}$ & $\bms{18.32}{0.15}$ & $\ms{35.05}{0.35}$ & $\bms{21.64}{0.37}$ & $\mathbf{24.34}$ \\
        ~ & ConvNeXt-ANGLE & $\ubms{22.55}{0.21}$ & $\ubms{25.36}{0.22}$ & $\ubms{33.70}{0.23}$ & $\ms{2.70}{0.12}$ & $\ubms{18.68}{0.20}$ & $\ubms{28.37}{0.24}$ & $\ubms{26.79}{0.24}$ & $\ubms{21.71}{0.16}$ & $\bms{23.40}{0.22}$ & $\ubms{16.75}{0.33}$ & $\ms{36.46}{0.93}$ & $\ubms{20.23}{0.61}$ & \underline{$\mathbf{23.06}$} \\ \midrule
        \multirow{4}{*}{\rotatebox[origin=c]{90}{MAAD}} & Mixture-Inception & $72$ & $50.84$ & $62.87$ & $3.03$ & $38.52$ & $80.91$ & $64.33$ & $\mathbf{27.76}$ & $70.77$ & $31.25$ & $41.3$ & $31.26$ & $47.90$ \\
        ~ & Mixture-DenseNet & $44.96$ & $72.59$ & $60.9$ & \underline{$\mathbf{2.88}$} & $43.18$ & $74.44$ & $53.66$ & $28.61$ & $70.06$ & $29.74$ & $41.26$ & $32.27$ & $46.21$ \\
        ~ & Mixture-MobileNet & $77.42$ & $74.02$ & $61.66$ & $3.24$ & $42.8$ & $72.09$ & $64.42$ & $29.3$ & $69.18$ & $28.6$ & $41.32$ & $32.69$ & $49.73$ \\
        ~ & Inception-v3-SCBR & $\ms{51.40}{0.29}$ & $\ms{57.07}{0.43}$ & $\ms{65.84}{0.52}$ & $\ms{2.91}{0.00}$ & $\ms{54.85}{1.51}$ & $\ms{50.77}{0.21}$ & $\ms{54.97}{0.20}$ & $\ms{36.05}{0.25}$ & $\ms{52.10}{0.46}$ & $\ms{29.22}{0.12}$ & $\ms{45.45}{0.62}$ & $\ms{31.93}{0.30}$ & $44.38$ \\
        \multirow{8}{*}{\rotatebox[origin=c]{90}{MAAD}} & Inception-v3-LEP-mGvM & $\ms{34.80}{0.00}$ & $\ms{40.72}{0.00}$ & $\ms{52.65}{0.00}$ & $\ms{4.34}{0.00}$ & $\ms{24.99}{0.00}$ & $\ms{37.56}{0.00}$ & $\ms{39.66}{0.00}$ & $\ms{27.83}{0.00}$ & $\ms{32.95}{0.00}$ & $\ms{24.63}{0.00}$ & $\ms{43.35}{0.00}$ & $\bms{26.92}{0.00}$ & $32.53$ \\
        ~ & Inception-v3-MCR2 & $\ms{37.50}{0.88}$ & $\ms{41.99}{0.53}$ & $\ms{54.47}{0.40}$ & $\ms{3.86}{0.10}$ & $\ms{28.13}{0.46}$ & $\ms{39.33}{0.70}$ & $\ms{38.99}{0.12}$ & $\ms{29.33}{0.10}$ & $\ms{32.98}{0.42}$ & $\ms{24.91}{0.66}$ & $\ms{42.41}{0.47}$ & $\ms{28.29}{0.35}$ & $33.52$ \\
        ~ & ConvNeXt-SCBR & $\ms{48.71}{0.33}$ & $\ms{55.70}{0.38}$ & $\ms{60.01}{0.63}$ & $\bms{2.89}{0.00}$ & $\ms{56.20}{2.37}$ & $\ms{47.66}{0.41}$ & $\ms{53.50}{0.19}$ & $\ms{34.79}{0.13}$ & $\ms{50.33}{0.37}$ & $\ms{28.49}{0.16}$ & $\ms{45.15}{0.23}$ & $\ms{31.55}{0.11}$ & $42.92$ \\
        ~ & ConvNeXt-LEP-mGvM & $\ms{47.54}{0.00}$ & $\ms{50.85}{0.00}$ & $\ms{54.67}{0.00}$ & $\ms{3.39}{0.00}$ & $\bms{24.56}{0.00}$ & $\ms{46.52}{0.00}$ & $\ms{57.75}{0.00}$ & $\ubms{26.07}{0.00}$ & $\ms{51.45}{0.00}$ & $\ms{24.85}{0.00}$ & $\ubms{37.70}{0.00}$ & $\ms{30.15}{0.00}$ & $37.96$ \\
        ~ & ConvNeXt-MCR2 & $\ms{34.05}{1.29}$ & $\ms{43.65}{1.65}$ & $\bms{51.87}{0.25}$ & $\ms{5.63}{0.39}$ & $\ms{30.71}{1.16}$ & $\bms{35.47}{0.40}$ & $\bms{36.77}{0.51}$ & $\ms{30.84}{0.28}$ & $\ms{38.53}{3.96}$ & $\ms{26.29}{0.34}$ & $\ms{44.32}{0.83}$ & $\ms{28.67}{0.33}$ & $33.90$ \\
        ~ & Inception-v3-SDR-ANGLE & $\ms{33.92}{0.30}$ & $\ms{40.14}{0.57}$ & $\ms{54.55}{0.95}$ & $\ms{3.93}{0.23}$ & $\ms{25.84}{0.26}$ & $\ms{37.25}{0.14}$ & $\ms{37.57}{0.21}$ & $\ms{30.23}{2.54}$ & $\ms{33.98}{2.66}$ & $\ms{24.55}{0.80}$ & $\ms{39.17}{0.10}$ & $\ms{29.42}{0.15}$ & $32.55$ \\
        ~ & Inception-v3-ANGLE & $\bms{32.88}{0.29}$ & $\bms{39.26}{0.26}$ & $\ms{53.92}{0.55}$ & $\ms{3.98}{0.21}$ & $\ms{25.24}{0.18}$ & $\ms{37.09}{0.26}$ & $\ms{37.22}{0.08}$ & $\ms{28.07}{0.18}$ & $\ubms{31.80}{0.53}$ & $\bms{23.66}{0.25}$ & $\bms{39.13}{0.30}$ & $\ms{27.74}{0.12}$ & $\mathbf{31.67}$ \\
        ~ & ConvNeXt-ANGLE & $\ubms{30.67}{0.27}$ & $\ubms{36.93}{0.41}$ & $\ubms{48.80}{0.53}$ & $\ms{3.41}{0.18}$ & $\ubms{22.40}{0.24}$ & $\ubms{34.63}{0.17}$ & $\ubms{35.39}{0.18}$ & $\ms{28.42}{0.22}$ & $\bms{32.30}{0.42}$ & $\ubms{21.39}{0.20}$ & $\ms{39.76}{0.68}$ & $\ubms{26.29}{0.19}$ & \underline{$\mathbf{30.03}$} \\ \midrule
        \multirow{12}{*}{\rotatebox[origin=c]{90}{CMDE}} & Mixture-Inception & $0.78$ & $0.49$ & $0.65$ & \underline{$\mathbf{0.01}$} & $0.38$ & $0.88$ & $0.68$ & $0.24$ & $0.74$ & $0.24$ & $0.39$ & $0.25$ & $0.48$ \\
        ~ & Mixture-DenseNet & $0.41$ & $0.78$ & $0.65$ & \underline{$\mathbf{0.01}$} & $0.4$ & $0.82$ & $0.51$ & $0.24$ & $0.74$ & $0.24$ & $\mathbf{0.38}$ & $0.26$ & $0.45$ \\
        ~ & Mixture-MobileNet & $0.82$ & $0.79$ & $0.64$ & \underline{$\mathbf{0.01}$} & $0.4$ & $0.77$ & $0.68$ & $0.24$ & $0.74$ & $0.24$ & $\mathbf{0.38}$ & $0.26$ & $0.50$ \\
        ~ & Inception-v3-SCBR & $\ms{0.46}{0.00}$ & $\ms{0.53}{0.01}$ & $\ms{0.67}{0.01}$ & $\ubms{0.01}{0.00}$ & $\ms{0.48}{0.01}$ & $\ms{0.49}{0.00}$ & $\ms{0.53}{0.00}$ & $\ms{0.26}{0.00}$ & $\ms{0.46}{0.01}$ & $\ms{0.24}{0.00}$ & $\ms{0.40}{0.00}$ & $\ms{0.25}{0.00}$ & $0.40$ \\
        ~ & Inception-v3-LEP-mGvM & $\ms{0.30}{0.00}$ & $\ms{0.37}{0.00}$ & $\ms{0.53}{0.00}$ & $\ubms{0.01}{0.00}$ & $\bms{0.22}{0.00}$ & $\ms{0.36}{0.00}$ & $\ms{0.35}{0.00}$ & $\ms{0.23}{0.00}$ & $\ms{0.28}{0.00}$ & $\ms{0.19}{0.00}$ & $\ms{0.43}{0.00}$ & $\bms{0.21}{0.00}$ & $0.29$ \\
        ~ & Inception-v3-MCR2 & $\ms{0.32}{0.01}$ & $\ms{0.38}{0.01}$ & $\ms{0.53}{0.01}$ & $\ubms{0.01}{0.00}$ & $\ms{0.25}{0.01}$ & $\ms{0.37}{0.01}$ & $\ms{0.33}{0.00}$ & $\bms{0.22}{0.00}$ & $\bms{0.27}{0.01}$ & $\bms{0.17}{0.00}$ & $\ms{0.39}{0.01}$ & $\ms{0.22}{0.00}$ & $0.29$ \\
        ~ & ConvNeXt-SCBR & $\ms{0.42}{0.00}$ & $\ms{0.52}{0.01}$ & $\ms{0.59}{0.01}$ & $\ubms{0.01}{0.00}$ & $\ms{0.50}{0.03}$ & $\ms{0.46}{0.00}$ & $\ms{0.51}{0.00}$ & $\ms{0.25}{0.00}$ & $\ms{0.44}{0.01}$ & $\ms{0.23}{0.00}$ & $\ms{0.40}{0.00}$ & $\ms{0.25}{0.00}$ & $0.38$ \\
        ~ & ConvNeXt-LEP-mGvM & $\ms{0.48}{0.00}$ & $\ms{0.51}{0.00}$ & $\ms{0.57}{0.00}$ & $\ubms{0.01}{0.00}$ & $\ms{0.24}{0.00}$ & $\ms{0.49}{0.00}$ & $\ms{0.60}{0.00}$ & $\ms{0.23}{0.00}$ & $\ms{0.53}{0.00}$ & $\ms{0.21}{0.00}$ & $\ubms{0.37}{0.00}$ & $\ms{0.24}{0.00}$ & $0.37$ \\
        ~ & ConvNeXt-MCR2 & $\ms{0.28}{0.02}$ & $\ms{0.39}{0.02}$ & $\bms{0.49}{0.00}$ & $\ubms{0.01}{0.00}$ & $\ms{0.25}{0.01}$ & $\bms{0.32}{0.00}$ & $\bms{0.30}{0.01}$ & $\ms{0.23}{0.00}$ & $\ms{0.32}{0.04}$ & $\bms{0.17}{0.00}$ & $\ms{0.41}{0.01}$ & $\ms{0.22}{0.00}$ & $0.28$ \\
        ~ & Inception-v3-SDR-ANGLE & $\ms{0.28}{0.00}$ & $\ms{0.36}{0.01}$ & $\ms{0.54}{0.01}$ & $\ubms{0.01}{0.00}$ & $\ms{0.23}{0.00}$ & $\ms{0.34}{0.00}$ & $\ms{0.32}{0.00}$ & $\ms{0.23}{0.02}$ & $\ms{0.28}{0.03}$ & $\bms{0.17}{0.01}$ & $\bms{0.38}{0.00}$ & $\ms{0.23}{0.00}$ & $0.28$ \\
        ~ & Inception-v3-ANGLE & $\bms{0.27}{0.00}$ & $\bms{0.35}{0.00}$ & $\ms{0.53}{0.01}$ & $\ubms{0.01}{0.00}$ & $\bms{0.22}{0.00}$ & $\ms{0.34}{0.00}$ & $\ms{0.32}{0.00}$ & $\ubms{0.21}{0.00}$ & $\ubms{0.26}{0.01}$ & $\bms{0.17}{0.00}$ & $\ubms{0.37}{0.00}$ & $\bms{0.21}{0.00}$ & $\mathbf{0.27}$ \\
        ~ & ConvNeXt-ANGLE & $\ubms{0.24}{0.00}$ & $\ubms{0.32}{0.00}$ & $\ubms{0.46}{0.01}$ & $\ubms{0.01}{0.00}$ & $\ubms{0.19}{0.00}$ & $\ubms{0.31}{0.00}$ & $\ubms{0.29}{0.00}$ & $\ubms{0.21}{0.00}$ & $\ubms{0.26}{0.00}$ & $\ubms{0.14}{0.00}$ & $\bms{0.38}{0.01}$ & $\ubms{0.20}{0.00}$ & \underline{$\mathbf{0.25}$} \\

\bottomrule
\end{longtable}
}

{\tiny 
\setlength{\tabcolsep}{5.5pt} 
\begin{longtable}[]{@{}ccccccccc@{}}
\caption{{\small Ablation study to evaluate the performance of different components in the proposed model on the Indian wind dataset. The negative of the kernels are taken from \citet[Table 1]{gneiting2013strictly}. For each variant, the mean and standard deviation of the metrics are reported across 50 independent train-test runs. The \underline{\textbf{best}} and \textbf{second best} results are highlighted. MedErr, CRPS, and MAAD are reported in degrees.}}\label{table:ablation}\tabularnewline
\toprule\noalign{}
Distance & Mode & Kernel & Output & $\text{Accuracy}_{\pi/6}$ & MedErr & CRPS & MAAD & CMDE \\
\midrule\noalign{}

\multirow{54}{*}{Geodesic} & \multirow{27}{*}{Single-index} & Powered exponential & Modulated & $0.506\pm0.111$ & $30.838\pm12.758$ & $28.314\pm5.137$ & $40.988\pm10.284$ & $0.347\pm0.131$ \\ 
        ~ & ~ & ~ & Scaled sigmoid & $0.601\pm0.025$ & $23.300\pm1.469$ & $23.930\pm0.533$ & $33.319\pm1.096$ & $0.252\pm0.011$ \\ 
        ~ & ~ & ~ & $\mathrm{atan2}$ & $0.583\pm0.032$ & $24.773\pm2.355$ & $24.969\pm0.814$ & $34.107\pm1.336$ & $0.258\pm0.011$ \\ \cmidrule{3-9}
        ~ & ~ & Generalized Cauchy & Modulated & $0.492\pm0.128$ & $33.834\pm17.438$ & $29.117\pm6.228$ & $43.494\pm14.470$ & $0.380\pm0.191$ \\ 
        ~ & ~ & ~ & Scaled sigmoid & $0.596\pm0.028$ & $23.353\pm2.090$ & $24.086\pm0.766$ & $33.573\pm1.603$ & $0.255\pm0.016$ \\ 
        ~ & ~ & ~ & $\mathrm{atan2}$ & $0.556\pm0.043$ & $26.253\pm2.940$ & $25.548\pm1.172$ & $35.137\pm1.620$ & $0.266\pm0.017$ \\ \cmidrule{3-9}
        ~ & ~ & Dagum & Modulated & $0.322\pm0.076$ & $52.150\pm16.489$ & $37.338\pm4.756$ & $58.019\pm14.486$ & $0.559\pm0.201$ \\ 
        ~ & ~ & ~ & Scaled sigmoid & $0.463\pm0.058$ & $32.561\pm3.952$ & $27.884\pm1.375$ & $40.255\pm2.376$ & $0.321\pm0.032$ \\ 
        ~ & ~ & ~ & $\mathrm{atan2}$ & $0.421\pm0.055$ & $35.461\pm4.058$ & $28.881\pm1.575$ & $41.498\pm2.821$ & $0.330\pm0.033$ \\ \cmidrule{3-9}
        ~ & ~ & Multiquadric & Modulated & $0.474\pm0.131$ & $34.726\pm16.454$ & $30.373\pm6.354$ & $45.419\pm13.377$ & $0.406\pm0.174$ \\ 
        ~ & ~ & ~ & Scaled sigmoid & $0.606\pm0.020$ & $22.552\pm1.312$ & $23.775\pm0.599$ & $32.939\pm1.169$ & $0.250\pm0.014$ \\ 
        ~ & ~ & ~ & $\mathrm{atan2}$ & $0.590\pm0.030$ & $24.039\pm2.186$ & $24.360\pm0.912$ & $33.744\pm1.437$ & $0.256\pm0.014$ \\ \cmidrule{3-9}
        ~ & ~ & Sine power & Modulated & $0.546\pm0.093$ & $27.677\pm8.280$ & $26.552\pm4.656$ & $37.190\pm7.268$ & $0.297\pm0.090$ \\ 
        ~ & ~ & ~ & Scaled sigmoid & $0.600\pm0.020$ & $23.773\pm1.836$ & $23.554\pm0.484$ & $33.403\pm1.190$ & $0.249\pm0.010$ \\ 
        ~ & ~ & ~ & $\mathrm{atan2}$ & $0.586\pm0.026$ & $24.890\pm1.953$ & $24.398\pm0.783$ & $34.016\pm1.232$ & $0.252\pm0.011$ \\ \cmidrule{3-9}
        ~ & ~ & Spherical & Modulated & $0.347\pm0.109$ & $50.913\pm20.542$ & $35.597\pm6.367$ & $56.665\pm17.111$ & $0.544\pm0.234$ \\ 
        ~ & ~ & ~ & Scaled sigmoid & $0.558\pm0.054$ & $25.493\pm4.074$ & $25.264\pm1.407$ & $35.825\pm2.729$ & $0.282\pm0.028$ \\ 
        ~ & ~ & ~ & $\mathrm{atan2}$ & $0.457\pm0.052$ & $33.615\pm4.288$ & $28.206\pm2.312$ & $40.353\pm3.806$ & $0.320\pm0.049$ \\ \cmidrule{3-9}
        ~ & ~ & Askey & Modulated & $0.358\pm0.109$ & $48.936\pm19.576$ & $35.737\pm5.798$ & $57.041\pm15.566$ & $0.553\pm0.212$ \\ 
        ~ & ~ & ~ & Scaled sigmoid & $0.554\pm0.057$ & $25.680\pm4.305$ & $25.764\pm1.339$ & $36.398\pm2.568$ & $0.291\pm0.025$ \\ 
        ~ & ~ & ~ & $\mathrm{atan2}$ & $0.462\pm0.065$ & $32.770\pm4.848$ & $27.815\pm1.761$ & $39.867\pm2.714$ & $0.315\pm0.028$ \\ \cmidrule{3-9}
        ~ & ~ & $C^2$-Wendland & Modulated & $0.547\pm0.097$ & $27.141\pm7.615$ & $26.047\pm3.720$ & $37.226\pm7.359$ & $0.298\pm0.088$ \\ 
        ~ & ~ & ~ & Scaled sigmoid & $0.600\pm0.018$ & $23.331\pm1.751$ & $23.421\pm0.486$ & $32.985\pm1.111$ & $0.246\pm0.010$ \\ 
        ~ & ~ & ~ & $\mathrm{atan2}$ & $0.588\pm0.023$ & $24.927\pm2.029$ & $24.225\pm0.748$ & $33.772\pm1.224$ & $0.249\pm0.010$ \\ \cmidrule{3-9}
        ~ & ~ & $C^4$-Wendland & Modulated & $0.541\pm0.113$ & $28.517\pm11.805$ & $27.149\pm5.533$ & $39.135\pm10.713$ & $0.328\pm0.137$ \\ 
        ~ & ~ & ~ & Scaled sigmoid & $0.603\pm0.022$ & $22.989\pm1.392$ & $23.584\pm0.536$ & $32.906\pm1.044$ & $0.247\pm0.011$ \\ 
        ~ & ~ & ~ & $\mathrm{atan2}$ & $0.595\pm0.021$ & $24.299\pm1.739$ & $24.246\pm0.580$ & $33.561\pm1.109$ & $0.250\pm0.009$ \\ \cmidrule{2-9}
        ~ & \multirow{27}{*}{Dense} & Powered exponential & Modulated & $0.451\pm0.120$ & $35.858\pm15.722$ & $29.635\pm6.077$ & $43.414\pm13.835$ & $0.369\pm0.182$ \\ 
        ~ & ~ & ~ & Scaled sigmoid & $0.636\pm0.029$ & $20.780\pm1.893$ & $22.108\pm0.815$ & $30.920\pm1.167$ & $0.229\pm0.012$ \\ 
        ~ & ~ & ~ & $\mathrm{atan2}$ & $0.608\pm0.026$ & $22.268\pm1.849$ & $23.777\pm0.860$ & $32.897\pm1.269$ & $0.251\pm0.014$ \\ \cmidrule{3-9}
        ~ & ~ & Generalized Cauchy & Modulated & $0.420\pm0.113$ & $39.099\pm16.116$ & $30.684\pm6.301$ & $45.639\pm14.440$ & $0.393\pm0.193$ \\ 
        ~ & ~ & ~ & Scaled sigmoid & $0.607\pm0.022$ & $22.226\pm1.513$ & $22.970\pm0.758$ & $32.209\pm0.931$ & $0.241\pm0.010$ \\ 
        ~ & ~ & ~ & $\mathrm{atan2}$ & $0.593\pm0.027$ & $23.061\pm2.227$ & $24.234\pm0.683$ & $33.579\pm1.047$ & $0.257\pm0.011$ \\ \cmidrule{3-9}
        ~ & ~ & Dagum & Modulated & $0.301\pm0.086$ & $55.856\pm19.890$ & $38.068\pm5.243$ & $61.646\pm17.354$ & $0.608\pm0.240$ \\ 
        ~ & ~ & ~ & Scaled sigmoid & $0.366\pm0.031$ & $39.528\pm2.197$ & $30.673\pm0.996$ & $45.573\pm2.121$ & $0.378\pm0.029$ \\ 
        ~ & ~ & ~ & $\mathrm{atan2}$ & $0.372\pm0.027$ & $38.773\pm2.181$ & $29.552\pm0.907$ & $43.905\pm1.697$ & $0.353\pm0.022$ \\ \cmidrule{3-9}
        ~ & ~ & Multiquadric & Modulated & $0.465\pm0.129$ & $34.633\pm14.711$ & $29.580\pm6.322$ & $42.983\pm12.499$ & $0.367\pm0.160$ \\ 
        ~ & ~ & ~ & Scaled sigmoid & $0.672\pm0.036$ & $18.840\pm2.170$ & $20.698\pm1.240$ & $28.826\pm2.034$ & $0.208\pm0.021$ \\ 
        ~ & ~ & ~ & $\mathrm{atan2}$ & $0.658\pm0.029$ & $19.119\pm1.622$ & $22.214\pm0.927$ & $30.367\pm1.771$ & $0.228\pm0.020$ \\ \cmidrule{3-9}
        ~ & ~ & Sine power & Modulated & $0.554\pm0.125$ & $27.014\pm10.983$ & $25.174\pm5.224$ & $36.001\pm9.907$ & $0.284\pm0.118$ \\ 
        ~ & ~ & ~ & Scaled sigmoid & $0.732\pm0.043$ & $16.233\pm1.900$ & $18.786\pm1.198$ & $25.509\pm2.005$ & $0.174\pm0.019$ \\ 
        ~ & ~ & ~ & $\mathrm{atan2}$ & $0.665\pm0.031$ & $19.008\pm1.857$ & $21.604\pm1.029$ & $29.514\pm1.688$ & $0.217\pm0.017$ \\ \cmidrule{3-9}
        ~ & ~ & Spherical & Modulated & $0.325\pm0.088$ & $50.511\pm19.264$ & $35.331\pm5.642$ & $56.105\pm16.993$ & $0.529\pm0.235$ \\ 
        ~ & ~ & ~ & Scaled sigmoid & $0.484\pm0.045$ & $31.104\pm3.022$ & $26.245\pm1.174$ & $38.034\pm2.030$ & $0.292\pm0.023$ \\ 
        ~ & ~ & ~ & $\mathrm{atan2}$ & $0.506\pm0.048$ & $29.524\pm3.122$ & $26.584\pm1.333$ & $37.273\pm2.312$ & $0.286\pm0.028$ \\ \cmidrule{3-9}
        ~ & ~ & Askey & Modulated & $0.303\pm0.081$ & $54.639\pm19.590$ & $37.873\pm5.146$ & $59.828\pm16.684$ & $0.582\pm0.232$ \\ 
        ~ & ~ & ~ & Scaled sigmoid & $0.415\pm0.038$ & $35.506\pm2.220$ & $27.932\pm1.155$ & $41.422\pm1.923$ & $0.323\pm0.024$ \\ 
        ~ & ~ & ~ & $\mathrm{atan2}$ & $0.444\pm0.036$ & $33.584\pm2.372$ & $27.495\pm0.652$ & $39.308\pm1.658$ & $0.301\pm0.017$ \\ \cmidrule{3-9}
        ~ & ~ & $C^2$-Wendland & Modulated & $0.581\pm0.114$ & $24.504\pm7.757$ & $24.105\pm4.472$ & $33.562\pm6.640$ & $0.255\pm0.068$ \\ 
        ~ & ~ & ~ & Scaled sigmoid & $\mathbf{0.753\pm0.038}$ & $15.229\pm1.468$ & $\mathbf{18.180\pm1.021}$ & $\mathbf{24.539\pm1.711}$ & $\mathbf{0.167\pm0.017}$ \\ 
        ~ & ~ & ~ & $\mathrm{atan2}$ & $0.702\pm0.036$ & $17.435\pm1.870$ & $20.464\pm1.215$ & $27.494\pm1.930$ & $0.197\pm0.020$ \\ \cmidrule{3-9}
        ~ & ~ & $C^4$-Wendland & Modulated & $0.592\pm0.101$ & $23.940\pm6.771$ & $24.047\pm4.219$ & $34.032\pm6.975$ & $0.265\pm0.078$ \\ 
        ~ & ~ & ~ & Scaled sigmoid & $0.739\pm0.041$ & $15.806\pm1.795$ & $18.460\pm1.298$ & $25.053\pm2.038$ & $0.171\pm0.020$ \\ 
        ~ & ~ & ~ & $\mathrm{atan2}$ & $0.701\pm0.033$ & $17.189\pm1.854$ & $20.599\pm1.154$ & $27.715\pm1.970$ & $0.200\pm0.020$ \\ \midrule
        \multirow{6}{*}{Chordal} & \multirow{3}{*}{Single-index} & \multirow{6}{*}{None} & Modulated & $0.583\pm0.059$ & $25.100\pm4.366$ & $24.638\pm2.660$ & $34.387\pm3.704$ & $0.262\pm0.043$ \\ 
        ~ & ~ & ~ & Scaled sigmoid & $0.602\pm0.014 $ & $23.032\pm1.255$ & $22.949\pm0.537$ & $32.528\pm0.769$ & $0.242\pm0.008$ \\ 
        ~ & ~ & ~ & $\mathrm{atan2}$ & $0.596\pm0.019$ & $24.012\pm1.628$ & $23.481\pm0.472$ & $33.138\pm0.796$ & $ 0.245\pm0.006$ \\ \cmidrule{4-9}
        ~ & \multirow{3}{*}{Dense} & ~ & Modulated & $0.666\pm0.090$ & $19.418\pm6.437$ & $20.884\pm4.080$ & $28.774\pm5.777$ & $0.208\pm0.064$ \\ 
        ~ & ~ & ~ & Scaled sigmoid & \underline{$\mathbf{0.781\pm0.027}$} & \underline{$\mathbf{14.234\pm1.284}$} & \underline{$\mathbf{16.948\pm0.800}$} & \underline{$\mathbf{22.700\pm1.259}$} & \underline{$\mathbf{ 0.149\pm0.011}$} \\ 
        ~ & ~ & ~ & $\mathrm{atan2}$ & $0.746\pm0.029$ & $\mathbf{14.984\pm1.192}$ & $18.578\pm0.952$ & $24.813\pm1.522$ & $0.171\pm0.016$ \\
\bottomrule\noalign{}
\endlastfoot
\end{longtable}
}

\begin{figure}[!h]
        \centering
        \includegraphics[width=0.8\linewidth]{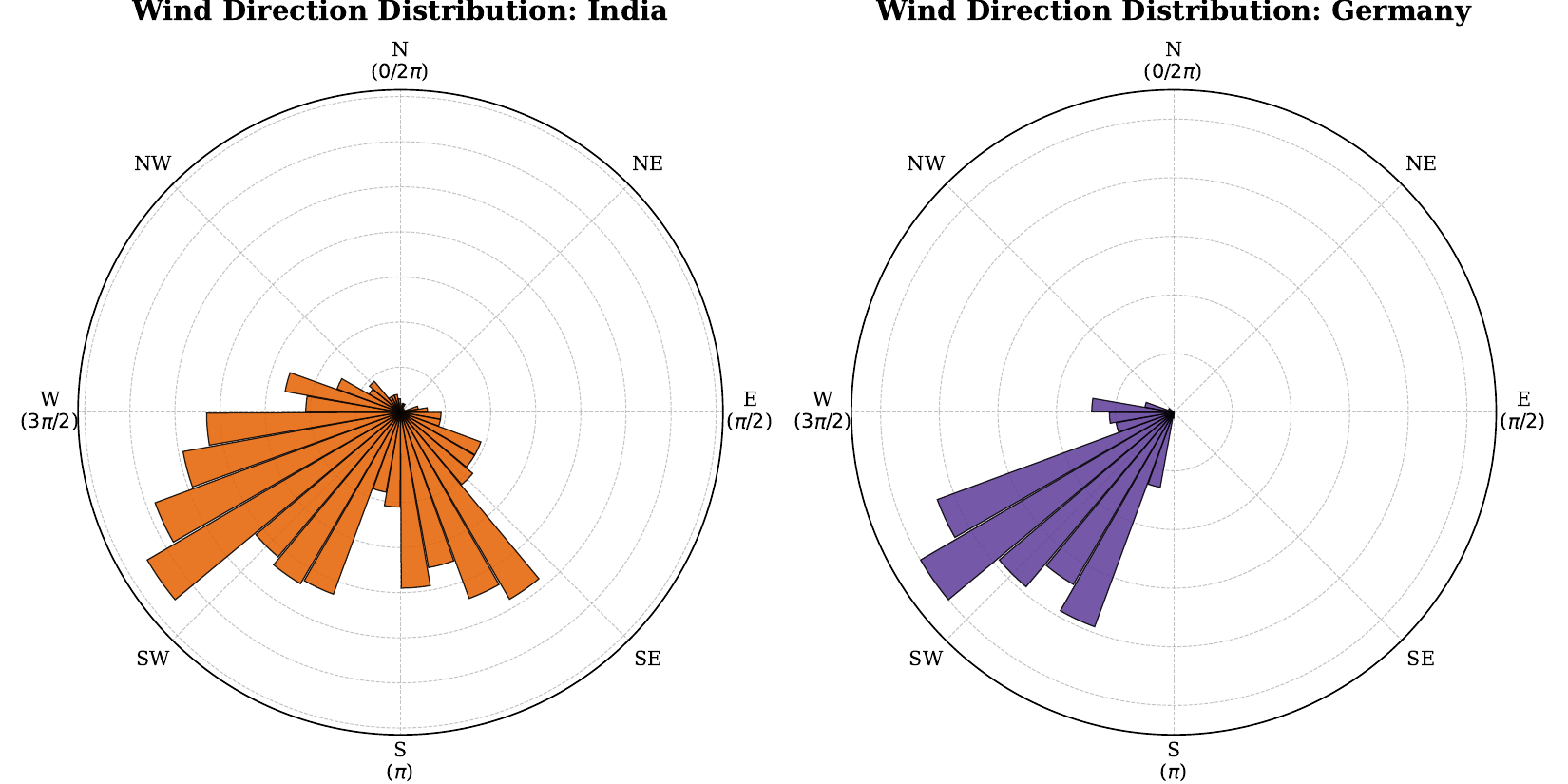}
        \caption{Rose plots indicating distribution of wind directions in $[0,2\pi)$ range for (A) Indian and (B) German datasets.}
        \label{fig:wind_distributions}
\end{figure}

\begin{figure}[!t]
        \centering
        \includegraphics[width=0.8\linewidth]{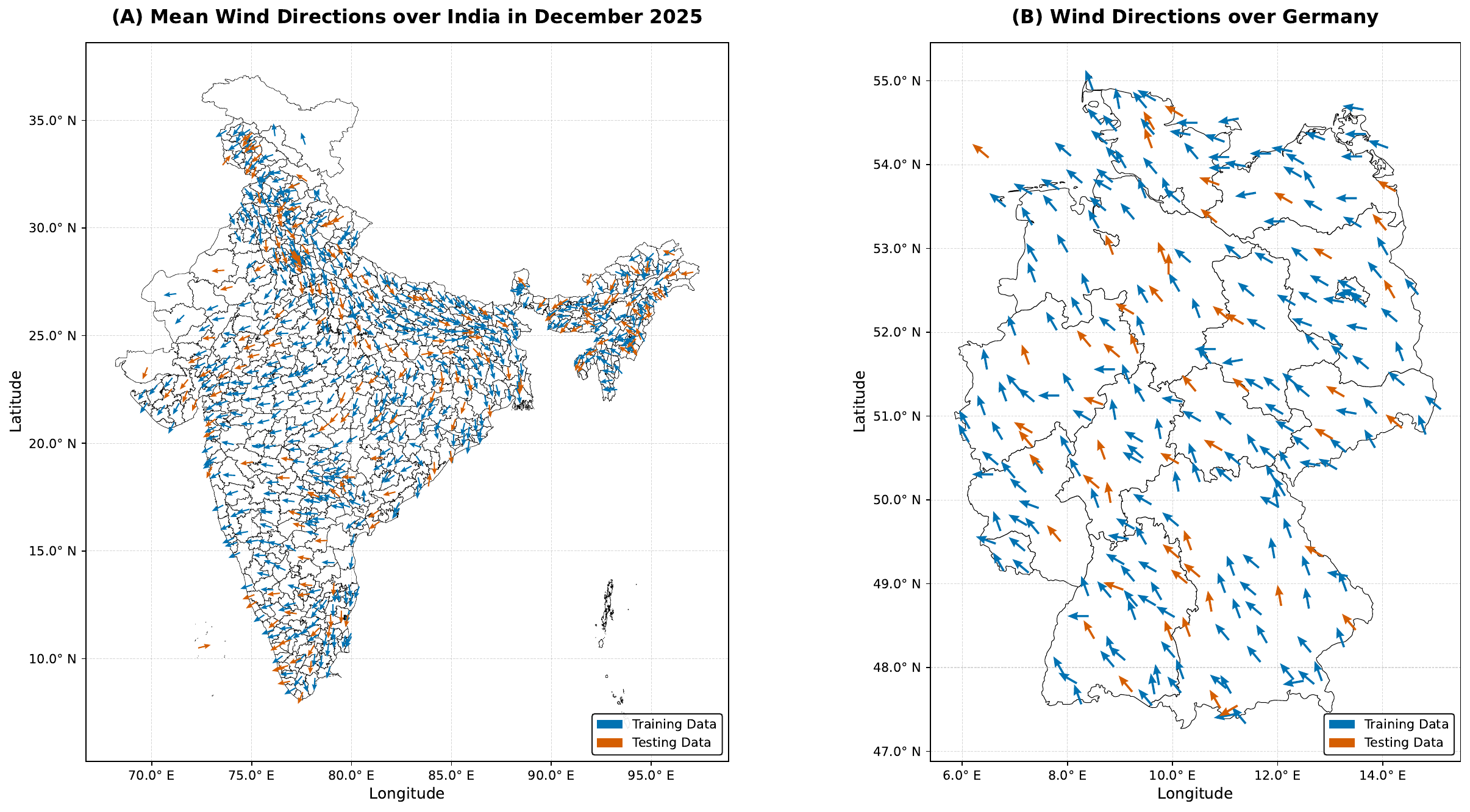}
        \caption{{\small Training and test observations for (A) Indian and (B) German wind datasets. The train-test split is done randomly for both the datasets. The boundaries are for illustrative purposes only, and imply no political assertions.}}
        \label{fig:combined_wind_datasets}
\end{figure}


\begin{figure}
        \centering
        \includegraphics[width=0.8\linewidth]{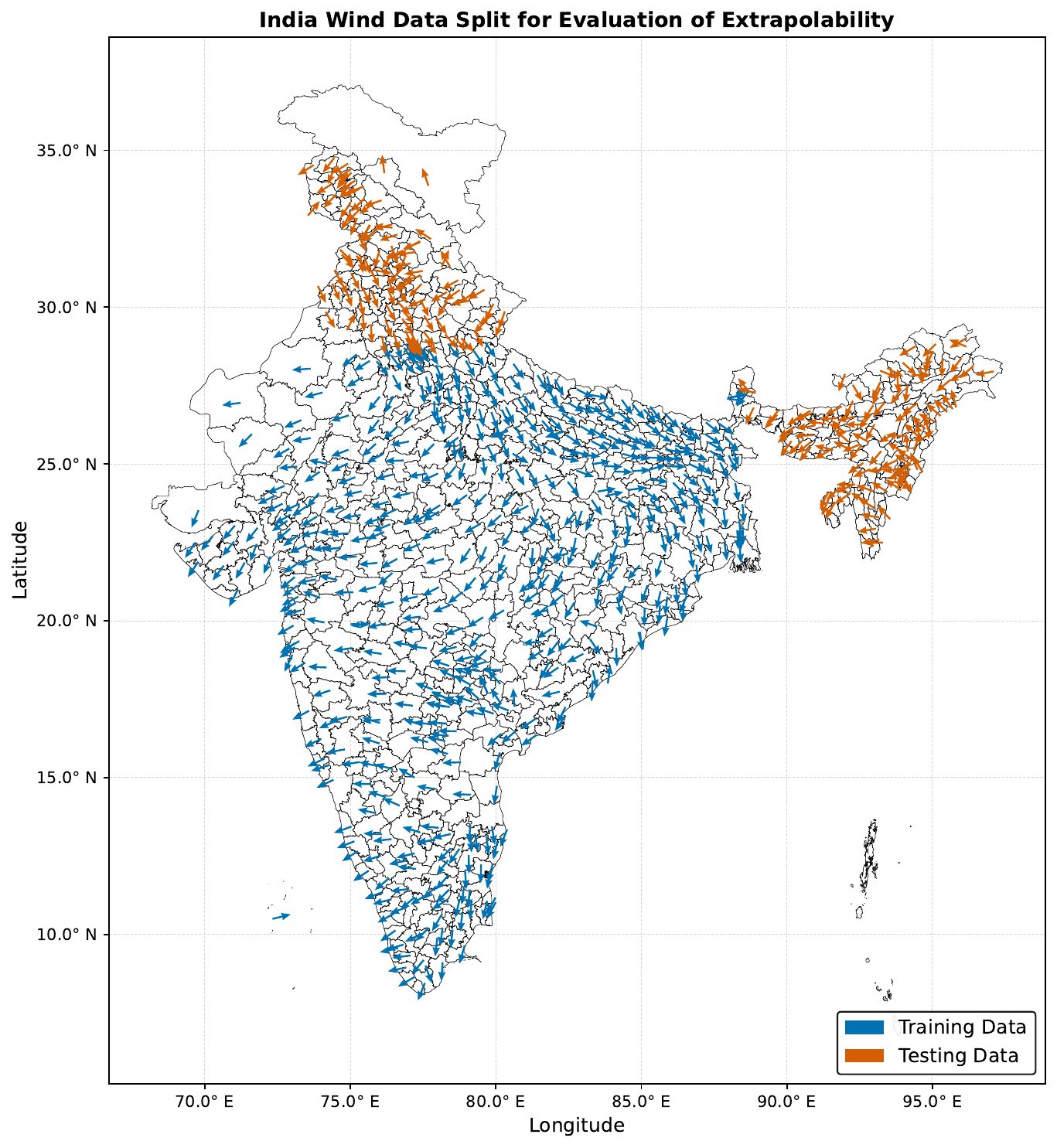}
        \caption{Training (70.1\%) and test (29.9\%) points of the Indian wind dataset for evaluation of extrapolability. The boundaries imply no political assertions and are for illustrative purposes only.}
        \label{fig:train_test_splits_Extrapolability}
\end{figure}


\clearpage
\bibliography{bibliography}

\end{document}